\documentclass[a4paper]{article}

\usepackage[tt=false]{libertine}
\usepackage{libertinust1math}
\usepackage[font={sf,small}]{caption} % caption font

\usepackage{mathtools} % \mathclap
\usepackage{siunitx}
\usepackage{bbm} % \mathbbm

\DeclareMathOperator{\diag}{diag}
\newcommand{\nvecs}{\mathbf n^-, \mathbf n^0, \mathbf n^+}
\newcommand{\barnvecs}{\bar{\mathbf n}^-, \bar{\mathbf n}^0, \bar{\mathbf n}^+}
\usepackage{amsthm}
\newtheorem{theorem}{Theorem}

\theoremstyle{definition}
\newtheorem{definition}{Definition}

\usepackage{tikz}
\usetikzlibrary{bayesnet}

\usepackage{hyperref}

\usepackage{booktabs}

\usepackage{graphicx}
\newcommand{\widecenter}[1]{\noindent\hspace{-\textwidth}\makebox[3\textwidth][c]{#1}}
\newcommand{\includempl}[1]{\includegraphics[scale=0.78125]{#1}}

\usepackage[
    style=authoryear-comp,
    uniquename=false,
    uniquelist=false,
    date=year,
    backref=true,
]{biblatex}
\AtEveryBibitem{
    \iffieldundef{doi}{}{ % remove other links if there is a DOI
        \clearfield{url}
        \clearfield{eprint}
    }
    \iffieldundef{eprint}{}{ % remove url if there is an eprint
        \clearfield{url}
    }
    \clearfield{addendum} % remove my personal notes
}

\title{On the Gaussian process limit of Bayesian Additive Regression Trees}
\author{Giacomo Petrillo\\Department of Statistics, Computer Science,
Applications\\(DISIA), University of Florence\\
\href{mailto:giacomo.petrillo@unifi.it}{giacomo.petrillo@unifi.it}}

\let\oldmarginpar\marginpar
\renewcommand{\marginpar}[1]{\oldmarginpar{\sffamily\scriptsize #1}}
\renewcommand{\marginpar}[1]{\relax} % hide marginpars

\begin{document}
    
    \maketitle
    
    \begin{abstract}
        
        Bayesian Additive Regression Trees (BART) is a nonparametric Bayesian regression technique of rising fame. It is a sum-of-decision-trees model, and is in some sense the Bayesian version of boosting. In the limit of infinite trees, it becomes equivalent to Gaussian process (GP) regression. This limit is known but has not yet led to any useful analysis or application. For the first time, I derive and compute the exact BART prior covariance function. With it I implement the infinite trees limit of BART as GP regression. Through empirical tests, I show that this limit is worse than standard BART in a fixed configuration, but also that tuning its hyperparameters in the natural GP way makes it competitive with BART. The advantage of using a GP surrogate of BART is the analytical likelihood, which simplifies model building and sidesteps the complex BART MCMC algorithm. More generally, this study opens new ways to understand and develop BART and GP regression. The implementation of BART as GP is available in the Python package \texttt{lsqfitgp}.
        
    \end{abstract}
    
    \section{Introduction}

    \paragraph{BART}
        
    Bayesian Additive Regression Trees (BART) is a nonparametric Bayesian regression method, introduced by \textcite{chipman2006,chipman2010}. It defines a prior distribution over the space of functions by representing them as a sum of binary regression trees, and then specifying a stochastic tree generation process. The posterior is then obtained with Metropolis-Gibbs sampling over the trees. See \textcite{hill2020} for a review and \textcite[ch.~5]{daniels2023} for a textbook treatment.

    \paragraph{BART's success}
    
    BART has proven empirically effective, and is gaining popularity in multiple fields \autocite[consider, e.g.,][]{tan2019}. As a single illustrative example, I consider its application in causal inference, where it was introduced by \textcite{hill2011}. In this context, BART is used for response surface methods, i.e., nonparametric modeling of the potential outcome conditional on the covariates in each treatment group. The Atlantic Causal Inference Conference (ACIC) Data Challenge has confirmed BART as one of the best methods \autocite{dorie2019,hahn2019,acic2019,thal2023}, in particular to infer heterogeneity in causal effects by subgroup. In causal inference, nonparametric models are useful because the quantity of interest is the (counterfactual) difference in some outcome variable between two (or more) treatment conditions, not the specific shape of the relationship linking covariates to outcome. In observational studies, i.e., without randomized assignment to treatment, strong unjustified functional assumptions, e.g., linear regression, can severely bias the effect estimate \autocite[\S14.7, p.~332]{imbens2015}.

    \paragraph{BART vs.\ Gaussian processes}
    
    Another general method for response surface modeling is Gaussian process (GP) regression (see \textcite{gramacy2020} for a textbook treatment and \textcite{linero2022b} for its application in causal inference), which uses as prior distribution a Gaussian process, i.e., a multivariate Normal on the unknown function values. This allows to write analytically the posterior as a linear algebra calculation, conditional on some free parameters of the covariance function, which may be more convenient compared to the Markov chain Monte Carlo (MCMC) algorithm of BART.

    The main disadvantage of GP regression is that in general the computational complexity is a hefty $O(n^2\max(n,p,d))$,\marginpar{Not sure about the $d$, maybe goes away with backprop.} with $n$ the number of datapoints, $p$ the number of predictors, and $d$ the number of free covariance function parameters; this can be alleviated with scaling techniques \autocite[ch.~9]{gramacy2020}, but not in full generality. With fixed number of MCMC samples and number of trees, BART is instead $O(n)$ \autocite[although these assumptions are too strong in practice, see][\S8, p.~25]{hill2020}.
    
    BART and GP regression are also considered different in their statistical features and performance. Typically GP regression is presented as less flexible than BART \autocite[see, e.g.,][]{jeong2023}. However, the BART prior distribution is defined by summing $m$ i.i.d.\ regression trees. Due to the central limit theorem, this implies that, as $m\to\infty$, the BART prior becomes Normal, and thus BART becomes equivalent to GP regression, with a specific covariance function \autocite[th.~1]{linero2017}. \textcite[\S2.2.5, p.~273 and fig.~6, p.~286]{chipman2010} observe empirically that, starting from $m=1$, the predictive performance drastically improves increasing $m$, until a value of $m$ in the hundredths, and then slowly worsens, but they do not investigate in depth the asymptotic $m\to\infty$ behavior.

    Thus, at the same time, there is an indicative argument that BART should be similar to a GP, but also reported experience that BART and GPs are in practice pretty different. This seeming contradiction is best exemplified in \textcite[\S5.2, p.~554]{linero2017}, who states that the BART prior covariance function can be approximated by the Laplace kernel
    \begin{equation}
        k(\mathbf x, \mathbf x') = \exp\left(
            -\sum_{i=1}^p \lambda_i |x_i - x_i'|
        \right),
    \end{equation}
    and that GP regression with this kernel is typically worse than BART. \textcite{linero2017} then hypothesizes that the reason behind this performance gap is the BART prior with a finite number of trees being more flexible than a GP prior, thus allowing BART to better adapt to the data.

    However, there could be another source for the discrepancy he observed: the Laplace kernel is only a loose approximation of the actual covariance function of BART. The present work starts from the following questions: how does BART compare with its infinite trees GP limit, taking care to compute precisely its prior covariance function? And would there be any advantages in using the GP limit in place of BART?

    \paragraph{Summary of the results}

    In \autoref{sec:limform}, I provide the first explicit expression of the BART prior covariance function, and in \autoref{sec:compcorr} a way to compute it with sufficient efficiency and accuracy to be used in practice for GP regression. In \autoref{sec:empirical} I run GP regressions with this covariance function, both on real (\autoref{sec:benchmarkmain}) and synthetic (\autoref{sec:acic}) data, comparing it against standard BART. This shows that the infinite trees limit of a given BART model is on average worse than the finite trees case. However, I also find that in practice, using the GP limit of BART for its own sake, without hewing to a strict limit of a fixed BART configuration, produces a method which is competitive with BART. This is interesting for applications because a GP formulation makes model extension and combination easier, and removes occasional MCMC headaches. Finally, in \autoref{sec:conclusions} I analyse the significance of my results and consider further research directions.

    \section{Theory}

    Sections \ref{sec:bart} and \ref{sec:gprecap} recap BART and GP regression. Section~\ref{sec:limform} presents the exact covariance function of BART derived here.
    
    \subsection{The BART prior}
    \label{sec:bart}
    
    BART, as a complete regression method, is the combination of a prior distribution over a space of regression functions, represented as an ensemble of decision trees, and a dedicated MCMC algorithm to sample from the posterior. My interest here lies mostly in the definition of the prior distribution, which I recap following the notation of \textcite{chipman2010}.
    
    \paragraph{Decision trees}

    \begin{figure}
        \widecenter{\includegraphics[width=\textwidth]{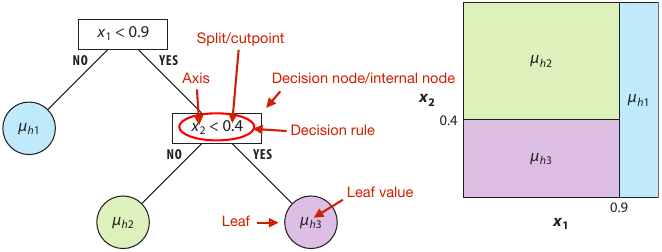}}
        \caption{\label{fig:tree} Depiction and terminology of a decision tree. Adapted from \textcite{hill2020}.}
    \end{figure}

    A decision tree is a way to represent a stepwise function $\mathbb R^p \to \mathbb R$ using a binary tree. Each non-leaf node contains a decision rule defined by a splitting axis $i$ and a splitting point $s$, partitioning the input space in two half-spaces along the plane $\{\mathbf x \in \mathbb R^p: x_i = s\}$, orthogonal to axis $i$. Each half is assigned to a children. Children nodes in turn further divide their space amongst their own children. Leaf nodes contain the constant function value to be associated to their whole assigned rectangular subregion. A minimal tree consisting of a single root leaf node represents a flat function. See \autoref{fig:tree} for an illustration.

    \paragraph{Decision trees for regression}

    A decision tree can be used to represent a parametrized family of functions for regression. The parameters are the tree structure $T$, including decision rules, and the leaf values $M$. I indicate the function represented by the tree with $g(\mathbf x;T,M)$. To use this model in Bayesian inference, it is necessary to specify a prior distribution over $(T,M)$. I factorize the prior as $p(T,M) = p(M\mid T) p(T)$ and specify the two parts separately.

    \paragraph{Prior over tree structure}

    The distribution $p(T)$ is defined by the following recursive generative model \autocite[see][]{chipman1998}. Fix an axes-aligned grid in the input space, not necessarily evenly spaced, by specifying a set of cutpoints $S_i = \{s_{i1}, \ldots, s_{in_i}\}$ for each axis $i \in \{1,\ldots,p\}$. Start generating the tree from the root, descending as nodes are added. At each new node, consider the subset of cutpoints which are available in its assigned subregion along each axis, and keep only the axes with at least one split available. If there are none, the node is a leaf. If there are any, decide with probability $P_d = \alpha/(1 + d)^\beta$ if the node is nonterminal, i.e., if it should have children, where $d$ is the depth of the node, $d=0$ at the root, and $\alpha \in [0, 1]$ and $\beta \in [0,\infty)$ are two fixed hyperparameters. If the node is nonterminal, its splitting axis is drawn from the allowed ones with a categorical distribution $\mathbf w$ (typically uniform), and the cutpoint along the chosen axis is drawn uniformly at random from the available ones.

    \paragraph{Prior over leaf values}

    If a node terminates, its leaf function value is sampled from a Normal distribution with mean $\mu_\mu$ and standard deviation $\sigma_\mu$, independently of other leaves. This defines $p(M\mid T)$.

    \paragraph{Complete regression model}

    The BART regression model uses a sum of $m$ independent trees, each with the prior above:
    \begin{align}
        y_i &= f(\mathbf x_i) + \varepsilon_i, &
        f(\mathbf x) &= \sum_{j=1}^m g(\mathbf x; T_j, M_j), \label{eq:bartdeffirst} \\
        &&
        p(\{T_j\}, \{M_j\}) &= \prod_{j=1}^m p(M_j \mid T_j) p(T_j), \\
        \varepsilon_i \mid\sigma^2 &\overset{i.i.d.}\sim \mathcal \mathcal N(0, \sigma^2), &
        \frac{\lambda\nu}{\sigma^2} &\sim \chi^2_\nu
        \text{ (for some fixed $\lambda$ and $\nu$)}. \label{eq:bartdeflast}
    \end{align}

    \paragraph{Hyperparameters}

    The free hyperparameters in the model, to be chosen by the user, are:
    \begin{itemize}

        \item The splitting grid. The two common defaults are 100 equally spaced splits along each $\mathbf x$ axis in the range of the data, or a split midway between each observed value along each $\mathbf x$ axis.
        
        \item The number of trees $m$, default 200. This should grow with the sample size, but there is no known precise scaling.
        
        \item $\alpha$ and $\beta$, regulating the depth distribution of the
        trees. Default 0.95 and 2.
        
        \item $\mu_\mu$ and $\sigma_\mu$, which set the prior mean and variance
        $m\mu_\mu$ and $m\sigma_\mu^2$ of $f(\mathbf x)$. By default set based on some measure of the location and scale of the data.
        
        \item $\nu$ and $\lambda$, which regulate the prior on the variance of the error term. By default $\nu=3$, while $\lambda$ is set to some measure of squared scale of the data.

    \end{itemize}
 
    \subsection{Gaussian process regression}
    \label{sec:gprecap}
    
    Here I briefly recap Gaussian process regression. For textbook references, see \textcite{rasmussen2006, gramacy2020, murphy2023}.

    \paragraph{Definition of Gaussian process}
    
    A stochastic process $f$ is a Gaussian process if all its finite marginals are multivariate Normals. In symbols:
    \begin{align}
        f &\sim \mathcal G \mathcal P(m(\cdot), k(\cdot, \cdot))
        \quad\text{iff} \notag \\
        \begin{pmatrix}
            f(x_1) \\ \vdots \\ f(x_n)
        \end{pmatrix} &\sim \mathcal N \left( \begin{pmatrix}
            m(x_1) \\ \vdots \\ m(x_n)
        \end{pmatrix}, \begin{pmatrix}
            k(x_1, x_1) & \cdots & k(x_1, x_n) \\
            \vdots & \ddots & \vdots \\
            k(x_n, x_1) & \cdots & k(x_n, x_n)
        \end{pmatrix}
        \right), \quad \forall \{x_1, \ldots, x_n\} \subseteq \mathcal X.
    \end{align}
    The functions $m$ and $k$ are respectively called mean function and covariance function or kernel. The domain is an arbitrary index space $\mathcal X$. If $\mathcal X = \mathbb R^p$, the kernel regulates the continuity and smoothness of the process in the mean-square sense, its correlation length, and its long-range memory.

    \paragraph{Bayesian inference with Gaussian processes}
    
    Bayesian Gaussian process regression takes advantage of the fact that the marginals are Normal, and easily computable from $m$ and $k$, to produce a posterior distribution given observed values of the process and a Gaussian process prior over it. Following the notation of \textcite{rasmussen2006}, let $f(\mathbf x) = (f(x_1), \ldots, f(x_n))$ be the observed function values, and $f(\mathbf x^*) = (f(x^*_1), \ldots, f(x^*_n))$ the unknown function values to infer. With the joint distribution written as
    \begin{align}
        \begin{pmatrix}
            f(\mathbf x) \\ f(\mathbf x^*)
        \end{pmatrix} &\sim \mathcal N \left( \begin{pmatrix}
            m(\mathbf x) \\ m(\mathbf x^*)
        \end{pmatrix}, \Sigma=\begin{pmatrix}
            \Sigma_{xx} & \Sigma_{xx^*} \\
            \Sigma_{x^*x} & \Sigma_{x^*x^*}
        \end{pmatrix} \right),
    \intertext{the posterior on $f(\mathbf x^*)$ is}
        (f(\mathbf x^*) \mid f(\mathbf x) = \mathbf y) &\sim
        \mathcal N\big(m(\mathbf x^*) + \Sigma_{x^*x} \Sigma_{xx}^-
        (\mathbf y - m(\mathbf x)), \notag \\
        &\phantom{{}\sim\mathcal N\big(}\Sigma_{x^*x^*} - \Sigma_{x^*x} \Sigma_{xx}^- \Sigma_{xx^*}\big),
        \label{eq:gpreg}
    \end{align}
    where $\Sigma_{xx}^-$ is any 1-inverse of $\Sigma_{xx}$, i.e., a matrix that satisfies $\Sigma_{xx} \Sigma_{xx}^- \Sigma_{xx} = \Sigma_{xx}$. In particular, the Moore-Penrose pseudoinverse $\Sigma_{xx}^+$ is a valid choice, and $\Sigma_{xx}^- = \Sigma_{xx}^{-1}$ if the matrix is invertible. The posterior covariance matrix $\Sigma_{x^*x^*} - \Sigma_{x^*x} \Sigma_{xx}^- \Sigma_{xx^*}$ is a generalized Schur complement, indicated with $\Sigma/\Sigma_{xx}$. For the proof of the conditioning formula see \textcite[ex.~7.4, p.~295]{schott2017}; for the general properties of the multivariate Normal, see \textcite{tong1990}.

    In the typical application there are independent Normal error terms between the latent function and the data. Since Normal distributions can be summed, the error term is absorbed into the definition of the Gaussian process.

    \paragraph{Hyperparameters}
    
    If mean and covariance function depend on additional hyperparameters $\theta$, the distribution can be set up as a hierarchical model with a marginal prior for $\theta$ and the Gaussian process being the conditional distribution given $\theta$. The posterior density on $\theta$ is then given by
    \begin{align}
        p(\theta \mid f(\mathbf x) = \mathbf y) &\propto
        \mathcal N(\mathbf y; m(\mathbf x;\theta), \Sigma_{xx}(\theta))
        p(\theta), \label{eq:thetapost} \\
    \intertext{while the posterior on $f(\mathbf x^*),\theta$ becomes}
        p(f(\mathbf x^*),\theta \mid f(\mathbf x) = \mathbf y) &=
        p(f(\mathbf x^*) \mid f(\mathbf x) = \mathbf y, \theta)
        p(\theta \mid f(\mathbf x) = \mathbf y),
    \end{align}
    with the conditional posterior $p(f(\mathbf x^*) \mid f(\mathbf x) = \mathbf y, \theta)$ given by \autoref{eq:gpreg}.

    It is common practice to find a single ``optimal'' value of $\theta$, e.g., its marginal maximum a posteriori (MAP). This is convenient, but often substantially deteriorates the quality of the predictive posterior distribution, shrinking its variance \autocite[e.g.,][fig.~18.18, p.~713]{murphy2023}. If the hyperparameters carry some meaning in the model, and inference about them is sought, a Laplace approximation of the $\theta$ posterior is usually inadequate, but may be sufficient for the predictive posterior.

    In principle it is possible to use Gaussian process regression with all hyperparameters fixed a priori. This however is a bad idea in most circumstances, because a single multivariate Normal distribution is too ``rigid'' for almost all prediction tasks.\marginpar{If someone asks for a citation for this, I may find something at the beginning of \textcite{balog2016}.}

    \paragraph{Computational speed aspects}
    
    The typical bottlenecks are, at each step of the iterative algorithm that maximizes or samples the posterior for $\theta$ in \autoref{eq:thetapost}, computing the decomposition of the prior covariance matrix $\Sigma_{xx}(\theta)$ and computing the gradient of the log-determinant term in the log-likelihood. Their computational complexities are respectively $O(n^2\max(n,p))$ and $O(n^2\max(n,d))$\marginpar{Not sure about $d$, maybe goes away with backprop.} where $p$ and $d$ are the dimensionalities of $x$ and $\theta$. As a reference, on my personal machine with $n=\num{10000}$ and 64~bit floating point, the matrix occupies \SI1{GB} of RAM, the Cholesky decomposition takes \SI2s, and diagonalization takes \SI{80}s. Due to the quadratic memory and cubic time scaling, it is challenging to significantly increase $n$ and $p$ even with larger computers. Thus one of the drawbacks of GP regression is that in general it does not scale to large datasets.
    
    \subsection{The BART prior correlation function}
    \label{sec:limform}
    
    Since BART defines a prior over functions, it has a prior covariance function, like Gaussian processes. Using the notation of Equations~\ref{eq:bartdeffirst} to~\ref{eq:bartdeflast} for $f$ and \autoref{sec:gprecap} for $k$:
    \begin{equation}
        k(\mathbf x, \mathbf x') =
            \operatorname{Cov}[f(\mathbf x), f(\mathbf x')].
    \end{equation}
    For convenience, from here onwards I fix the variance $m\sigma^2_\mu$ to 1, or, in other words, I work with the correlation function instead of the covariance function.

    The BART correlation function is equal to the probability that the two points end up in the same tree leaf under the generative definition of the prior \autocite[\S5.2]{linero2017}:
    \begin{equation}
        k(\mathbf x, \mathbf x') = P(\text{$\mathbf x$ in same leaf as $\mathbf x'$}).
    \end{equation}
    Previous literature provided expressions for this \autocites[th.~4.1]{oreilly2022}[prop.~1]{linero2017}[prop.~1]{balog2016}, but making simplyfing assumptions to ease analytical computability. The following theorem provides the first exact derivation of the BART correlation function:
    \begin{theorem}
        \label{th:corr}
        The BART prior correlation function of the latent regression function is given recursively by
        \begin{align}
            k(\mathbf x, \mathbf x') &= k_0(\nvecs), \\
            k_d(\mathbf 0, \mathbf 0, \mathbf 0) &=
            k_d((),(),()) = 1, \qquad\text{\rm(``$()$'' indicates a 0d vector)} \\
            k_d(\nvecs) &= 1 - P_d \Bigg[1 - \frac1{W(\mathbf n)}
                \sum_{\substack{i=1 \\ n_i\ne 0}}^p \frac{w_i}{n_i} \Bigg( \notag \\
                    &\qquad \sum_{k=0}^{n^-_i - 1}
                    k_{d+1}(\mathbf n^-_{n^-_i=k}, \mathbf n^0, \mathbf n^+)
                    + \sum_{k=0}^{n^+_i - 1}
                    k_{d+1}(\mathbf n^-, \mathbf n^0, \mathbf n^+_{n^+_i=k})
                \Bigg)
            \Bigg], \label{eq:corr} \\
            \mathbf n &= \mathbf n^- + \mathbf n^0 + \mathbf n^+, \\
            W(\mathbf n) &= \sum_{\substack{i=1 \\ n_i\ne 0}}^p w_i, \qquad
            \mathbf w > 0, \\
            P_d &= \frac \alpha {(1+d)^\beta},
        \end{align}
        where $\mathbf n^-$, $\mathbf n^0$, and $\mathbf n^+$ are, respectively, the number of splitting points below, between, and above the predictors vectors $\mathbf x$ and $\mathbf x'$, counted separately along each covariate dimension (see \autoref{fig:counts}), \marginpar{It could be a poor choice to have the number of splitting points with the same name as the number of datapoints, although it is not ambiguous since the first appears as $\mathbf n$ or $n_i$ while the latter as $n$.}
        \begin{figure}
            \centering\includegraphics[width=65ex]{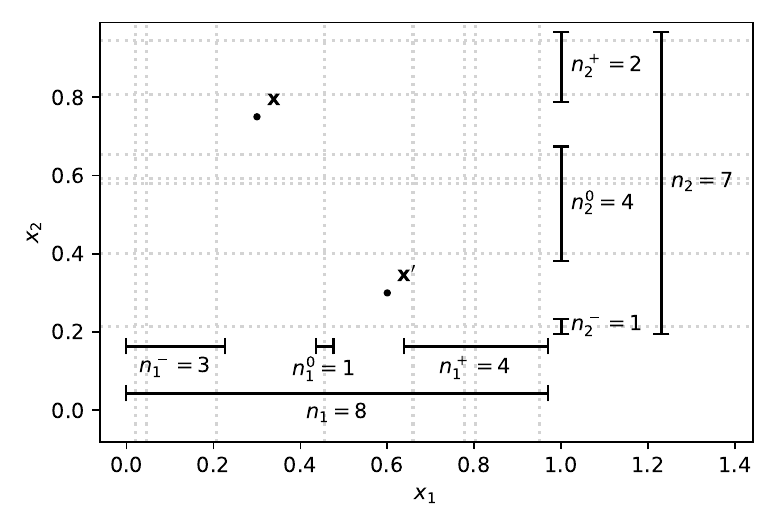}
            \caption{\label{fig:counts} Illustration of the definition of $\mathbf n^-$, $\mathbf n^0$ and $\mathbf n^+$. The dotted lines represent the predetermined splitting grid used to construct the trees.}
        \end{figure}
        while $\mathbf w$ are the (unnormalized) selection probabilities for the $p$ predictor axes. All predictors with $w_i = 0$ shall be excluded beforehand.
    \end{theorem}

    \begin{proof}
        See \autoref{sec:corrproof}.
    \end{proof}

    For completeness I also considered the case $p = 0$, but in the following I will always assume $p \ge 1$. Note that the $\mathbf n^\pm$ vectors appearing in recursive evaluations of the $k_d$ function are different from the ones initially obtained from $\mathbf x$ and $\mathbf x'$, due to the modifications indicated by the subscripts ${}_{n^-_i=k}$ and ${}_{n^+_i=k}$. I consider arbitrary axis weights $\mathbf w$, but where not specified otherwise, in numerical experiments I will set $\mathbf w = (1, \ldots, 1)$.

    The number of trees $m$ never appears in \autoref{th:corr}: the prior distribution changes with $m$, but its covariance function stays the same.

    In the following, I always work with the recursive function $k_d$, representing the ``sub-correlation function'' of a subtree at depth $d$, instead of the final correlation function $k$. All the results apply to $k$ by considering that $k=k_0$. I also manipulate directly $P_d$ without involving $\alpha$ and $\beta$, so the results apply to any valid sequence of probabilities $P_d$. The next theorem formally characterizes the properties of $k_d$ that reflect corresponding properties of the prior:

    \begin{theorem}
        \label{th:corrprop}
        The BART sub-correlation function $k_d$ defined in \autoref{th:corr} has the following properties:
        \begin{enumerate}
            
            \item $k_d(\mathbf n^-, \mathbf 0, \mathbf n^+) = 1$. \label{it:corr1}
            
            \item Assuming $\mathbf n \ne \mathbf 0$, then $k_d(\mathbf 0, \mathbf n, \mathbf 0) = 1 - P_d$. \label{it:lower}
            
            \item $k_d \in [1 - P_d, 1]$, in particular $k_0 \ge 1 - \alpha$. \label{it:bounds}
            
            \item Assuming $P_d > 0$, then $k_d(\nvecs) = 1 \implies \mathbf n^0 = \mathbf 0$. \label{it:center}
            
            \item Assuming $P_d > 0$ and $P_{d+1} < 1$, then $k_d(\nvecs) = 1 - P_d \implies (\nvecs) = (\mathbf 0, \mathbf n, \mathbf 0)$; the conditions on $P_d$ and $P_{d+1}$ are satisfied for any $d$ if $\alpha > 0$ and $\beta < \infty$. \label{it:corner}
            
            \item Given $\mathbf n^{-\prime} \le \mathbf n^-$ and $\mathbf n^{+\prime} \le \mathbf n^+$, then $k_d(\mathbf n^{-\prime}, \mathbf n^0, \mathbf n^{+\prime}) \le k_d(\mathbf n^-, \mathbf n^0, \mathbf n^+)$. \label{it:monoext}
            
            \item Given $\mathbf n^{0\prime} \ge \mathbf n^0$, then $k_d(\mathbf n^-, \mathbf n^{0\prime}, \mathbf n^+) \le k_d(\mathbf n^-, \mathbf n^0, \mathbf n^+)$. \label{it:monoint} \marginpar{Additional properties: I guess that if $\mathbf n' = \mathbf n$, $\mathbf n^{0\prime} = \mathbf n^0$, $|\mathbf n^{-\prime} - \mathbf n^{+\prime}| \le |\mathbf n^- - \mathbf n^+|$, then $k_d(\mathbf n^{-\prime}, \mathbf n^{0\prime}, \mathbf n^{+\prime}) \ge k_d(\nvecs)$. And how does $k_d(\nvecs)$ compare to $k_d(N\mathbf n^-, N\mathbf n^0, N\mathbf n^+)$?}
            
            \item If $P_{d+1} = 0$ (satisfied by $\beta \to \infty$), and $\mathbf n \ne \mathbf 0$, then
            \begin{equation}
                k_d(\nvecs) = 1 - \frac{P_d}{W(\mathbf n)}
                \sum_{\substack{i=1 \\ n_i \ne 0}}^p w_i \frac{n^0_i}{n_i}. \notag
            \end{equation} \label{it:nointer}
            \item If $\forall d' > d : P_{d'} = 1$, which occurs if $\alpha = 1$ and $\beta = 0$, then \label{it:white}
            \begin{equation}
                k_d(\nvecs) = \begin{cases}
                    1 & \mathbf n^0 = \mathbf 0, \\
                    1 - P_d & \mathbf n^0 \ne \mathbf 0.
                \end{cases} \notag
            \end{equation}
            
        \end{enumerate}
    \end{theorem}

    \begin{proof}
        See \autoref{sec:corrprop}.
    \end{proof}
    
    Property~\ref{it:corr1} is just the fact that a function value is totally correlated with itself. Properties~\ref{it:lower} and~\ref{it:bounds} reflect that BART has a random intercept term with prior variance $1 - \alpha$ due to the possibility of the trees not growing at all. Property~\ref{it:center} says that any degree of separation (w.r.t.\ the grid) allows the function values to differ, and property~\ref{it:corner} that, apart from degenerate cases, the lower bound can be reached only by maximally separated points, i.e., points that sit at completely opposite corners of the splitting grid hypercube envelope. Properties~\ref{it:monoext} and~\ref{it:monoint} mean that the correlation monotonically decreases as points get farther along each axis.
        
    Properties~\ref{it:nointer} and~\ref{it:white} show two opposite limiting cases of BART. For $\beta \to \infty$, the trees can at most reach depth~1, so each tree can split at most on one predictor. This means that the regression function has no interactions, in the usual sense that it can be written as a sum of functions each acting only on one predictor, property which is reflected by the correlation function. Instead, if $\alpha = 1$ and $\beta = 0$, the trees continue growing until exhausting all the available splits, so the process becomes white noise (in grid space). Notice that in the $\beta \to \infty$ case the process is stationary (in grid space) because it depends only on the distance $n^0_i$ between the two points.
    
    \autoref{fig:plotcov} shows the correlation function with the default values of the hyperparameters given by \textcite{chipman2010}. It's interesting to note how close the function is to the no-interaction $\beta\to\infty$ limit of property~\ref{it:nointer}, which would be similar but with a perfectly straight profile. That small amount of curvature makes the difference between a regression without interactions and a rich model like BART. This is an instance of the general observation that it is difficult to infer the properties of a stochastic process by visually looking at its correlation function, and that small numerical differences in correlation can correspond to large qualitative differences in inference. In this case, the ``straight'' $\beta\to\infty$ limit of the kernel yields a degenerate stochastic process with support over functions without additive interactions, but this degeneracy is not apparent on visual inspection of its ``tight-tent''-shaped plot, and a slight deviation from the straight shape is sufficient to bring back all the other functions in the support.
    \begin{figure}
        
        \widecenter{\includempl{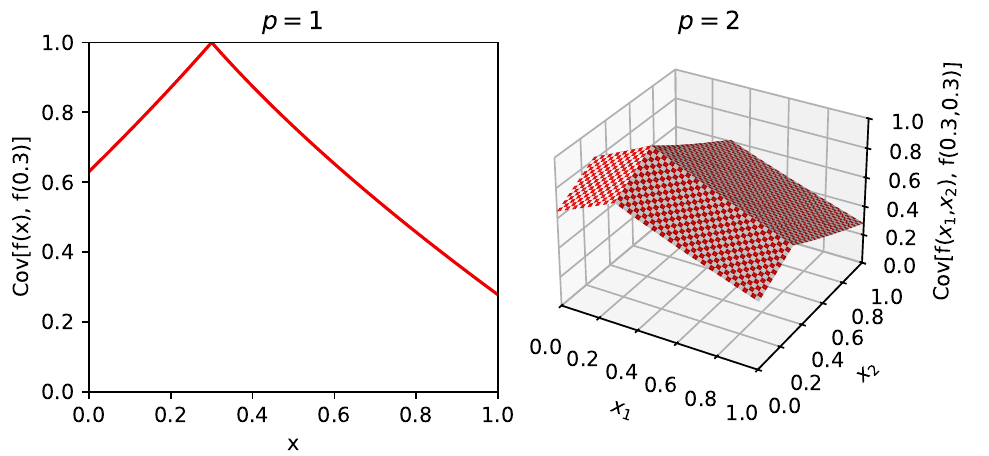}}
        
        \caption{\label{fig:plotcov} Section of the BART correlation function, with $\alpha=0.95$, $\beta=2$, and $n=100$ evenly spaced splitting points between~0 and~1. Computed with $k^{2,5}_{0,1}$ (see \autoref{sec:compcorr}).}
        
    \end{figure}
        
    \section{Computation of the BART prior correlation function}
    \label{sec:compcorr}
    
    Computing exactly the BART prior correlation following \autoref{eq:corr} amounts, in words, to walking recursively through all possible trees, calculating the probability of each, and checking whether it puts the two points in different leaves. Albeit finite, the number of possible trees is huge, making the computation unfeasible.

    Analogously to the previous literature, I make simplifying assumptions that allow to compute an approximation. However, I hew close to the exact BART specification, deriving a proper approximation method that gives guaranteed bounds and converts additional computational power in arbitrary accuracy. In particular, I take these steps:
    \begin{enumerate}

        \item Limit the maximum depth of the trees (\autoref{sec:truncorr}).

        \item Instead of just an estimate, get an upper and a lower bound on the exact result by putting either an upper or a lower bound in the truncated leaves in place of the subtree computation. These bounds are themselves valid (i.e., positive semi-definite) covariance functions (\autoref{sec:truncorr}).

        \item ``Collapse'' the recursion at various depths to reduce the combinatorial explosion, in a way which again preserves positive semi-definiteness and bounding properties (\autoref{sec:ptruncorr}).

        \item Compute two steps of the recursion analytically in closed form instead of actually branching, to get two tree levels ``for free'' (\autoref{sec:special2}).

        \item Use slow-to-compute bounds to measure empirically the accuracy of a more approximate fast-to-compute estimate (\autoref{sec:finalkernel}).

    \end{enumerate}
    I now describe in detail each of these in turn.
    
    \subsection{Truncated correlation function}
    \label{sec:truncorr}
    
    As first step in constructing a computationally tractable approximation to the sub-correlation function $k_d$ (defined in \autoref{th:corr}), I truncate the trees at depth $D$, i.e., stop the recursion after $D-d$ steps. The BART prior is designed to favor shallow trees, and the well-functioning of its MCMC scheme relies on the trees being shallow, so this should give a good approximation.

    By property~\ref{it:bounds}, \autoref{th:corrprop}, $k_D \in [1 - P_D, 1]$, so I stop the recursion by replacing the term $k_D$ with $1 - (1 - \gamma) P_D$, $\gamma \in [0, 1]$, an interpolation between the lower and upper bounds. I handle separately the $\mathbf n^0=\mathbf 0$ case, since I want the function to represent a correlation function, which must yield~1 in that case.

    \begin{definition}
        \label{def:truncorr}

        The \emph{truncated sub-correlation function} $k^D_{d,\gamma}$, with $d \le D$, is defined recursively as
        \begin{align}
            k^D_{d,\gamma}(\mathbf 0, \mathbf 0, \mathbf 0) &= 1, \\
            k^D_{D,\gamma}(\mathbf n^-, \mathbf 0, \mathbf n^+) &= 1, \label{eq:n0e0base} \\
            k^D_{D,\gamma}(\mathbf n^-, \mathbf n^0\ne \mathbf0, \mathbf n^+) &= 1 - (1 - \gamma) P_D, \label{eq:n0ne0base} \\
            k^{D>d}_{d,\gamma}(\nvecs) &= 1 - P_d \Bigg[1 - \frac1{W(\mathbf n)}
                \sum_{\substack{i=1 \\ n_i\ne 0}}^p \frac{w_i}{n_i} \Bigg( \notag \\
                    &\qquad \sum_{k=0}^{n^-_i - 1}
                    k^D_{d+1,\gamma}(\mathbf n^-_{n^-_i=k}, \mathbf n^0, \mathbf n^+)
                    + \sum_{k=0}^{n^+_i - 1}
                    k^D_{d+1,\gamma}(\mathbf n^-, \mathbf n^0, \mathbf n^+_{n^+_i=k})
                \Bigg)
            \Bigg]. \label{eq:truncorr}
        \end{align}

    \end{definition}

    The following theorems show that $k^D_{d,\gamma}$ can be used to approximate or place bounds on $k_d$, with accuracy and computational cost that increase with $D$. Intuitively, the truncation of the recursion corresponds to a variant of the BART model with a limit on the depth of the trees, so $k^D_{d,\gamma}$ shares the qualitative properties of $k_d$, and is practically equivalent at high $D$. See \autoref{sec:truncproof} for the proofs.

    \begin{theorem}
        As $\gamma$ is varied in $[0,1]$, $k^D_{d,\gamma}$ linearly interpolates between a lower and upper bound on $k_d$. \label{th:truncinterp}
    \end{theorem}

    \begin{theorem}
        The bounding interval by truncation to $D' > D$ is contained in the one with~$D$, i.e., $[k^{D'}_{d,0}, k^{D'}_{d,1}] \subseteq [k^D_{d,0}, k^D_{d,1}]$. \label{th:filter}
    \end{theorem}

    \begin{theorem}
        $k^D_{d,\gamma}$ is a valid correlation function, and the properties of \autoref{th:corrprop} hold for $k^D_{d,\gamma}$ as well, apart when referring to a level deeper than $D$. \label{th:truncprop}
    \end{theorem}

    \subsection{Pseudo-recursive truncated correlation function}
    \label{sec:ptruncorr}

    Due to the recursive iteration over all possible decision rules, the calculation of $k^D_{d,\gamma}$ is exponential in $D$ with a high base, so it is still unfeasible but for the lowest $D$. As a further step in making the approximation faster to compute, I modify the recursion to ``restart'' at various levels: at some pre-defined depths, the arguments to the recursive invocation of the function are set to the initial arguments, instead of the modified values handed down by preceding levels.

    Computationally, this allows to share the same function value across branches of the recursion, effectively limiting the combinatorial growth to the depth spans without resets. It also allows to share terms between different levels, although the utility of this is not evident at this point of the exposition.

    \begin{definition}
        \label{def:prtruncorr}
        The \emph{pseudo-recursive truncated sub-correlation function} $k^{(D_1,\ldots,D_r)}_{d,\gamma}$, with $D_1 < \ldots < D_r$ and $d\le D_r$, is defined as $k^{D_r}_{d,\gamma}$, but in the recursion the functions $k^{(D_1,\ldots,D_r)}_{D_1,\gamma}, \ldots, k^{(D_1,\ldots,D_r)}_{D_{r-1},\gamma}$ are evaluated at the initial arguments $\nvecs$ determined using the full splitting grid, instead of the restricted ones handed down by the recursion:
        \begin{align}
            k^{(D_1,\ldots,D_r)}_{d,\gamma}(\nvecs)
            &= \tilde k^{(D_1,\ldots,D_r)}_{d,\gamma}(\nvecs;\nvecs), \\
            \tilde k^{(D_1,\ldots,D_r)}_{d,\gamma}(\mathbf 0, \mathbf 0, \mathbf 0; \barnvecs)
            &= 1, \\
            \tilde k^{(D_1,\ldots,D_r)}_{D_r,\gamma}(\mathbf n^-, \mathbf 0, \mathbf n^+; \barnvecs)
            &= 1, \label{eq:prtruncorrbasen00} \\
            \tilde k^{(D_1,\ldots,D_r)}_{D_r,\gamma}(\mathbf n^-, \mathbf n^0\ne \mathbf0, \mathbf n^+; \barnvecs)
            &= 1 - (1 - \gamma) P_{D_r}, \label{eq:prtruncorrbasen0nz} \\
            \tilde k^{(D_1,\ldots,D_r)}_{d<D_r,\gamma}(\nvecs;\barnvecs)
            &= 1 - P_d \Bigg[1 - \frac1{W(\mathbf n)}
                \sum_{\substack{i=1 \\ n_i\ne 0}}^p \frac{w_i}{n_i} \Bigg( \notag \\
                    &\hspace{-25ex} \sum_{k=0}^{n^-_i - 1} \begin{cases}
                        \tilde k^{(D_1,\ldots,D_r)}_{d+1,\gamma}(\barnvecs;\barnvecs)
                            & \text{if $d+1 \in \{D_1, \ldots, D_{r-1}\}$} \\
                        \tilde k^{(D_1,\ldots,D_r)}_{d+1,\gamma}(\mathbf n^-_{n^-_i=k}, \mathbf n^0, \mathbf n^+;\barnvecs)
                            & \text{otherwise}
                    \end{cases}
                    + {} \notag \\
                    &\hspace{-25ex} \sum_{k=0}^{n^+_i - 1} \begin{cases}
                        \tilde k^{(D_1,\ldots,D_r)}_{d+1,\gamma}(\barnvecs;\barnvecs)
                            & \text{if $d+1 \in \{D_1, \ldots, D_{r-1}\}$} \\
                        \tilde k^{(D_1,\ldots,D_r)}_{d+1,\gamma}(\mathbf n^-, \mathbf n^0, \mathbf n^+_{n^+_i=k};\barnvecs)
                            & \text{otherwise}
                    \end{cases}
                \Bigg)
            \Bigg]. \label{eq:ptruncorr}
        \end{align}
        For convenience, I define the shorthand
        \begin{equation}
            k^{D_0,r}_{d,\gamma} = k^{(D_0, 2D_0,\ldots,rD_0)}_{d,\gamma}.
        \end{equation}
    \end{definition}

    This corresponds to a modification of the BART model in which at the reset depths the decision rules are not bound to be consistent with the decision rules of ancestor nodes \autocite[which was already used more extensively in][prop.~1]{linero2017}, so the result is still a valid correlation function with the same qualitative properties, as formally stated in the next theorem.

    \begin{theorem}
        $k^{(D_1,\ldots,D_r)}_{d,\gamma}$ is a valid correlation function. \label{th:prtruncorrpos}
    \end{theorem}

    The following theorems show that $k^{D,r}_{d,1}$ improves on $k^D_{d,1}$ as upper bound on $k_d$. Unfortunately, $k^{D,r}_{d,0}$ is not a lower bound, because resetting the set of possible decision rules tends to increase the correlation, as the presence of logically empty leaves reduces the total effective number of leaves and thus also the probability of separating any two points.

    \begin{theorem}
        $k^{(D_1,\ldots,D_r)}_{d,1}$ is an upper bound on $k_d$. \label{th:prtruncorrup}
    \end{theorem}

    \begin{theorem}
        $k^{(D_1,\ldots,D_r)}_{d,1} \le k^{(D_1,\ldots,D_{r'})}_{d,1}$ if $r\ge r'$, in particular $k^{D,r}_{d,1} \le k^{D,r'}_{d,1}$ and $k^{D,r}_{d,1} \le k^{D}_{d,1}$. \label{th:prtruncorrupbetter}
    \end{theorem}

    For the proofs, see \autoref{sec:ptruncproof}.

    \subsection{Special-casing of one recursion level}

    Since the recursion step (Equations~\ref{eq:corr}, \ref{eq:truncorr}, \ref{eq:ptruncorr}) is arithmetically simple, expanding the recursion (i.e., plugging the recursive definition within itself one or more times) and simplifying the resulting expression looks promising as a way to make the calculation faster.
    
    For starters, I plug the base case into the recursive step: if $d = D - 1$, $k^D_{d,\gamma}$ (\autoref{def:truncorr}) reduces to
    \begin{align}
        k^D_{D-1,\gamma}(\nvecs) &= 1 - P_{D-1} \Bigg(
            1 - k^D_{D,\gamma} \Bigg(1 -
            \frac 1 {W(\mathbf n)}
            \sum_{\substack{i=1 \\ n_i\ne 0}}^p
            w_i \frac{n^0_i}{n_i} \Bigg)
        \Bigg), \label{eq:depth1} \\
        \text{with } k^D_{D,\gamma} &= \begin{cases}
            1 - (1 - \gamma) P_D & \mathbf n^0 \ne \mathbf 0, \\
            1                    & \mathbf n^0   = \mathbf 0.
        \end{cases}
    \end{align}
    It is worthwhile to look at the bounds given by $\gamma = 0, 1$ in this case:
    \begin{align}
        k_d(\nvecs) &\in [k^{d+1}_{d,0}(\nvecs), k^{d+1}_{d,1}(\nvecs)] = \notag \\
        &= \Bigg[
            1 - \frac{P_d}{W(\mathbf n)}
            \sum_{\substack{i=1 \\ n_i \ne 0}}^p
            w_i \left(
                \frac{n^0_i}{n_i} + P_{d+1} \left( 1 - \frac{n^0_i}{n_i} \right)
            \right), \notag \\
            &\phantom{{} \in \Bigg[}
            1 - \frac{P_d}{W(\mathbf n)}
            \sum_{\substack{i=1 \\ n_i \ne 0}}^p
            w_i \frac{n^0_i}{n_i}
        \Bigg]. \label{eq:bounds1}
    \end{align}
    First, note that the upper bound in \autoref{eq:bounds1} is equal to the $P_{d+1} = 0$ case of property~\ref{it:nointer}, \autoref{th:corrprop}. This means that the no-interactions limit of BART is an upper bound on the general correlation function. Second, consider the width of the interval:
    \begin{align}
        \Delta_d^{d+1}(\nvecs) = \frac{P_d P_{d+1}}{W(\mathbf n)}
        \sum_{\substack{i=1 \\ n_i \ne 0}}^p
        w_i \left( 1 - \frac{n^0_i}{n_i} \right),
    \end{align}
    where I defined $\Delta_d^D = k^D_{d,1} - k^D_{d,0}$. The term $1 - n^0_i/n_i$ is proportional to how close the two points are. This suggests that the maximum estimation error in using $k^D_{d,\gamma}$ as approximation of $k_d$ is low for correlations between faraway points and is maximum for almost overlapping points.

    \paragraph{Usage with the pseudo-recursion}

    With $k^D_{D-1,\gamma}$, I am considering an expansion only at the bottom level of the recursion. However, in the pseudo-recursive function $k^{(D_1,\ldots,D_r)}_{d,\gamma}$ (\autoref{def:prtruncorr}), it is possible to use this expansion at all intermediate reset levels $D_1,\ldots,D_{r-1}$. To see this, consider that in \autoref{eq:ptruncorr}, if $d+1 \in \{D_1,\ldots,D_{r-1}\}$, the terms of the inner summations are all identical, and so can be collected exactly like in \autoref{eq:depth1}.

    \subsection{Special-casing of two recursion levels}
    \label{sec:special2}

    The recursive step (\autoref{eq:corr}) has two stacked summations, one along predictors and one along the cutpoints. With the convention of placing cutpoints between observed predictor values, the cost of the operation (taking the subsequent levels as given) is $O(np)$. In GP regression, the correlation function is invoked to compute each entry of the $n\times n$ covariance matrix, for a total cost of $O(n^3p)$. This complexity is prohibitive because it's higher than the already expensive $O(n^2\max(n,p))$ bottleneck (see \autoref{sec:gprecap}). Thus I need to not use the recursion step even once.

    The one-level expansion (\autoref{eq:depth1}) is a sum that separates predictors, so it represents a model without interactions. This is probably too rigid as approximation, even if stacked with the pseudo-recursion as $k^{1,r}_{d,\gamma} = k^{(1,2,3,\ldots,r)}_{d,\gamma}$. So I want to ``compress'' at least two recursion levels, expanding the recursion step within itself and simplifying until I get an $O(p)$ expression. Fortunately, this is possible, as stated in the next theorem.

    \begin{theorem}
        \label{th:depth2}
        $k^D_{D-2}$ admits the following non-recursive expression:
        \newcommand{\somespace}{\hspace{-6ex}}
        \begin{align}
            k^D_{D-2}(\mathbf n^-, \mathbf 0, \mathbf n^+) &= 1, \\
            k^D_{D-2}(\mathbf n^-,
                      \underbrace{\mathbf n^0}_{\mathclap{{}\ne \mathbf 0}},
                      \mathbf n^+) &=
            1 - P_{D-2} \Bigg[1 - \frac 1 {W(\mathbf n)} \Bigg[
                (1 - P_{D-1}) S +
                P_{D-1} k^D_D
                \sum_{\substack{i=1 \\ n_i \ne0}}^p \frac {w_i} {n_i} \Bigg( \notag \\
                    & \phantom{{}+{}}
                    \left(S + w_i \frac {n^0_i} {n_i}\right)
                    \left(
                        \frac 1 {W(\mathbf n_{n^-_i=0})} +
                        \frac 1 {W(\mathbf n_{n^+_i=0})} +
                        \frac {n^-_i + n^+_i - 2} {W(\mathbf n)}
                    \right)
                    + {} \notag \\
                    & {} +
                    \frac {w_i} {W(\mathbf n_{n^-_i=0})}
                    \left(\left\{
                        n^0_i + n^+_i
                        \,\middle|\,
                        \frac {n^+_i} {n^0_i + n^+_i}
                    \right\} - 1\right) + {} \notag \\
                    & {} +
                    \frac {w_i} {W(\mathbf n_{n^+_i=0})}
                    \left(\left\{
                        n^0_i + n^-_i
                        \,\middle|\,
                        \frac {n^-_i} {n^0_i + n^-_i}
                    \right\} - 1\right)
                    + {} \notag \\
                    & {} -
                    \frac {w_i n^0_i} {W(\mathbf n)}
                    \big(
                        2\psi(n_i) -
                        \psi(1 + n^0_i + n^-_i) -
                        \psi(1 + n^0_i + n^+_i)
                    \big)
            \Bigg)\Bigg]\Bigg], \label{eq:depth2} \\
            \intertext{where}
            k^D_D &= 1 - (1 - \gamma) P_D, \\
            S &= \sum_{\substack{i=1 \\n_i \ne 0}}^p
            w_i \left(1 - \frac{n^0_i} {n_i} \right), \\
            \{ x \mid E \} &= \begin{cases}
                E & x > 0, \\
                0 & \text{$x = 0$, even if $E$ is not well defined,}
            \end{cases}
        \end{align}
        and $\psi$ is the digamma function \autocite[\S5.15]{dlmf}.
    \end{theorem}

    \begin{proof}
        See \autoref{sec:depth2}.
    \end{proof}

    Like the one-level expansion of the previous section, this two-level expansion can be used at all reset depths in the pseudo-recursive version, allowing to compute $k^{2,r}_{d,\gamma} = k^{(2,4,6,\ldots,2r)}_{d,\gamma}$ without actual recursions.

    \paragraph{Computational optimizations}

    Even if \autoref{eq:depth2} is $O(p)$, in practice the computation time of the expression as written is too high because of its sheer length and the usage of a special function. A lot can be gained by carefully re-arranging the operations. The formula is typically applied to compute a covariance matrix, so there is a list of points $X = (\mathbf x_1, \ldots, \mathbf x_n)$ and the formula is evaluated at all possible pairs $(\mathbf x, \mathbf x') \in X \times X$. I state the following without proof: even though the calculation is overall $O(n^2p)$, the terms that involve only one of $\mathbf x$ or $\mathbf x'$ instead of both require only $O(np)$, and those that involve neither $O(p)$ or $O(1)$; by re-arranging the expression and re-writing it in terms of $(\mathbf x, \mathbf x')$ instead of $(\nvecs)$, only a small number of terms actually require $O(n^2p)$, and they are simple operations. This kind of optimizations is common in GP regression \autocite[see, e.g.,][\S A.2]{epperly2024}.
    
    \subsection{Final choice of estimate and its accuracy}
    \label{sec:finalkernel}

    Putting together the results of the previous sections, I have that $k^2_{0,0}$ and $k^{2,r}_{0,1}$ are efficiently computable lower and upper bounds on $k$, and are themselves valid (p.s.d.) correlation functions.

    To get a single estimate of the correlation function, I could set a constant interpolation coefficient $\gamma\in[0,1]$ to pick a value $(1-\gamma)k^2_{0,0} + \gamma k^{2,r}_{0,1}$ within the bounds. However, some numerical tests show that the upper bound $k^{2,r}_{0,1}$ is already pretty close to the true value, and that $r=5$ is sufficient to get most of the gain from increasing $r$, so to make the computation faster and simpler, I decide to use directly $k^{2,5}_{0,1}$ as estimate.

    To measure the accuracy of $k^{2,5}_{0,1}$ as estimate of $k_0$, in \autoref{sec:interpolation} I compute narrow bounding intervals at high depth (spending a lot of computation) on a range of values of hyperparameters, split configurations, and pairs of points, and compare it with the estimate to determine a maximum error. The result is that the correlation thus computed is accurate to 0.007 with $p=1$ predictors, and to 0.0005 with $p=10$, relatively to the interval $[1 - \alpha, 1]$, at the default BART hyperparameters $\alpha=0.95$, $\beta=2$. These figures are not proven bounds, because of the many arbitrary choices in this empirical verification.
    
    \paragraph{Numerical cross checks}

    There are various error-prone parts in the calculation of the correlation function, in particular all the arbitrary details of the BART model (\autoref{sec:bart}) and the cumbersome calculation of the two-level expansion (\autoref{th:depth2}). To make the results dependable, I want to be confident there is not even one mistake, not in the math nor in the implementation. I assure this in two ways. First, I check many self-consistency properties of the correlation function that do not need an oracle reference value, like the properties in \autoref{th:corrprop}. Second, I sample from the BART prior by explicitly generating the trees, compute the sample covariance matrix, and compare it with my covariance function. All these tests work out fine, see \autoref{sec:crosschecks} for the details.

    \section{GP regression with the BART kernel}
    \label{sec:empirical}

    Due to the splitting grid in predictor space being finite, the BART prior over the sum of trees is effectively a distribution over a finite-dimensional vector, where each component is the value of the regression function in one cell of the grid. The regression function is a sum of $m$ a priori i.i.d.\ terms, one per tree. Then, due to the multivariate CLT \autocite[16]{vandervaart1998}, as the number of trees tends to infinity, the prior converges to a multivariate Normal distribution, and so regression with the BART model becomes equivalent to GP regression.
    
    Using the efficiently computable approximation $k^{2,5}_{0,1}$ of the BART prior correlation function derived in the previous section, I can implement BART with an infinite number of trees as a GP regression. This is interesting both as an exploration of the BART model, and as a potential surrogate to use in place of BART.

    In \autoref{sec:benchmarkmain}, I compare the predictive performance of BART against the one of its infinite trees limit, in various configurations. In \autoref{sec:acic}, I use the GP surrogate to replace the BART components of BCF, a causal inference model, and compare its accuracy at inferring an average causal effect against other competitors in a simulated data competition.

    \subsection{Out-of-sample predictive accuracy on real data}
    \label{sec:benchmarkmain}

    I take the 42 datasets of the original BART article \autocite{chipman2010, kim2007}, and compare the predictions of various regression methods on held-out test sets, with 20 random 5:1 train/test splits for each dataset. The dataset sizes range approximately from $n\approx 100$ to $n\approx 5000$, and the number of features from $p=3$ to $p\approx 70$ (with dummies).

    The regression methods I consider are:
    \begin{itemize}
        \item \emph{MCMC}: standard BART.
        \item \emph{MCMC-CV}: BART cross-validated as in \textcite{chipman2010}, tuning parameters $\nu$, $q$, $k$, $m$.
        \item \emph{MCMC-XCV}: BART cross-validated with the pre-packaged routine \texttt{xbart} provided by \texttt{dbarts} \autocite{dorie2024}, tuning parameters $\alpha$, $\beta$, $m$.
        \item \emph{GP}: the infinite trees limit of \emph{MCMC} as GP regression.
        \item \emph{GP-hyp}: like \emph{GP} but tuning the parameters $\alpha$, $\beta$, $k$.
        \item \emph{GP+MCMC}: standard BART with parameters set to those found by \emph{GP-hyp}.
    \end{itemize}
    I take care to set up \emph{MCMC} and \emph{GP} such that \emph{GP} is the infinite trees limit of \emph{MCMC}, with all other model details identical. I also make an effort to accurately reproduce the benchmark in \textcite{chipman2010} with \emph{MCMC-CV}. For the other methods, I take more freedom in the configuration to obtain good performance in various ways.

    I compare the methods by looking at the distribution across datasets of RMSE and log-loss for each method, shown in \autoref{fig:comparison4}. Compared to the RMSE, the log-loss is a more stringent metric that takes into account not just the central value but the complete joint distribution of the predictions, weighing together inaccuracy, overconfidence, underconfidence, and mismatch of the correlational structure.
    \begin{figure}
        
        \widecenter{\includempl{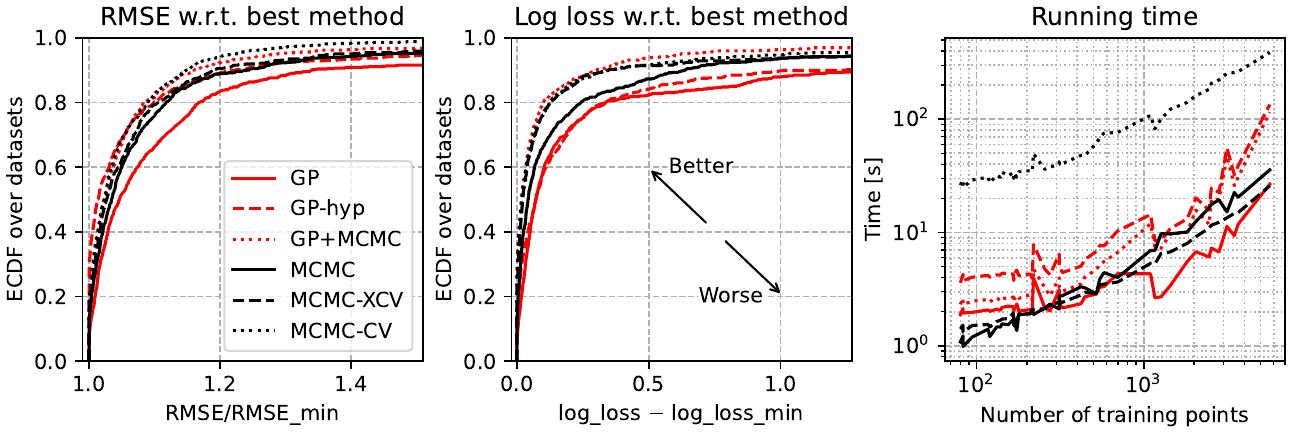}}
        
        \caption{\label{fig:comparison4} Comparison between standard BART (so with a finite number of trees and implemented with MCMC), and BART with an infinite number of trees as GP regression, in various configurations. On each dataset, methods are compared on their predictions on a held-out test set. I use the RMSE and log loss ($-\log p_\text{model}(\mathbf Y_\text{test}=\mathbf y_\text{test,true}|\mathbf Y_\text{train})/n_\text{test}$) as performance metrics. For each dataset, and for each metric, I take the best-performing method, and rescale/shift the metric relative to that. The plots show the distribution of relative metrics over datasets and 20 random 5:1 train/test splits per dataset. The plots are truncated, the relative RMSE maxima range from 1.9 (\emph{MCMC-CV}) to 5.1 (\emph{GP+MCMC}), the relative log-loss maxima from 9.8 (\emph{GP+MCMC}) to 190 (\emph{GP-hyp}). The number of predictors $p$ ranges from 3 to 67. See \autoref{sec:benchmarkmain}.}
        
    \end{figure}
    I interpret the results as follows:
    \begin{itemize}
        \item \emph{MCMC} is better than \emph{GP}, so the infinite trees limit of BART, without other modifications, is worse than the original. This is loosely in line with what reported by \textcite[\S5.2, p.~554]{linero2017} (who used a looser approximation of the kernel, see \autoref{sec:laplacekernel}), with the analogous result for the GP limit of wide Neural networks \autocite[\S5]{arora2019}, and with general results on Gaussian processes (see \autoref{sec:gpbad}).

        \item \emph{GP-hyp} is overall almost as good as \emph{MCMC}. This is a fairer comparison in some sense, as tuning GP hyperparameters is standard practice. In the BART vs.\ GP discussion in \textcite[1052]{hahn2020}, they observe that choosing the right kernel in GP regression is important, that BART is a complex mixture of Gaussian processes, and that BART empirically seems to do better than GP by adapting to the data covariance structure. (In this optic, BART is already a GP regression, but with a huge number of hyperparameters.) This result suggests that tuning a few hyperparameters with the right kernel may be sufficient to make a simple GP as flexible as BART.

        \item The top performers are \emph{MCMC-CV} and \emph{GP+MCMC}. It is known that tuning the hyperparameters improves BART \autocites[fig.~2]{chipman2010}[18]{imai2022}[445]{imai2013}[\S5.4 p.~55, \S6.2 p.~58]{dorie2019}.\marginpar{Here I'd like to cite the published imai2024 instead of the arxiv imai2022, but I can't access the published version so nope.} Interestingly, these two methods tune different hyperparameters and in a different way.

        \item The fact that \emph{GP+MCMC} works well means that the hyperparameters tuned for the infinite-trees GP transfer over to the finite-trees MCMC.

        \item If I also take into account convenience, I consider \emph{MCMC-XCV} the winner, as it runs fast and mostly works out of the box with little configuration, but achieves a good log-loss profile.
    \end{itemize}

    Further investigations (see \autoref{sec:anal}) show that \emph{GP-hyp} seems to be taking advantage of tuning $\alpha$ and $\beta$, the parameters that regulate the depth of the trees, employing a wide span of settings which go from depth one (no interactions) to very deep trees. This suggests that standard BART might benefit from a similar extreme exploration of tree depth settings, although it's not granted that the behavior should transfer over. \emph{MCMC-XCV} tunes those parameters as well, in a more limited fashion, but surprisingly there is no relationship at all between the values chosen by the two methods. I also find \emph{GP-hyp} does worse, with \emph{MCMC-CV} as reference, on low noise datasets where accurate predictions are possible, and better on noisy ones.

    Overall the GP-based methods are slower than BART, taking \emph{MCMC-XCV} as the BART champion. Apart from scaling worse with sample size, which is an (in this context) inherent limit of GP regression, they are also slower at low sample size. Detailed timings of my code indicate that the bottleneck of these GP regressions is evaluating the kernel (and its derivatives w.r.t.\ hyperparameters) to compute the prior covariance matrix, which has complexity $O(n^2p)$. This is not the most common bottleneck of GP regression, it is caused by my kernel $k^{2,5}_{0,1}$ being inordinately sophisticated.
    
    See \autoref{sec:benchmark} for the complete details of the models and of the analysis.

    \subsection{Inferential accuracy on synthetic data}
    \label{sec:acic}

    I build a GP version of Bayesian Causal Forests (BCF) \autocite{hahn2020}, replacing its two BART components with GP regression with the BART kernel. BCF is a BART-based method taylored to causal inference problems, in the basic observational setting where a causal effect $Z\rightarrow Y$ is estimated by adjusting for deconfounding covariates $X$ via regressions $Z|X$ and $Y|X$.

    I test this method on the pre-made synthetic datasets of the Atlantic Causal Inference Conference (ACIC) 2022 Data Challenge \autocite{thal2023}. I use the datasets of a competition, instead of generating them myself, because I expect them to be better that what I would make, and because I can immediately compare my method against the many participants. The goal of the competition is estimating average causal effects, both grand and within subgroups, in longitudinal data with the treatment $Z$ active from a certain time onwards. The study is not randomized, but the provided covariates are guaranteed to be sufficient for adjustment.

    The results are summarized in \autoref{fig:simcompact}.
    \begin{figure}
        \widecenter{\includempl{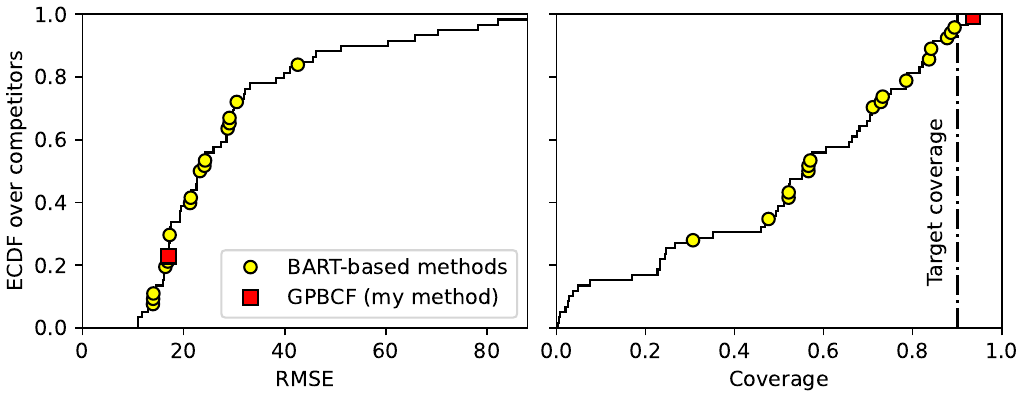}}
        \caption{\label{fig:simcompact} Comparison of a GP regression method based on my BART kernel against the competitors of the ACIC 2022 Data Challenge. The task is estimating a causal effect from synthetic data. Note: I did not participate in the competition, this is a post-competition comparison. The plots show the empirical cumulative distribution function (ECDF) across competitors of two performance metrics: root mean squared error (RMSE), and coverage of a bounding interval with target confidence level \SI{90}\%, for the estimation of the grand effect. See \autoref{sec:acic}.}
    \end{figure}
    The complete details are in \autoref{sec:simdata}. Overall, my GPBCF method performs in line with the other BART-based methods, including standard BCF\marginpar{I should really run a standard BCF together, to be sure I have the same feature engineering. The BCF in the competition may be poorly set up.}, providing evidence that in practice the infinite trees limit of BART is a functional regression method.

    In particular, GPBCF excels in having coverage close to the target confidence level, and overcovering instead of undercovering like almost all other competitors. I hypothesize this is due to the fact that I take into account the uncertainty over the hyperparameters by sampling their posterior distribution.\marginpar{I could actually investigate this by not sampling the hypers after having optimized them, instead of hypothesize.} This is made easy by the GP formulation: I find the mode with BFGS and do a simple Laplace approximation. The running time of the whole procedure is about 1 minute with 1000 datapoints.

    \section{Conclusions}
    \label{sec:conclusions}

    \paragraph{Results}

    I provide the first calculation of the prior covariance function of BART. With it, I show empirically that the infinite trees Gaussian process limit of BART is not as statistically performant as the standard finite-trees BART, in line with what was expected from the literature and with the analogous result for the GP limit of Neural networks. However, if I forsake carrying out a strict limit of the BART model, and instead tune the GP limit with standard GP methodology, I get a method with similar performance. Standard BART still has an edge if properly tuned.

    \paragraph{Applications}

    My GP limit of BART can be used right away in place of BART. My Python module for GP regression, \texttt{lsqfitgp} \autocite{petrillo2024c}, provides packaged BART-like and BCF-like GP regression routines, and a function to compute the prior covariance matrix and its derivatives for arbitrary usage. The GP version of BCF is potentially the most useful as standard BCF uses a Gibbs sampler on two BART components that reportedly can get stuck.\marginpar{Ask Linero for a citation on this, although I fear there isn't any.}

    \paragraph{Pro}

    The general advantage of the GP version is having an analytical likelihood function. This simplifies extending the model arbitrarily, and allows to do away with the MCMC algorithm, making the computation more reliable and hyperparameter tuning easier. It is nontrivial to adapt BART to settings different from the basic continuous or binary regression cases, and every new version requires a dedicated effort \autocite[consider, e.g.,][]{murray2021,hahn2020}. An arbitrary model can be more quickly implemented on top of a GP version with a generic likelihood-based optimization or sampling algorithm.\marginpar{I guess there exist better examples to show as complex BART extensions.}

    \paragraph{Con}

    A disadvantage is that GP regression does not scale to large datasets, and in this particular case the bottleneck of the computation is the evaluation of the kernel, so existing agnostic techniques that accelerate the covariance matrix decomposition do not help, while approximate techniques would probably reduce the statistical performance \autocite[fig.~18.4, p.~707]{murphy2023}. There is a scaling problem also w.r.t.{} the number of predictors~$p$, as the kernel calculation unconditionally involves all predictors, while the MCMC algorithm uses a stable number of predictors per cycle. Another disadvantage is that the predictive performance that can be obtained from standard BART is on average higher.

    \paragraph{Extensions}

    There are many straightforward further developments of this work, like making the GP versions of BART variants \autocite[e.g.,][]{linero2018b}, or scaling the inference to large datasets. However, I think the most interesting result is not the specific method I obtained, but showing in general that BART---a method considered to be of a different kind than Gaussian processes, with different use cases---is instead in a practical sense similar to an appropriate GP regression. This suggests to me that the important matters to investigate are:
    \begin{itemize}
        
        \item Whether the BART prior, ostensibly a complex Normal mixture, is actually well approximated by a simple Normal mixture; and if so, if this can be changed to improve BART.
        
        \item Developing new GP kernels in general: if it's possible to almost match BART with the constraint of using its own limit kernel with just three free hyperparameters, it should be possible to surpass it (see \autoref{sec:extrapolation}).

    \end{itemize}

    \section*{Acknowledgements}
    
    I thank Hugh Chipman for providing the data files used in \textcite{chipman2010}; and Antonio Linero, Fabrizia Mealli, Alessandra Mattei, Joseph Antonelli, Jared Murray and Francesco Stingo for useful comments and suggestions. I used the following open-source scientific software: Numpy \autocite{harris2020}, Scipy \autocite{scipy}, Matplotlib \autocite{matplotlib}. I did this work as a PhD candidate at the Department of Statistics, Computer Science, Applications (DISIA) ``G.\ Parenti'' of the University of Florence (UNIFI), Italy, and as a visiting researcher at the Department of Statistics and Data Science (SDS) of the University of Texas at Austin (UT Austin), USA.

    \printbibliography[heading=bibintoc]

    \clearpage
    
    \appendix
    
    \section{Proofs}

    \subsection{Proof of \autoref{th:corr}}
    \label{sec:corrproof}
    
    I follow the notation of \textcite{chipman2010} and \autoref{sec:bart}. First observe that, since the BART prior is a sum of i.i.d.\ trees, the covariance function can be computed on a single tree:
    \begin{align}
        \operatorname{Cov}[f(\mathbf x), f(\mathbf x')]
        &= \operatorname{Cov} \left[
            \sum_{j=1}^m g(\mathbf x;T_j,M_j),
            \sum_{l=1}^m g(\mathbf x';T_l,M_l)
        \right] = \notag \\
        &= \sum_{j=1}^m \sum_{l=1}^m
        \operatorname{Cov}[g(\mathbf x;T_j,M_j), g(\mathbf x';T_l,M_l)] = \notag \\
        &= \sum_{j=1}^m
        \operatorname{Cov}[g(\mathbf x;T_j,M_j), g(\mathbf x';T_j,M_j)] = \notag \\
        &= m \operatorname{Cov}[g(\mathbf x;T,M), g(\mathbf x';T,M)] = \notag \\
        &=: m \sigma_\mu^2 k(\mathbf x, \mathbf x'),
    \end{align}
    where, since $\sigma_\mu^2$ is the variance of all the leaves of the tree,
    $k(\mathbf x, \mathbf x')$ is the correlation function. To evaluate this
    function, consider that, if the two points $\mathbf x$ and $\mathbf x'$
    end up in the same cell of the stepwise function represented by
    the tree, i.e., in the same leaf, then they are assigned the same value,
    so their correlation is~1. Instead, if they are separated by any tree split,
    they are assigned to distinct leaves, and since the leaves are independent,
    their correlation is~0. The marginal correlation is thus
    \begin{align}
        k(\mathbf x, \mathbf x')
        &= 0 \cdot P(\text{the two points are separated by a split}) + 1 \cdot P(\text{not separated}) = \notag \\
        &= P(\text{$\mathbf x$ and $\mathbf x'$ not separated by a split}),
    \end{align}
    as stated in \textcite[\S5.2]{linero2017}. To continue the calculation,
    first introduce the counts of allowed splitting points $\mathbf n^-$,
    $\mathbf n^0$ and $\mathbf n^+$ respectively before, between and after
    $\mathbf x$ and $\mathbf x'$ along each predictor axis, totalling $\mathbf
    n = \mathbf n^- + \mathbf n^0 + \mathbf n^+$ (see \autoref{fig:counts}).
    These counts change while traversing the tree, since each children of a
    splitting node is assigned only a subset of the space, and thus a subset of
    the splitting points.
    
    Let $k_d(\mathbf n^-, \mathbf n^0, \mathbf n^+)$ be the probability that a
    subtree at depth $d$ does not separate the two points with its allowed
    splits, in particular $k(\mathbf x, \mathbf x') = k_0(\mathbf n^-, \mathbf
    n^0, \mathbf n^+)$. Then I proceed by expanding the probabilities
    following the recursive definition of the tree prior:
    \begin{align}
        k_d(\mathbf n^-, \mathbf n^0, \mathbf n^+)
        &= P(\text{the two points are \emph{not} separated by a split}) = \notag \\
        &= 1 - P(\text{the two points are separated by a split}) = \notag \\
        &= 1 - P(\text{nonterminal node}) P(\text{separated by a split}\mid\text{nonterminal}) = \notag \\
    \intertext{Now if $\mathbf n = \mathbf 0$ then $P(\text{nonterminal}) = 0$
    and thus the result is~1. The same holds if $\mathbf n \ne \mathbf 0$, but
    for all $n_i > 0$, the axis weight $w_i = 0$. In particular the following
    total weight must be nonzero:
    \begin{align}
        W(\mathbf n) = \sum_{\substack{i=1 \\ n_i\ne 0}}^p w_i,
    \end{align}
    I continue the calculation assuming $W(\mathbf n) > 0$ and so
    $P(\text{nonterminal}) = P_d$:}
        &= 1 - P_d \big[ 1 -
        P(\text{not separated by a split}\mid\text{nonterminal}) \big] = \notag \\
        &= 1 - P_d \left[ 1 - \sum_{\substack{i=1 \\ P(\text{split on axis $i$}) > 0}}^p
        P(\text{not separated by a split}\mid\text{split on axis $i$}) \right. \notag \\
        & \left. \phantom{=1 - P_d \left[ 1 - \sum_{\substack{i=1 \\ P(\text{split on axis $i$}) > 0}}^p\right.}
        P(\text{split on axis $i$}) \right], \label{eq:corrstart}
    \end{align}
    Now I consider separately the probabilities appearing in the last
    expression. The simplest one is
    \begin{align}
        P(\text{split on axis $i$})
        &= \begin{cases}
            \displaystyle \frac {w_i} {W(\mathbf n)} & n_i \ne 0, \\
            0                          & n_i = 0.
        \end{cases} \label{eq:spliti}
    \end{align}
    The other probability expands as
    \begin{align}
        &P(\text{not separated by a split}\mid\text{split on axis $i$}) = \notag \\
        &\quad= P(\text{not separated by the current split \&} \notag \\
        &\quad\phantom{=P(\,} \text{\& not by eventual descendants}\mid \text{split on axis $i$}) = \notag \\
        &\quad= \sum_{q \in \text{splits}}
        P(\text{not separated by current \& not by descendants}\mid \text{split on axis $i$ at $q$}) \cdot {} \notag \\
        &\quad\hphantom{= \sum_{q \in \text{splits}}\,\,}
        P(\text{split at $q$}\mid\text{on axis $i$}) = \notag \\
        &\quad= \sum_{q=0}^{n_i-1}
        P(\text{not separated by current}\mid \text{split on axis $i$ at $q$}) \cdot {} \notag \\
        &\quad\hphantom{= \sum_{q \in \text{splits}}\,\,}
        P(\text{not by descendants} \mid \text{not by current}, \text{split on axis $i$ at $q$}) \cdot {} \notag \\
        &\quad\hphantom{= \sum_{q \in \text{splits}}\,\,}
        P(\text{split at $q$} \mid \text{on axis $i$}) = \notag \\
        &\quad= \sum_{q=0}^{n_i-1}
        P(\text{the split falls out of $n^0_i$}\mid\text{split on axis $i$ at $q$}) \cdot {} \notag \\
        &\quad\hphantom{= \sum_{q \in \text{splits}}\,\,}
        P(\text{not by left nor right descendant}\mid\text{not by current}, \text{split on axis $i$ at $q$}) \cdot {} \notag \\
        &\quad\hphantom{= \sum_{q \in \text{splits}}\,\,}
        P(\text{split at $q$}\mid\text{on axis $i$}). \label{eq:condsplit}
    \end{align}
    
    Now consider separately the three probabilities appearing in the last
    expression. First:
    \begin{align}
        P(\text{the split falls out of $n^0_i$}\mid\text{split on axis $i$ at $q$})
        &= \begin{cases}
            1 & 0 \le q < n^-_i \lor n_i - n^+_i \le q < n_i, \\
            0 & \text{otherwise.}
        \end{cases} \label{eq:fallout}
    \end{align}
    Then, given that the current split is not separating, it must fall either
    in $n^+_i$ or $n^-_i$, and only one of the children can divide the two
    points again, so
    \begin{align}
        &P(\text{not by left nor right descendant}\mid\text{not by current}, \text{split on axis $i$ at $q$}) = \notag \\
        &\quad = \begin{cases}
            k_{d+1}(\mathbf n^-, \mathbf n^0, \mathbf n^+_{n^+_i\mapsto n^+_i-(n_i-q)}) \cdot 1 & n_i - n^+_i \le q < n_i, \\
            1 \cdot k_{d+1}(\mathbf n^-_{n^-_i\mapsto q}, \mathbf n^0, \mathbf n^+) & 0 \le q < n^-_i,
        \end{cases} \label{eq:leftright}
    \end{align}
    where when $q$ is in $n^+_i$, only the left subtree has a chance to
    separate again the two points, since the right one is assigned to the space
    to the right of both points, and viceversa for $q$ in $n^-_i$. Finally
    \begin{align}
        P(\text{split at $q$}\mid\text{on axis $i$})
        &= \frac 1 {n_i}, \label{eq:splitk}
    \end{align}
    and by putting together Equations~\ref{eq:corrstart}, \ref{eq:spliti},
    \ref{eq:condsplit}, \ref{eq:fallout}, \ref{eq:leftright}, \ref{eq:splitk},
    I obtain \autoref{eq:corr}:
    \begin{align}
        k_d(\mathbf 0, \mathbf 0, \mathbf 0) &=
        k_d((),(),()) = 1, \notag \\
        k_d(\nvecs) &= 1 - P_d \Bigg[1 - \frac1{W(\mathbf n)}
            \sum_{\substack{i=1 \\ n_i\ne 0}}^p \frac{w_i}{n_i} \Bigg( \notag \\
                &\qquad \sum_{k=0}^{n^-_i - 1}
                k_{d+1}(\mathbf n^-_{n^-_i=k}, \mathbf n^0, \mathbf n^+)
                + \sum_{k=0}^{n^+_i - 1}
                k_{d+1}(\mathbf n^-, \mathbf n^0, \mathbf n^+_{n^+_i=k})
            \Bigg)
        \Bigg], \quad \mathbf w > 0. \notag
    \end{align}
    
    To simplify the next proofs and calculations, I assume the weights positive. To handle zero weights, I fix the convention that all predictors with $w_i = 0$ shall be ignored. This gives the correct result, since at each recursive step the summation term would be multiplied by $w_i = 0$, and if $\mathbf w > 0$ then $W(\mathbf n) = 0 \iff \mathbf n = 0$.
    
    \subsection{Proof of \autoref{th:corrprop}}
    \label{sec:corrprop}
    
    \subsubsection{Property \ref{it:corr1}}
    
    Here I check that the recursive formula for the correlation (\autoref{th:corr}) satisfies $k_d(\mathbf x, \mathbf x) = 1$. Written in terms of splitting counts, this translates to $k_d(\mathbf n^-, \mathbf 0, \mathbf n^+) = 1$. I prove it by induction over $\mathbf n^-$ and $\mathbf n^+$. The base case, $k_d(\mathbf 0, \mathbf 0, \mathbf 0) = 1$, corresponds to the base case of the recursion. Now, assuming that the property holds for all counts up to some given $\mathbf n^-$ and $\mathbf n^+$, I show that it is valid also if one component of $\mathbf n^-$ is incremented by~1:
    \begin{align}
        k_d(\mathbf n^- + \underbrace{\mathbf 1_j}
            _{\mathclap{{}=(\delta_{1j}, \ldots, \delta_{pj})}},
        \mathbf 0, \mathbf n^+)
        &= 1 - P_d \Bigg[1 - \frac1{W(\mathbf n + \mathbf 1_j)}
            \sum_{\substack{i=1 \\ n_i + \delta_{ij}\ne 0}}^p
            \frac{w_i}{n_i + \delta_{ij}} \Bigg( \notag \\
                &\qquad \sum_{k=0}^{n^-_i - 1}
                k_{d+1}(\mathbf n^-_{n^-_i=k}, \mathbf 0, \mathbf n^+)
                + \sum_{k=0}^{n^+_i - 1}
                k_{d+1}(\mathbf n^-, \mathbf 0, \mathbf n^+_{n^+_i=k})
            \Bigg) - {} \notag \\
                &\qquad \frac{w_j}{W(\mathbf n + \mathbf 1_j)}
                \frac{k_{d+1}(\mathbf n^-, \mathbf 0, \mathbf n^+)}{n_j + 1}
        \Bigg] = \notag \\
        \intertext{(I have extracted from the summations the term with $i = j$
        and $k = n^-_j$)}
        &= 1 - P_d \Bigg[1 - \frac1{W(\mathbf n + \mathbf 1_j)}
            \sum_{\substack{i=1 \\ n_i + \delta_{ij}\ne 0}}^p
            \frac{w_i}{n_i + \delta_{ij}} \Bigg(
                \sum_{k=0}^{n^-_i - 1} 1 + \sum_{k=0}^{n^+_i - 1} 1
            \Bigg) - {} \notag \\
                &\phantom{{}= 1 - P_d \Bigg[1 - {}}
                \frac{w_j}{W(\mathbf n + \mathbf 1_j)} \frac1{n_j + 1}
        \Bigg] = \notag \\
        &= 1 - P_d \Bigg[1 - \frac1{W(\mathbf n + \mathbf 1_j)}
            \sum_{\substack{i=1 \\ n_i + \delta_{ij}\ne 0}}^p
            \frac{w_i}{n_i + \delta_{ij}} (\overbrace{n^+_i + n^-_i}^{{}=n_i} + \delta_{ij})
        \Bigg] = 1.
    \end{align}
    The case where $\mathbf n^+$ is incremented is analogous.

    \subsubsection{Property \ref{it:lower}}

    The recursive step applies, rather than the base case. So
    \begin{align}
        k_d(\mathbf 0, \mathbf n, \mathbf 0) &=
        1 - P_d\left[
            1 - \frac 1 {W(\mathbf n)} \sum_{\substack{i=0 \\ n_i\ne0}}^p
                \frac{w_i}{n_i}
                \underbrace{(0 + 0)}_\text{empty summations}
        \right] = \notag \\
        &= 1 - P_d.
    \end{align}
    
    \subsubsection{Property \ref{it:bounds}}
    \label{sec:bounds}
    
    I prove it by induction along the recursion. The recursion either bottoms down at the base case or when the inner summations are empty. The base case trivially is in $[1 - P_d, 1]$. Next, assuming that $k_{d+1} \in [1 - P_{d+1}, 1]$, I prove that $k_d \in [1 - P_d, 1]$, using the recursive step for $k_d$. The inner summations over $k$ have $n^+_i$ and $n^-_i$ terms in $[0, 1]$, so overall they are bounded in $[0, n^+_i + n^-_i]$. Since $n^+_i + n^-_i \le n_i$, the whole term of the outer summation over $i$ is in $[0, w_i]$, which in turn means that the complete summation is in $[0, 1]$. I remain with $1 - P_d(1 - \text{something in $[0, 1]$})$ which is in $[1 - P_d, 1]$.
    
    \subsubsection{Properties \ref{it:center} and \ref{it:corner}}
    
    Now I prove that $\mathbf n^0 = \mathbf n$ is a necessary condition for reaching the lower bound $k_d = 1 - P_d$, provided the nontermination probabilities are such that the bounds are not degenerate ($P_d > 0$) and that the subtree has a nonzero termination probability ($P_{d+1} < 1$).
    
    Assume $k_d(\nvecs) = 1 - P_d$. This implies that the outer summation over
    $i$ is equal to~0. Since all the terms are nonnegative, each one of them
    must be zero too, so under the assumption $\mathbf w > 0$ I have
    $k_{d+1}(\mathbf n^-_{n^-_i=k}, \mathbf n^0, \mathbf n^+) = 0$ and
    $k_{d+1}(\mathbf n^-, \mathbf n^0, \mathbf n^+_{n^+_i=k}) = 0$ respectively
    for all $k \in \mathbb N$ less than $n^-_i$ and $n^+_i$. Since $k_{d+1} \ge
    1 - P_{d+1} > 0$, this is not possible, so the only way to yield a zero
    summation is to have no terms, i.e., $n^+_i = n^-_i = 0$.
    
    Now instead consider $k_d = 1$, under the assumption $P_d > 0$. I prove
    that necessarily $\mathbf n^0 = 0$. The term which multiplies $P_d$ must
    be~0, so each term of the outer summation over $i$ must be equal to $w_i$.
    Since $k_{d+1} \le 1$ the summations over $k$ together yield at most $n^-_i
    + n^+_i$, so the only way to achieve the bound is to have $n^-_i
    + n^+_i = n_i$, which implies $n^0_i = 0$.
    
    \subsubsection{Properties \ref{it:monoext} and \ref{it:monoint}}
    
    Intuitively, points which are more distant w.r.t.\ the grid along all axes
    should have lower correlation. In fact, if the movement is split in two
    steps, first shrinking the grid outside of the points, then enlarging it
    between, this property holds separately for both variations:
    \begin{align}
        \mathbf n^{-\prime} &\le \mathbf n^-, \quad
        \mathbf n^{+\prime} \le \mathbf n^+ &&\implies&
        k_d(\mathbf n^{-\prime}, \mathbf n^0, \mathbf n^{+\prime}) &\le
        k_d(\nvecs), \label{eq:monoext} \\
        \mathbf n^{0\prime} &\ge \mathbf n^0 &&\implies&
        k_d(\mathbf n^-, \mathbf n^{0\prime}, \mathbf n^+) &\le k_d(\nvecs).
        \label{eq:monoint}
    \end{align}
    I prove the inequalities by recursion. The base case satisfies. Assuming
    \autoref{eq:monoext} holds for $k_{d+1}$, I prove it for $k_d$:
    \begin{align}
        &k_d(\mathbf n^{-\prime}, \mathbf n^0, \mathbf n^{+\prime}) = \notag \\
        &= 1 - P_d \Bigg[1 - \frac1{W(\mathbf n)}
            \sum_{\substack{i=1 \\ n'_i\ne 0}}^p \frac{w_i}{n'_i} \Bigg(
                \sum_{k=0}^{n^{-\prime}_i - 1}
                k_{d+1}(\mathbf n^{-\prime}_{n^-_i=k}, \mathbf n^0, \mathbf n^{+\prime})
                + \sum_{k=0}^{n^{+\prime}_i - 1}
                k_{d+1}(\mathbf n^{-\prime}, \mathbf n^0, \mathbf n^{+\prime}_{n^+_i=k})
            \Bigg)
        \Bigg] \le \notag \\
        \intertext{My goal is replace primed terms with unprimed ones. I
        apply the inequality to each $k_{d+1}$ term:}
        &\le 1 - P_d \Bigg[1 - \frac1{W(\mathbf n)}
            \sum_{\substack{i=1 \\ n'_i\ne 0}}^p \frac{w_i}{n'_i} \Bigg(
                \sum_{k=0}^{n^{-\prime}_i - 1}
                k_{d+1}(\mathbf n^-_{n^-_i=k}, \mathbf n^0, \mathbf n^+)
                + \sum_{k=0}^{n^{+\prime}_i - 1}
                k_{d+1}(\mathbf n^-, \mathbf n^0, \mathbf n^+_{n^+_i=k})
            \Bigg)
        \Bigg] \le \notag \\
        \intertext{I am left with primed terms in the inner summation ranges
        and in the denominator. Extending the summations preserves the
        inequality, while increasing the denominator has the opposite effect,
        so I have to deal with both at once. Since the function $x/(1+x)$ is
        increasing, with $x \propto n^-_i + n^+_i$, it is sufficient to observe
        that the terms added to the summations, i.e., those with
        $n^{\pm\prime}_i \le k < n^\pm_i$, are greater than the previous ones,
        which again derives from the inequality assumed on $k_{d+1}$. Thus:}
        &\le 1 - P_d \Bigg[1 - \frac1{W(\mathbf n)}
            \sum_{\substack{i=1 \\ n'_i\ne 0}}^p \frac{w_i}{n_i} \Bigg(
                \sum_{k=0}^{n^-_i - 1}
                k_{d+1}(\mathbf n^-_{n^-_i=k}, \mathbf n^0, \mathbf n^+)
                + \sum_{k=0}^{n^+_i - 1}
                k_{d+1}(\mathbf n^-, \mathbf n^0, \mathbf n^+_{n^+_i=k})
            \Bigg)
        \Bigg] \le \notag \\
        \intertext{Finally, for the condition on the outer summation, $n'_i >
        0$ implies $n_i > 0$, so the latter condition can only eventually add
        nonnegative terms compared to the first:}
        &\le 1 - P_d \Bigg[1 - \frac1{W(\mathbf n)}
            \sum_{\substack{i=1 \\ n_i\ne 0}}^p \frac{w_i}{n_i} \Bigg(
                \sum_{k=0}^{n^-_i - 1}
                k_{d+1}(\mathbf n^-_{n^-_i=k}, \mathbf n^0, \mathbf n^+)
                + \sum_{k=0}^{n^+_i - 1}
                k_{d+1}(\mathbf n^-, \mathbf n^0, \mathbf n^+_{n^+_i=k})
            \Bigg)
        \Bigg] = \notag \\
        &= k_d(\nvecs). \notag
    \end{align}
    
    Analogously, assuming \autoref{eq:monoint} holds for $k_{d+1}$:
    \begin{align}
        &k_d(\mathbf n^-, \mathbf n^{0\prime}, \mathbf n^+) = \notag \\
        &= 1 - P_d \Bigg[1 - \frac1{W(\mathbf n)}
            \sum_{\substack{i=1 \\ n'_i\ne 0}}^p \frac{w_i}{n'_i} \Bigg(
                \sum_{k=0}^{n^-_i - 1}
                k_{d+1}(\mathbf n^-_{n^-_i=k}, \mathbf n^{0\prime}, \mathbf n^+)
                + \sum_{k=0}^{n^+_i - 1}
                k_{d+1}(\mathbf n^-, \mathbf n^{0\prime}, \mathbf n^+_{n^+_i=k})
            \Bigg)
        \Bigg] \le \notag \\
        &\le 1 - P_d \Bigg[1 - \frac1{W(\mathbf n)}
            \sum_{\substack{i=1 \\ n'_i\ne 0}}^p \frac{w_i}{n'_i} \Bigg(
                \sum_{k=0}^{n^-_i - 1}
                k_{d+1}(\mathbf n^-_{n^-_i=k}, \mathbf n^0, \mathbf n^+)
                + \sum_{k=0}^{n^+_i - 1}
                k_{d+1}(\mathbf n^-, \mathbf n^0, \mathbf n^+_{n^+_i=k})
            \Bigg)
        \Bigg] \le \notag \\
        \intertext{Since $1/n'_i \le 1/n_i$, I can replace the denominator.
        This time the condition $n'_i > 0$ is less restrictive than $n_i > 0$,
        so the set of summation terms potentially gets larger. However, if $n_i
        = 0$ then necessarily $n^+_i = n^-_i = 0$, so the corresponding
        term is zero:}
        &\le 1 - P_d \Bigg[1 - \frac1{W(\mathbf n)}
            \sum_{\substack{i=1 \\ n_i\ne 0}}^p \frac{w_i}{n_i} \Bigg(
                \sum_{k=0}^{n^-_i - 1}
                k_{d+1}(\mathbf n^-_{n^-_i=k}, \mathbf n^0, \mathbf n^+)
                + \sum_{k=0}^{n^+_i - 1}
                k_{d+1}(\mathbf n^-, \mathbf n^0, \mathbf n^+_{n^+_i=k})
            \Bigg)
        \Bigg] = \notag \\
        &= k_d(\nvecs). \notag
    \end{align}

    \subsubsection{Property \ref{it:nointer}}

    If $P_{d+1} = 0$, then $k_{d+1} = 1 - P_{d+1} [\ldots] = 1$ with the recursive step, and $k_{d+1} = 1$ with the base case. Assuming $\mathbf n \ne \mathbf 0 \implies (\nvecs) \ne (\mathbf 0,\mathbf 0,\mathbf 0)$, the recursive step applies to $k_d$, so
    \begin{align}
        k_d(\nvecs) &= 1 - P_d \left[
            1 - \frac1{W(\mathbf n)} \sum_{\substack{i=0 \\ n_i\ne 0}}^{p} \frac{w_i}{n_i} \left(
                \sum_{k=0}^{n^-_i-1} 1 + \sum_{k=0}^{n^+_i-1} 1
            \right)
        \right] = \notag \\
        &= 1 - P_d \left[1 - \frac1{W(\mathbf n)} \sum_{\substack{i=0 \\ n_i\ne 0}}^{p} w_i \left( 1 - \frac{n^0_i}{n_i} \right) \right] = \notag \\
        &= 1 - P_d \left[1 - \frac{W(\mathbf n)}{W(\mathbf n)} + \frac1{W(\mathbf n)} \sum_{\substack{i=0 \\ n_i\ne 0}}^{p} w_i \frac{n^0_i}{n_i} \right] = \notag \\
        &= 1 - \frac{P_d}{W(\mathbf n)} \sum_{\substack{i=0 \\ n_i\ne 0}}^{p} w_i \frac{n^0_i}{n_i}.
    \end{align}

    \subsubsection{Property \ref{it:white}}

    The case $k_d = 1$ if $\mathbf n^0 = \mathbf 0$ is already covered by property~\ref{it:corr1}. If $\mathbf n^0 \ne \mathbf 0$, the recursive case applies, with the recursion bottoming when the inner summations are empty. Now I prove that $k_d=0$ when $P_d, P_{d+1}, \ldots = 1$. Base case:
    \begin{align}
        k_d(\mathbf 0, \mathbf n^0, \mathbf 0)
        &= 1 - \underbrace{1}_{P_d}\left[
            1 - \frac1{W(\mathbf n)} \sum_{\substack{i=0 \\ n_i\ne 0}}^{p} \frac{w_i}{n_i} (0 + 0)
        \right] = 0.
    \end{align}
    The calculation for the induction step is analogous to the one of the base case, with the zeros given by $k_{d+1}$ rather than the summations being empty. Finally:
    \begin{align}
        k_d(\nvecs) = 1 - P_d\left[1 - \frac1{W(\mathbf n)} \sum_{\substack{i=0 \\ n_i\ne 0}}^{p} \frac{w_i}{n_i} (0 + 0) \right]
        = 1 - P_d.
    \end{align}

    \subsection{Truncated correlation function}
    \label{sec:truncproof}
    
    \subsubsection{Proof of \autoref{th:truncinterp}}

    In the recursive step, $k_d$ depends linearly on the $k_{d+1}$ terms, and with nonnegative coefficients. This implies that the inlined expansion to depth $D$ of $k_d$ is linear and increasing w.r.t.\ the terms $k_D$. So, if I replace $k_D$ with a lower or upper bound on its value, as in the definition of $k^D_{d,\gamma}$, I get respectively a lower or upper bound on $k_d$. Moreover, by linearity, any interpolation between the two bounds at level $D$ corresponds to an interpolation at level $d < D$. In the special case $\mathbf n^0=\mathbf 0$, $k^D_{d,\gamma}=k_d=1$.

    \subsubsection{Proof of \autoref{th:filter}}
    
    I have to prove that \marginpar{Necessary and sufficient
    conditions under which these inequalities are strict? Sufficient I guess $P_d \in (0, 1)$, $\mathbf n^0 \ne \mathbf 0$.}
    \begin{align}
        D' \ge D \implies \begin{cases}
            k^{D'}_{d,1} \le k^D_{d,1}, \\
            k^{D'}_{d,0} \ge k^D_{d,0}.
        \end{cases}
    \end{align}
    
    It is sufficient to prove the property for $D' = D + 1$. It also sufficient
    to check the bottom case $d = D$, since the higher levels of the recursion
    are identical even if $D$ is different, and depend with positive
    coefficients on lower levels. This reduces the inequalities to the general
    bounds for the correlation function (property~\ref{it:bounds}, \autoref{th:corrprop}):
    \begin{align}
        k^{D+1}_{D,1} &\le 1 \phantom{{}-P_D}
        = k^D_{D,1}, \\
        k^{D+1}_{D,0} &\ge 1 - P_D = k^D_{D,0}.
    \end{align}

    \subsubsection{Proof of \autoref{th:truncprop}}
    \label{sec:truncpropproof}

    Due to property~\ref{it:white}, \autoref{th:corrprop}, $k^D_{d,\gamma}=k_d$ if in $k_d$ I replace $P_d \mapsto 1$ for $d > D$ and $P_D \mapsto (1 - \gamma) P_D$, thus it is p.s.d.\ and the same properties apply.

    \subsection{Pseudo-recursive truncated correlation function}
    \label{sec:ptruncproof}
        
    \subsubsection{Proof of \autoref{th:prtruncorrpos}}
    \label{sec:prtruncorrposproof}

    I'll prove that $k^{(D_1,\ldots,D_r)}_{d,\gamma}$ is p.s.d.\ by showing it stems from the correlation function of a variant of BART.

    Starting from BART, modify it by re-allowing all splits at levels $D_1$ to $D_{r-1}$: instead of forcing the subtrees to use only the regions delimited by the splits of the ancestor nodes, they start off with the initial complete set of splits. This may look wasteful because it produces redundant decision rules and empty leaves where no point ever falls, but it's formally valid.

    The correlation function of this process is obtained from the calculation in the proof of \autoref{th:corr} (\autoref{sec:corrproof}), with two modifications.

    First, like in \autoref{def:prtruncorr}, duplicate the arguments of the sub-correlation function, to ``remember'' the initial arguments: $\tilde k^{(D_1,\ldots,D_r)}_d(\nvecs;\barnvecs)$ is the probability that a subtree with root at depth $d$ does not separate the two points $\mathbf x$ and $\mathbf x'$, with $\nvecs$ the splits to use in decision rules, potentially restricted, and $\barnvecs$ all the splits. In particular the correlation function is
    \begin{equation}
        \operatorname{Corr}[f(\mathbf x), f(\mathbf x')] = \tilde k^{(D_1,\ldots,D_r)}_0(\nvecs;\nvecs).
    \end{equation}
    
    Second, modify \autoref{eq:leftright} as
    \begin{equation}
        \begin{cases}
            \tilde k^{(D_1,\ldots,D_r)}_{d+1}(\barnvecs; \barnvecs)
                & d+1 \in \{D_1,\ldots,D_{r-1}\} \land n_i - n^+_i \le q < n_i, \\
            \tilde k^{(D_1,\ldots,D_r)}_{d+1}(\barnvecs; \barnvecs)
                & d+1 \in \{D_1,\ldots,D_{r-1}\} \land 0 \le q < n^-_i, \\
            \tilde k^{(D_1,\ldots,D_r)}_{d+1}(\mathbf n^-, \mathbf n^0, \mathbf n^+_{n^+_i\mapsto n^+_i-(n_i-q)}; \barnvecs)
                & d+1 \notin \{D_1,\ldots,D_{r-1}\} \land n_i - n^+_i \le q < n_i, \\
            \tilde k^{(D_1,\ldots,D_r)}_{d+1}(\mathbf n^-_{n^-_i\mapsto q}, \mathbf n^0, \mathbf n^+; \barnvecs)
                & d+1 \notin \{D_1,\ldots,D_{r-1}\} \land 0 \le q < n^-_i.
        \end{cases}
    \end{equation}

    $\tilde k^{(D_1,\ldots,D_r)}_d$ has the right recursive step, but it's not yet equal to $\tilde k^{(D_1,\ldots,D_r)}_{d,\gamma}$ because it does not terminate at depth $D_r$ and it lacks the interpolation coefficient $\gamma$. To add these features, set $P_d \mapsto 1$ for $d > D_r$ and $P_{D_r} \mapsto (1 - \gamma) P_{D_r}$ like in the proof of \autoref{th:truncprop} (\autoref{sec:truncpropproof}). This finally yields $k^{(D_1,\ldots,D_r)}_{d,\gamma}$ as correlation function of the process.

    \subsubsection{Proof of \autoref{th:prtruncorrup}}

    In the recursive step, the sub-correlation function at depth $d$ always depends with nonnegative coefficients on the sub-correlation function at depth $d+1$. Thus replacing $k_{d+1}$ with an upper bound on it produces in turn an upper bound on $k_d$.

    Restricting the available splits keeps $\mathbf n^0$ unchanged while reducing $\mathbf n^\pm$, so by property~\ref{it:monoext}, \autoref{th:corrprop}, $k_d$ evaluated on the original unrestricted splits is greater than $k_d$ evaluated on restricted splits.

    Using these two facts, I can recursively modify $k_d$ in a non-decreasing way until it's equal to $k^{(D_1,\ldots,D_r)}_{d,1}$. The base case is $k_{D_r} \le 1 = k^{(D_1,\ldots,D_r)}_{D_r,1}$. I plug this bound into the recursive step for $k_{D_r - 1}$, obtaining an upper bound on it, and yielding $k^{(D_1,\ldots,D_r)}_{D_r-1,1}$, and so on until $k_{D_{r-1}-1}$. Here I replace the restricted $\nvecs$ with the initial ones, obtaining again an upper bound equal to $k^{(D_1,\ldots,D_r)}_{D_{r-1}-1,1}$. I continue until $d$.

    \subsubsection{Proof of \autoref{th:prtruncorrupbetter}}

    Consider $k^{(D_1,\ldots,D_r)}_{d,1}$ and $k^{(D_1,\ldots,D_{r'})}_{d,1}$, with $r'\le r$. Their recursive expansions are identical until hitting level $D_{r'}$. At that point, the first invokes $\tilde k^{(D_1,\ldots,D_r)}_{D_{r'},1}$, while the second $\tilde k^{(D_1,\ldots,D_{r'})}_{D_{r'},1}$, which is equal to 1 by Equations~\ref{eq:prtruncorrbasen00} and~\ref{eq:prtruncorrbasen0nz}. Since $\tilde k^{(D_1,\ldots,D_r)}_{D_{r'},1}\le 1$ because it is a correlation function (see \autoref{sec:prtruncorrposproof}), and the recursive step depends with nonnegative coefficients on the sub-correlations, the inequality back-propagates up to $k^{(D_1,\ldots,D_r)}_{d,1} \le k^{(D_1,\ldots,D_{r'})}_{d,1}$.

    \subsection{Proof of \autoref{th:depth2}}
    \label{sec:depth2}
    
    Here I derive the formula for $k^D_{D-2}$ of \autoref{eq:depth2}. In the
    following I assume $\mathbf n^0 \ne \mathbf 0$ since otherwise I can
    shortcut to $k^D_{D-2} = 1$. I start from \autoref{eq:truncorr}:
    \begin{align}
        &k^D_{D-2}(\nvecs) =
        1 - P_{D-2} \Bigg[
            1 - \frac 1 {W(\mathbf n)} \sum_{\substack{i=1 \\ n_i \ne 0}}^p
            \frac {w_i} {n_i} \Bigg( \notag \\
                & \phantom{{}={}}
                \sum_{k=0}^{n^-_i-1}
                k^D_{D-1}(\mathbf n^-_{n^-_i=k}, \mathbf n^0, \mathbf n^+) +
                \sum_{k=0}^{n^+_i-1}
                k^D_{D-1}(\mathbf n^-, \mathbf n^0, \mathbf n^+_{n^+_i=k})
            \Bigg)
        \Bigg] = \notag \\
        \intertext{and expand again $k^D_{D-1}$:}
        &= 1 - P_{D-2} \Bigg[
            1 - \frac 1 {W(\mathbf n)}
            \sum_{\substack{i=1 \\ n_i \ne 0}}^p \frac {w_i} {n_i} \Bigg( \notag \\
                & \phantom{{}={}}
                \sum_{k=0}^{n^-_i - 1} \Bigg[
                    1 - P_{D-1} \Bigg[
                        1 - \frac 1 {W(\mathbf n_{n^-_i=k})}
                        \sum_{\substack{j=1 \\ n_{j,n^-_i=k} \ne 0}}^p
                        \frac {w_j} {n_{j,n^-_i=k}} \Bigg(
                            \sum_{q=0}^{n^-_{j,n^-_i=k} - 1} k^D_D +
                            \sum_{q=0}^{n^+_j - 1} k^D_D
                        \Bigg)
                    \Bigg]
                \Bigg] + {} \notag \\
                & \phantom{{}={}}
                \sum_{k=0}^{n^+_i - 1} \Bigg[
                    1 - P_{D-1} \Bigg[
                        1 - \frac 1 {W(\mathbf n_{n^+_i=k})}
                        \sum_{\substack{j=1 \\ n_{j,n^+_i=k} \ne 0}}^p
                        \frac {w_j} {n_{j,n^+_i=k}} \Bigg(
                            \sum_{q=0}^{n^-_j - 1} k^D_D +
                            \sum_{q=0}^{n^+_{j,n^+_i=k} - 1} k^D_D
                        \Bigg)
                    \Bigg]
                \Bigg]
            \Bigg)
        \Bigg] = \notag \\
        \intertext{Note that, since $\mathbf n^0 \ne \mathbf 0$, $k^D_D = 1 -
        (1 - \gamma) P_D$ and $W(\mathbf n_{n^*_i=k}) > 0$. I continue
        simplyfing the expression:}
        &= 1 - P_{D-2} \Bigg[
            1 - \frac 1 {W(\mathbf n)}
            \sum_{\substack{i=1 \\ n_i \ne 0}}^p  \frac {w_i} {n_i} \Bigg( \notag \\
                & \phantom{{}={}}
                \sum_{k=0}^{n^-_i - 1} \Bigg[
                    1 - P_{D-1} +
                    \frac{P_{D-1} k^D_D} {W(\mathbf n_{n^-_i=k})}
                    \sum_{\substack{j = 1 \\ n_{j,n^-_i=k}\ne 0}}^p w_j
                    \frac {n^+_j + n^-_{j,n^-_i=k}} {n_{j,n^-_i=k}}
                \Bigg] + {} \notag \\
                & \phantom{{}={}}
                \sum_{k=0}^{n^+_i - 1} \Bigg[
                    1 - P_{D-1} +
                    \frac{P_{D-1} k^D_D} {W(\mathbf n_{n^+_i=k})}
                    \sum_{\substack{j=1 \\ n_{j,n^+_i=k}\ne 0}}^p w_j
                    \frac {n^+_{j, n^+_i=k} + n^-_j} {n_{j, n^+_i=k}}
                \Bigg]
            \Bigg)
        \Bigg] = \notag \\
        \intertext{I take out the terms with $j=i$ from the sums over $j$:}
        &= 1 - P_{D-2} \Bigg[
            1 - \frac 1 {W(\mathbf n)}
            \sum_{\substack{i=1 \\ n_i \ne 0}}^p  \frac {w_i} {n_i} \Bigg(
                \sum_{k=0}^{n^-_i - 1} (1 - P_{D-1}) +
                \sum_{k=0}^{n^+_i - 1} (1 - P_{D-1}) + {} \notag \\
                & \phantom{{}={}}
                \sum_{k=0}^{n^-_i - 1}
                \frac {P_{D-1} k^D_D} {W(\mathbf n_{n^-_i=k})} \Bigg[
                    \sum_{\substack{j=1,j \ne i \\ n_{j,n^-_i=k} \ne 0}}^p w_j
                    \frac {n^+_j + n^-_{j,n^-_i=k}} {n_{j,n^-_i=k}} +
                    w_i \left\{
                        n_{i,n^-_i=k} \,\middle|\,
                        \frac {n^+_i + n^-_{i,n^-_i=k}} {n_{i,n^-_i=k}}
                    \right\}
                \Bigg] + {} \notag \\
                & \phantom{{}={}}
                \sum_{k=0}^{n^+_i - 1}
                \frac {P_{D-1} k^D_D} {W(\mathbf n_{n^+_i=k})} \Bigg[
                    \sum_{\substack{j=1,j \ne i \\ n_{j,n^+_i=k} \ne 0}}^p w_j
                    \frac {n^+_{j,n^+_i=k} + n^-_j} {n_{j,n^+_i=k}} +
                    w_i \left\{
                        n_{i,n^+_i=k} \,\middle|\,
                        \frac {n^+_{i,n^+_i=k} + n^-_i} {n_{i,n^+_i=k}}
                    \right\}
                \Bigg]
            \Bigg)
        \Bigg] = \notag \\
        \intertext{Where $\{x \mid E\} = E$ if $x > 0$ and $0$ otherwise, even
        if $E$ is not well defined. So now I can remove all the exceptions
        of the kind $_{n^\pm_i=k}$:}
        &= 1 - P_{D-2} \Bigg[
            1 - \frac 1 {W(\mathbf n)}
            \sum_{\substack{i=1 \\ n_i \ne 0}} \frac {w_i} {n_i} \Bigg(
                (n^+_i + n^-_i) (1 - P_{D-1}) + {} \notag \\
                & \phantom{{}={}}
                \sum_{k=0}^{n^-_i - 1}
                \frac {P_{D-1} k^D_D} {W(\mathbf n_{n^-_i=k})} \Bigg[
                    \sum_{\substack{j=1,j \ne i \\ n_j \ne 0}}^p
                    w_j \frac {n^+_j + n^-_j} {n_j} +
                    w_i \left\{
                        n^+_i + n^0_i + k \,\middle|\,
                        \frac {n^+_i + k} {n^+_i + n^0_i + k}
                    \right\}
                \Bigg]
                 + {} \notag \\
                & \phantom{{}={}}
                \sum_{k=0}^{n^+_i - 1}
                \frac {P_{D-1} k^D_D} {W(\mathbf n_{n^+_i=k})} \Bigg[
                    \sum_{\substack{j=1,j \ne i \\ n_j\ne 0}}^p
                    w_j \frac {n^+_j + n^-_j} {n_j} +
                    w_i \left\{
                        k + n^0_i + n^-_i \,\middle|\,
                        \frac {n^-_i + k} {k + n^0_i + n^-_i}
                    \right\}
                \Bigg]
            \Bigg)
        \Bigg] = \notag \\
        \intertext{Now the only nontrivial summation terms are the fractions
        where $k$ appears. My goal is to write them as harmonic sums ($\sum
        1/k$) which can be written in terms of the digamma function $\psi$, so
        it would be convenient to remove the exception indicated by the braces.
        I can do this by extracting the $k = 0$ terms from the summations.
        This brings the added simplification that, if $k > 0$, necessarily
        $W(\mathbf n_{n^\pm_i=k}) = W(\mathbf n)$. I also need to remember that
        the $k = 0$ term is missing if the summation is empty, i.e., if $n^\pm_i
        = 0$:}
        &= 1 - P_{D-2} \Bigg[
            1 - \frac 1 {W(\mathbf n)}
            \sum_{\substack{i=1 \\ n_i \ne 0}}^p  \frac {w_i} {n_i} \Bigg(
                (n^+_i + n^-_i) (1 - P_{D-1}) + {} \notag \\
                & \phantom{{}={}}
                \Bigg\{
                    n^-_i \,\Bigg|\,
                    \frac {P_{D-1} k^D_D} {W(\mathbf n_{n^-_i=0})} \Bigg[
                        \sum_{\substack{j=1,j\ne i \\ n_j \ne 0}}^p
                        w_j \frac {n^+_j + n^-_j} {n_j} +
                        w_i \left\{
                            n^+_i + n^0_i \,\middle|\,
                            \frac {n^+_i} {n^+_i + n^0_i}
                        \right\}
                    \Bigg]
                \Bigg\} + {} \notag \\
                & \phantom{{}={}}
                \Bigg\{
                    n^+_i \,\Bigg|\,
                    \frac {P_{D-1} k^D_D} {W(\mathbf n_{n^+_i=0})} \Bigg[
                        \sum_{\substack{j=1,j\ne i \\ n_j \ne 0}}^p
                        w_j \frac {n^+_j + n^-_j} {n_j} +
                        w_i \left\{
                            n^0_i + n^-_i \,\middle|\,
                            \frac {n^-_i} {n^0_i + n^-_i}
                        \right\}
                    \Bigg]
                \Bigg\} + {} \notag \\
                & \phantom{{}={}}
                \sum_{k=1}^{n^-_i - 1}
                \frac {P_{D-1} k^D_D} {W(\mathbf n)} \Bigg[
                    \sum_{\substack{j=1,j\ne i \\ n_j \ne 0}}^p
                    w_j \frac {n^+_j + n^-_j} {n_j} +
                    w_i \frac {n^+_i + k} {n^+_i + n^0_i + k}
                \Bigg] + {} \notag \\
                & \phantom{{}={}}
                \sum_{k=1}^{n^+_i - 1}
                \frac {P_{D-1} k^D_D} {W(\mathbf n)} \Bigg[
                    \sum_{\substack{j=1,j\ne i \\ n_j\ne 0}}^p
                    w_j \frac {n^+_j + n^-_j} {n_j} +
                    w_i \frac {k + n^-_i} {k + n^0_i + n^-_i}
                \Bigg]
            \Bigg)
        \Bigg] = \notag \\
        \intertext{Notice there are many terms $(n^+_* + n^-_*)/n_*$ to be
        summed over. However, the terms in $j$ change as $i$ runs over its
        range because of the conditions $j\ne i$. I remove them by adding and
        subtracting the $j=i$ missing term:}
        &= 1 - P_{D-2} \Bigg[
            1 - \frac 1 {W(\mathbf n)}
            \sum_{\substack{i=1 \\ n_i \ne 0}}^p  \frac {w_i} {n_i} \Bigg(
                (n^+_i + n^-_i) (1 - P_{D-1}) + {} \notag \\
                & \phantom{{}={}}
                 \frac {P_{D-1} k^D_D} {W(\mathbf n_{n^-_i=0})} \Bigg\{
                    n^-_i \,\Bigg|\,
                    \sum_{\substack{j=1 \\ n_j \ne 0}}^p
                    w_j \frac {n^+_j + n^-_j} {n_j} -
                    w_i \frac {n^+_i + n^-_i} {n_i} +
                    w_i \left\{
                        n^+_i + n^0_i \,\middle|\,
                        \frac {n^+_i} {n^+_i + n^0_i}
                    \right\}
                \Bigg\} + {} \notag \\
                & \phantom{{}={}}
                \frac {P_{D-1} k^D_D} {W(\mathbf n_{n^+_i=0})} \Bigg\{
                    n^+_i \,\Bigg|\,
                    \sum_{\substack{j=1 \\ n_j \ne 0}}^p
                    w_j \frac {n^+_j + n^-_j} {n_j} -
                    w_i \frac {n^+_i + n^-_i} {n_i} +
                    w_i \left\{
                        n^0_i + n^-_i \,\middle|\,
                        \frac {n^-_i} {n^0_i + n^-_i}
                    \right\}
                \Bigg\} + {} \notag \\
                & \phantom{{}={}}
                \frac {P_{D-1} k^D_D} {W(\mathbf n)} \sum_{k=1}^{n^-_i - 1}
                \Bigg[
                    \sum_{\substack{j=1 \\ n_j \ne 0}}^p
                    w_j \frac {n^+_j + n^-_j} {n_j} -
                    w_i \frac {n^+_i + n^-_i} {n_i} +
                    w_i \frac {n^+_i + k} {n^+_i + n^0_i + k}
                \Bigg] + {} \notag \\
                & \phantom{{}={}}
                \frac {P_{D-1} k^D_D} {W(\mathbf n)} \sum_{k=1}^{n^+_i - 1}
                \Bigg[
                    \sum_{\substack{j=1 \\ n_j \ne 0}}^p
                    w_j \frac {n^+_j + n^-_j} {n_j} -
                    w_i \frac {n^+_i + n^-_i} {n_i} +
                    w_i \frac {k + n^-_i} {k + n^0_i + n^-_i}
                \Bigg]
            \Bigg)
        \Bigg] = \notag \\
        \intertext{So I define $\displaystyle S = \sum_{\substack{i=1 \\ n_i
        \ne 0}}^p w_i \frac {n^+_i + n^-_i} {n_i}$ and replace it throughout:}
        &= 1 - P_{D-2} \Bigg[
            1 - \frac {S (1 - P_{D-1})} {W(\mathbf n)}
            - \frac {P_{D-1} k^D_D} {W(\mathbf n)}
            \sum_{\substack{i=1 \\ n_i \ne 0}}^p  \frac {w_i} {n_i} \Bigg( \notag \\
                & \phantom{{}={}}
                \frac 1 {W(\mathbf n_{n^-_i=0})} \Bigg\{
                    n^-_i \,\Bigg|\,
                    S -
                    w_i \frac {n^+_i + n^-_i} {n_i} +
                    w_i \left\{
                        n^+_i + n^0_i \,\middle|\,
                        \frac {n^+_i} {n^+_i + n^0_i}
                    \right\}
                \Bigg\} + {} \notag \\
                & \phantom{{}={}}
                \frac 1 {W(\mathbf n_{n^+_i=0})} \Bigg\{
                    n^+_i \,\Bigg|\,
                    S -
                    w_i \frac {n^+_i + n^-_i} {n_i} +
                    w_i \left\{
                        n^0_i + n^-_i \,\middle|\,
                        \frac {n^-_i} {n^0_i + n^-_i}
                    \right\}
                \Bigg\} + {} \notag \\
                & \phantom{{}={}}
                \frac 1 {W(\mathbf n)} \sum_{k=1}^{n^-_i - 1}
                \Bigg[
                    S - w_i \frac {n^+_i + n^-_i} {n_i} +
                    w_i \frac {n^+_i + k} {n^+_i + n^0_i + k}
                \Bigg] + {} \notag \\
                & \phantom{{}={}}
                \frac 1 {W(\mathbf n)} \sum_{k=1}^{n^+_i - 1}
                \Bigg[
                    S - w_i \frac{n^+_i + n^-_i} {n_i} +
                    w_i \frac {k + n^-_i} {k + n^0_i + n^-_i}
                \Bigg]
            \Bigg)
        \Bigg] = \notag \\
        \intertext{To have $k$ only in the denominators, I add and subtract
        $n^0_i$ to the numerators and split the fractions:}
        &= 1 - P_{D-2} \Bigg[
            1 - \frac {S (1 - P_{D-1})} {W(\mathbf n)}
            - \frac {P_{D-1} k^D_D} {W(\mathbf n)}
            \sum_{\substack{i=1 \\ n_i \ne 0}}^p  \frac {w_i} {n_i} \Bigg( \ldots + {} \notag \\
                & \phantom{{}={}}
                \frac 1 {W(\mathbf n)} \sum_{k=1}^{n^-_i - 1}
                \Bigg[
                    S - w_i \frac {n^+_i + n^-_i} {n_i} +
                    w_i - w_i \frac {n^0_i} {n^+_i + n^0_i + k}
                \Bigg] + {} \notag \\
                & \phantom{{}={}}
                \frac 1 {W(\mathbf n)} \sum_{k=1}^{n^+_i - 1}
                \Bigg[
                    S - w_i \frac{n^+_i + n^-_i} {n_i} +
                    w_i - w_i \frac {n^0_i} {k + n^0_i + n^-_i}
                \Bigg]
            \Bigg)
        \Bigg] = \notag \\
        \intertext{So I can reduce all the terms which do not depend on $k$
        and be left with a harmonic sum:}
        &= 1 - P_{D-2} \Bigg[
            1 - \frac {S (1 - P_{D-1})} {W(\mathbf n)}
            - \frac {P_{D-1} k^D_D} {W(\mathbf n)}
            \sum_{\substack{i=1 \\ n_i \ne 0}}^p  \frac {w_i} {n_i} \Bigg( \ldots + {} \notag \\
                & \phantom{{}={}}
                \frac {\{ n^-_i \mid n^-_i - 1 \}} {W(\mathbf n)}
                \left( S + w_i \frac{n^0_i} {n_i} \right) -
                \frac {w_i n^0_i} {W(\mathbf n)}
                \sum_{k'=1+n^0_i+n^+_i}^{n_i - 1} \frac 1 {k'} + {} \notag \\
                & \phantom{{}={}}
                \frac {\{ n^+_i \mid n^+_i - 1 \}} {W(\mathbf n)}
                \left( S + w_i \frac{n^0_i} {n_i} \right) -
                \frac {w_i n^0_i} {W(\mathbf n)}
                \sum_{k'=1+n^0_i+n^-_i}^{n_i - 1} \frac 1 {k'}
            \Bigg)
        \Bigg] = \notag \\
        \intertext{Now I use the identity from \cite[eq.~5.4.14]{dlmf}
        \begin{equation}
            \psi(n) = \sum_{k=1}^{n-1} \frac 1 k - \gamma_E
        \end{equation}
        to obtain
        \begin{align}
            \sum_{k=1+m}^{n-1} \frac 1 k
            &= \left\{ n - m \,\middle|\,
                \sum_{k=1}^{n-1} \frac 1 k - \sum_{k=1}^m \frac 1 k \right\} = \notag \\
            &= \{ n - m \mid \psi(n) - \psi(m + 1) \},
        \end{align}
        so}
        &= 1 - P_{D-2} \Bigg[
            1 - \frac {S (1 - P_{D-1})} {W(\mathbf n)}
            - \frac {P_{D-1} k^D_D} {W(\mathbf n)}
            \sum_{\substack{i=1 \\ n_i \ne 0}}^p  \frac {w_i} {n_i} \Bigg( \ldots + {} \notag \\
                & \phantom{{}={}}
                \frac {\{ n^-_i \mid n^-_i - 1 \}} {W(\mathbf n)}
                \left( S + w_i \frac{n^0_i} {n_i} \right) -
                \frac {w_i n^0_i} {W(\mathbf n)}
                \{ \underbrace{n_i - n^0_i - n^+_i}_{=n^-_i} \mid
                    \psi(n_i) - \psi(1 + n^0_i + n^+_i) \} + {} \notag \\
                & \phantom{{}={}}
                \frac {\{ n^+_i \mid n^+_i - 1 \}} {W(\mathbf n)}
                \left( S + w_i \frac{n^0_i} {n_i} \right) -
                \frac {w_i n^0_i} {W(\mathbf n)}
                \{ \underbrace{n_i - n^0_i - n^-_i}_{=n^+_i} \mid
                    \psi(n_i) - \psi(1 + n^0_i + n^-_i) \}
            \Bigg)
        \Bigg] = \notag \\
        \intertext{Now all the outer conditions are on $n^\pm_i$. By grouping
        them together and reordering, I obtain:}
        &= 1 - P_{D-2} \Bigg[1 - \frac 1 {W(\mathbf n)} \Bigg[
            (1 - P_{D-1}) S +
            P_{D-1} k^D_D
            \sum_{\substack{i=1 \\ n_i \ne0}}^p \frac {w_i} {n_i} \Bigg( \notag \\
                & \phantom{{}+{}}
                \Bigg\{n^-_i \,\Bigg|\,
                \frac 1 {W(\mathbf n_{n^-_i=0})} \left( S
                - w_i \frac{n^-_i + n^+_i} {n_i}
                + w_i \left\{n^0_i + n^+_i \,\middle|\, \frac {n^+_i} {n^0_i + n^+_i}\right\}\right) + {} \notag \\
                & \phantom{\Bigg\{n^-_i \,\Bigg|\,}
                {} + \frac {n^-_i - 1} {W(\mathbf n)}
                \left(S + w_i \frac {n^0_i} {n_i}\right)
                - \frac {w_i n^0_i} {W(\mathbf n)}
                (\psi(n_i) - \psi(1 + n^0_i + n^+_i)) \Bigg\} + {} \notag \\
                & {} +
                \Bigg\{n^+_i \,\Bigg|\,
                \frac 1 {W(\mathbf n_{n^+_i=0})} \left( S
                - w_i \frac{n^+_i + n^-_i} {n_i}
                + w_i \left\{n^0_i + n^-_i \,\middle|\, \frac {n^-_i} {n^0_i + n^-_i}\right\}\right) + {} \notag \\
                & \phantom{\Bigg\{n^+_i \,\Bigg|\,}
                {} + \frac {n^+_i - 1} {W(\mathbf n)}
                \left(S + w_i \frac {n^0_i} {n_i}\right)
                - \frac {w_i n^0_i} {W(\mathbf n)}
                (\psi(n_i) - \psi(1 + n^0_i + n^-_i)) \Bigg\}
        \Bigg)\Bigg]\Bigg] = \notag \\
        \intertext{Observing that, if $n^\pm_i=0$ and $n_i > 0$, then
        $n^0_i+n^\mp_i = n_i$, $W(\mathbf n_{n^\pm_i=0}) = W(\mathbf n)$, and
        $\psi(n_i) - \psi(1 + n^0_i + n^\mp_i) = -1/n_i$, calculation shows that
        the two large conditional terms would yield~0 anyway if the condition was false,
        so I can remove the outer braces:}
        &= 1 - P_{D-2} \Bigg[1 - \frac 1 {W(\mathbf n)} \Bigg[
            (1 - P_{D-1}) S +
            P_{D-1} k^D_D
            \sum_{\substack{i=1 \\ n_i \ne0}}^p \frac {w_i} {n_i} \Bigg( \notag \\
                & \phantom{{}+{}}
                \frac 1 {W(\mathbf n_{n^-_i=0})} \left( S
                - w_i \frac{n^-_i + n^+_i} {n_i}
                + w_i \left\{n^0_i + n^+_i \,\middle|\, \frac {n^+_i} {n^0_i + n^+_i}\right\}\right) + {} \notag \\
                & {} + \frac {n^-_i - 1} {W(\mathbf n)}
                \left(S + w_i \frac {n^0_i} {n_i}\right)
                - \frac {w_i n^0_i} {W(\mathbf n)}
                (\psi(n_i) - \psi(1 + n^0_i + n^+_i)) + {} \notag \\
                & {} +
                \frac 1 {W(\mathbf n_{n^+_i=0})} \left( S
                - w_i \frac{n^+_i + n^-_i} {n_i}
                + w_i \left\{n^0_i + n^-_i \,\middle|\, \frac {n^-_i} {n^0_i + n^-_i}\right\}\right) + {} \notag \\
                & {} + \frac {n^+_i - 1} {W(\mathbf n)}
                \left(S + w_i \frac {n^0_i} {n_i}\right)
                - \frac {w_i n^0_i} {W(\mathbf n)}
                (\psi(n_i) - \psi(1 + n^0_i + n^-_i))
        \Bigg)\Bigg]\Bigg] = \notag \\
        \intertext{Which allows to gather some terms, arriving at
        \autoref{eq:depth2}:}
        &= 1 - P_{D-2} \Bigg[1 - \frac 1 {W(\mathbf n)} \Bigg[
            (1 - P_{D-1}) S +
            P_{D-1} k^D_D
            \sum_{\substack{i=1 \\ n_i \ne0}}^p \frac {w_i} {n_i} \Bigg( \notag \\
                & \phantom{{}+{}}
                \left(S + w_i \frac {n^0_i} {n_i}\right)
                \left(
                    \frac 1 {W(\mathbf n_{n^-_i=0})} +
                    \frac 1 {W(\mathbf n_{n^+_i=0})} +
                    \frac {n^-_i + n^+_i - 2} {W(\mathbf n)}
                \right)
                + {} \notag \\
                & {} +
                \frac {w_i} {W(\mathbf n_{n^-_i=0})}
                \left(\left\{
                    n^0_i + n^+_i
                    \,\middle|\,
                    \frac {n^+_i} {n^0_i + n^+_i}
                \right\} - 1\right) +
                \frac {w_i} {W(\mathbf n_{n^+_i=0})}
                \left(\left\{
                    n^0_i + n^-_i
                    \,\middle|\,
                    \frac {n^-_i} {n^0_i + n^-_i}
                \right\} - 1\right)
                + {} \notag \\
                & {} -
                \frac {w_i n^0_i} {W(\mathbf n)}
                \big(
                    2\psi(n_i) -
                    \psi(1 + n^0_i + n^-_i) -
                    \psi(1 + n^0_i + n^+_i)
                \big)
        \Bigg)\Bigg]\Bigg]. \notag
    \end{align}

    \paragraph{Next levels} It is probably possible to obtain an $O(p)$ formula for $k^D_{D-k}$ for any $k$ by writing generalized harmonic sums in terms of higher order polygamma functions. However the calculation would be even more tedious and error-prone than for $k^D_{D-2}$ due to the need to keep track of the special cases $n_i = 0$ that multiply going down the recursion, so I did not attempt it. It would likely also not be computationally convenient because the bottleneck in my code is currently evaluating the formula for $k^D_{D-2}$ due to its sheer length.

    \section{Accuracy of the calculation of the prior correlation function}
    \label{sec:interpolation}

    With numerical experiments, I justify the choice made in \autoref{sec:finalkernel} of using $k^{2,5}_{0,1}$ as approximate estimate of $k_0$, and measure its estimation error. The idea is computing accurate reference values by reaching higher depths, and then compare them with the the less accurate but fast repeated depth~2 calculation.

    \paragraph{Setup of the experiment}

    For each combination of the following values of hyperparameters $\alpha$ and $\beta$, number of predictors $p$, base truncation depth $D_0$, and recursions~$r$:
    \begin{center}
                
        \begin{tabular}{l|*{13}{l@{\hspace{1.3ex}}}}
            $\alpha$ & 0.01 & 0.15 & 0.30 & 0.45 & 0.60 & 0.70 & 0.80 & 0.90 & 0.95 & 0.99 & 1.00 & $\phantom{0.00}$ & $\phantom{0.00}$ \\
            $\beta$  & 0.0 & 0.1 & 0.2 & 0.4 & 0.6 & 0.8 & 1.0 & 1.5 & 2.0 & 2.5 & 3.0 & 3.5 & 4.0 \\
            $p$      & 1 & 2 & 3 & 4 & 5 & 6 & 7 & 8 & 9 & 10 \\
            $D_0$    & 1 & 2 & 3 & 4 & 5 \\
            $r$      & 2 & 5
        \end{tabular}
        
    \end{center}
    I generate $N=250$ quasi-Monte Carlo pairs of locations on a grid with $n_i=10$ splits \marginpar{Is the result sensitive to $n_i$? I'd need to prove some sort of interpolation property.} per axis and evaluate the lower and upper bounds $k^{D_0}_{0,0}$ and $k^{D_0,r}_{0,1}$ on $k_0$. I modify the functions with the substitution $P_0\mapsto 1$ to have them range in $[0,1]$ instead of $[1-\alpha, 1]$.\marginpar{I haven't stated the bounding property for the pseudo-recursive version.} The locations are re-sampled for each combination of $\alpha$, $\beta$, $p$.

    \paragraph{Determination of an accurate reference value}

    I somewhat arbitrarily define a ``most accurate'' estimate $K$ of $k_0$, as the point which has the same interpolation coefficient with respect to the $D_0=3$ and $D_0=5$ intervals at $r=5$:
    \begin{align}
        \frac{K - k^{5}_{0,0}}{\Delta^{5,5}} &=
        \frac{K - k^{3}_{0,0}} {\Delta^{3,5}} &&\implies&
        K &= \frac{\Delta^{3,5} k^{5}_{0,0} - \Delta^{5,5} k^{3}_{0,0}}
        {\Delta^{3,5} - \Delta^{5,5}},
    \end{align}
    where $\Delta^{D_0,r} = k^{D_0,r}_{0,1} - k^{D_0}_{0,0}$. This choice is made on the intuition that, since the function is recursive, the position of the more accurate $D_0=5$ interval within the larger $D_0=3$ one gives an indication on the position of deeper intervals w.r.t.\ the one at $D_0=5$. In particular, in the white noise limit (property~\ref{it:white}, \autoref{th:corrprop}), $k^{5}_{0,0} = k^{3}_{0,0} = 1 - P_0 = K$, so $K$ is exact.

    \paragraph{Interpolation coefficient}

    For base depths $D_0 < 5$, I compute the interpolation coefficient $\gamma_K$ that yields the reference value $K$:
    \begin{equation}
        K = (1 - \gamma_K) k^{D_0}_{0,0} + \gamma_K k^{D_0,r}_{0,1}.
    \end{equation}
    I will leave implicit that $\gamma_K$ depends on $\alpha$, $\beta$, $p$, $D_0$, $r$, and the $N$ pairs of locations.

    \paragraph{Averaging of the interpolation coefficient}

    I'd like to have a single ``good'' value $\bar\gamma_K$ for each combination of hyperparameters. I define $\bar\gamma_K(\alpha,\beta,p,D_0,r)$ as the median of $\gamma_K$ over location pairs, weighted with the bounding interval width $\Delta^{D_0,r}$.

    The reason for weighting with interval width is that the value of $\gamma_K$ is more important when the width is larger: if the width is small, the lower and upper bounds are close, so are all the interpolated values in between. The maximum error is indeed proportional to the width.

    The reason for using a median instead of an average is to ignore points where the value of $\gamma_K$ is very different from the others.

    There is a potential problem in that averaging over locations could create artefacts that depend on the distribution of locations. As $p$ increases, it is less likely for a pair of points to be close, so the typical correlation is smaller. On the other hand, this may simply reflect the actual usage of the correlation function. I haven't investigated this issue.

    \paragraph{Bound on the estimation error}

    At fixed hyperparameters, I compute the maximum over location pairs of the width of the bounding interval, and an upper bound on the maximum error in estimating $k_0$ with $k^{D_0,r}_{0,1}$:
    \begin{equation}
        \text{error} = |k^{D_0,r}_{0,1} - k_0| \le \max_\text{pairs of locations}\max \left\{
            |k^{5,5}_{0,1} - k^{D_0,r}_{0,1}|,
            |k^{D_0,r}_{0,1} - k^{5}_{0,0}|
        \right\}.
    \end{equation}
    I used the interval at depth $D_0=5$, $[k^5_{0,0}, k^{5,5}_{0,1}]$, because it is the narrowest one, thus producing the smallest error bound possible. \autoref{fig:gamma} shows $\bar\gamma_K$, the maximum width, and maximum error, only at base depth $D_0=2$, since this is the depth that matters for calculations.
    
    \begin{figure}
        
        \centering
        \widecenter{\includempl{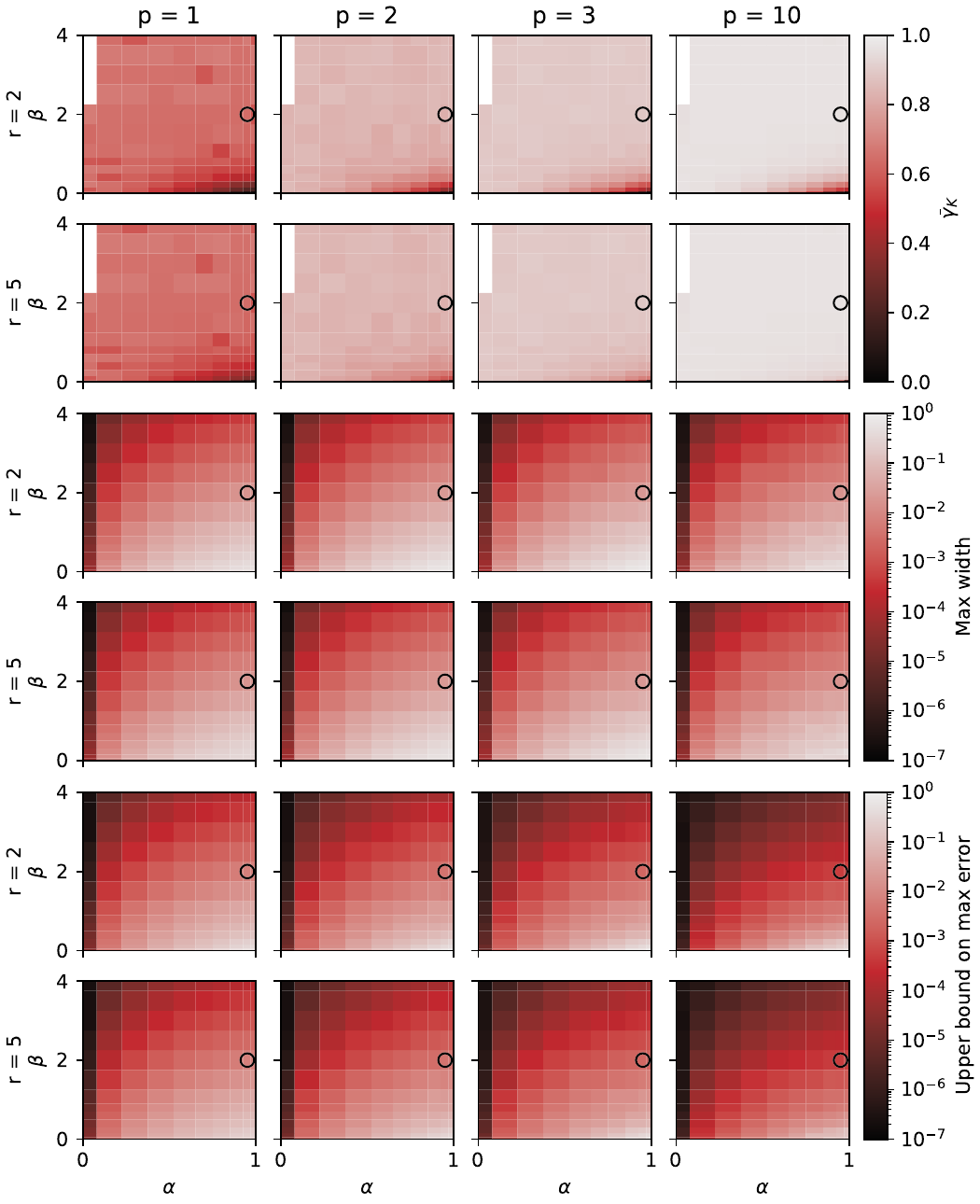}}
        
        \caption{\label{fig:gamma} Top two rows: Interpolation coefficient $\gamma_K$ such that $K = (1-\gamma_K) k^{2}_{0,0} + \gamma_K k^{2,r}_{0,1}$, averaged over locations, where $K \approx k_0$. The missing values are due to very small interval widths, so they can be safely ignored. Middle rows: the maximum observed width of the bounding interval $k^2_{0,0} \le k_0 \le k^{2,r}_{0,1}$. Bottom rows: maximum observed error in estimating $k_0$ with $k^{2,r}_{0,1}$. The circles are centered on the default BART hyperparameters $\alpha=0.95$, $\beta=2$. See \autoref{sec:interpolation}.}
        
    \end{figure}
    
    As expected, $\bar\gamma_K \to 0^+$ in the white noise corner $\alpha \to 1^-$ and $\beta \to 0^+$ (property~\ref{it:white}, \autoref{th:corrprop}). Increasing $r$ makes a difference almost only in this corner. For other values of the hyperparameters, the relationship $\bar\gamma_K(\alpha,\beta,p)$ looks flat (at fixed $p$), and increases quite quickly from $p=1$ to $p=2$ but then stabilizes, suggesting an asymptote w.r.t.\ $p$ at $\bar\gamma_K\approx0.96$. This means that the reference value $K\approx k_0$ is pretty close to the upper bound $k^{2,r}_{0,1}$. The upper bound on the maximum error at the default BART values $\alpha=0.95$, $\beta=2$ is, both at $r=2$ and $r=5$:
    \begin{center}
        \begin{tabular}{l|*4l}
            $p$ & 1 & 2 & 3 & 10 \\
            error $\le$ & 0.0065 & 0.0040 & 0.0022 & 0.0005
        \end{tabular}
    \end{center}

    \section{Numerical cross checks}
    \label{sec:crosschecks}

    This section extends \autoref{sec:finalkernel}, providing assurance that the calculation of the kernel is correct.
    
    \subsection{Comparison between kernel and generative tree process}
    \label{sec:checkprior}
    
    To check at once the correctness of the derivations of the BART correlation function $k_d$ (\autoref{th:corr}), its truncated version $k^D_{d,\gamma}$ (\autoref{def:truncorr}), the special-casing $k^D_{D-2,\gamma}$ (\autoref{th:depth2}), and the computer implementation based on them, I sample directly from the BART prior, generating the trees, and compare the sample covariance matrix with the one computed with my kernel.
    
    I fix some arbitrary values of the hyperparameters, the number of predictors, the splitting grid, and 100 predictor vectors. I generate \num{100000} samples to have a precise estimate of the covariance, and use $m = \num{10000}$ trees to make the distribution of the samples Normal within good approximation. I use the Normality of the samples to assume a scaled Wishart distribution for the sample covariance matrix and perform a likelihood ratio test for the equality of the covariance to the one computed by the kernel. I use $(k^D_{0,0}+k^D_{0,1})/2$ as kernel estimate, increasing $D$ until the maximum error $k^D_{0,1}- k^D_{0,0}$ is much smaller than the Wishart standard deviation, reaching $D=4$. Note that, due to sufficiency, a Normal LR test on the samples would be identical to the Wishart LR test on the sample covariance matrix. The result is shown in \autoref{fig:checkprior}. The test does not reject the equality hypothesis.\marginpar{What version of the kernel did I use exactly here?}
        
    \begin{figure}

        \centering
        \includegraphics[width=80ex]{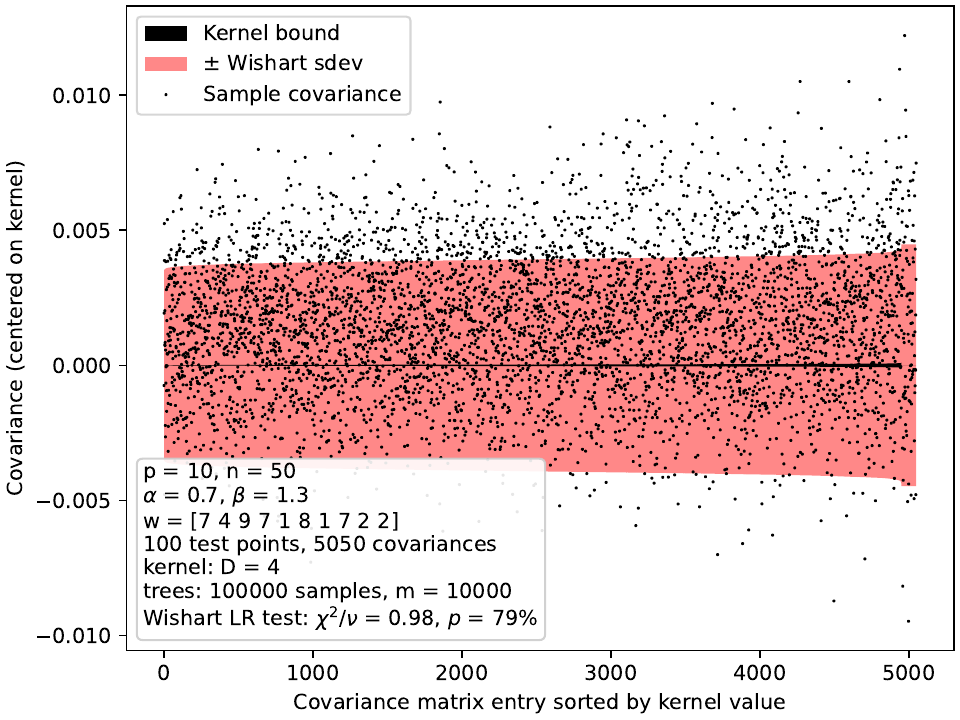}
        
        \caption{\label{fig:checkprior} Comparison of the sample covariance
        matrix of samples drawn from the BART prior to my numerical
        approximation of the covariance function, to cross-check my
        calculations. See \autoref{sec:checkprior}.}
        
    \end{figure}
    
    \subsection{Self-consistency of the correlation function}
    \label{sec:kernelcheck}
    
    I check numerically the following qualitative properties and corner cases of the correlation function on a large set of values of hyperparameters, problem sizes and locations:
    \begin{itemize}
        
        \item $0 \le k^D_{0,0} \le k^{D,2}_{0,1} \le k^D_{0,1} \le 1$ (Theorems~\ref{th:truncinterp}, \ref{th:prtruncorrup}, \ref{th:prtruncorrupbetter}).
        
        \item $[k^D_{0,0}, k^D_{0,1}] \subseteq [k^{D'}_{0,0}, k^{D'}_{0,1}]$ if $D \ge D'$ (\autoref{th:filter}).

        \item $k^D_{0,\gamma}[\alpha] \le k^D_{0,\gamma}[\alpha']$ if $\alpha \ge \alpha'$.
        
        \item $k^D_{0,\gamma}[\beta] \le k^D_{0,\gamma}[\beta']$ if $\beta \le \beta'$.
        
        \item $k^D_{0,\gamma}$ decreases as $\mathbf n^0$ increases at fixed
        $\mathbf n$ (properties~\ref{it:monoext} and~\ref{it:monoint}, \autoref{th:corrprop}).
        
        \item $k^D_{0,\gamma} = 1$ under either one of these sufficient conditions: $\mathbf n^0 = \mathbf 0$ (property~\ref{it:corr1} \autoref{th:corrprop}), $\mathbf w \odot \mathbf n=\mathbf 0$, $p = 0$.
        
        \item $k^D_{0,\gamma}$ is invariant under: per-axis swapping of $n^-_i$
        with $n^+_i$, simultaneous reordering of $\nvecs$ and $\mathbf w$,
        appending zeros to $\nvecs$ with any $\mathbf w$, appending zeros to
        $\mathbf w$ with any $\nvecs$.
        
        \item $k^D_{0,0} = k^D_{0,1}$ under either one of these sufficient
        conditions: $\beta = \infty$ if $D > 0$, $\alpha = 0$, $\mathbf n^0 =
        \mathbf 0$.
        
        \item Using the closed-form formulas for $k^D_{D-1,\gamma}$ and $k^D_{D-2,\gamma}$ (Equations~\ref{eq:depth1} and~\ref{eq:depth2}) yields the same result as walking through the recursion (\autoref{eq:truncorr}).

        \item $k^2_{0,1} = 1 - \alpha + \alpha k^2_{0,1}[P_0\mapsto 1]$.
        
    \end{itemize}
    
    \section{Out-of-sample predictive accuracy on real data}
    \label{sec:benchmark}

    I compare GP regression with my BART kernel against the original BART MCMC by comparing their predictions on held-out test sets on a series of datasets. This section extends \autoref{sec:benchmarkmain}.

    \subsection{Data}
    \label{sec:benchmarkdata}

    I use the 42 datasets of the original BART paper \autocite{chipman2010,kim2007}. The dataset size ranges from $n_0=96$ to $n_0=6806$, while the number of predictors ranges from $p_0=3$ to $p_0=28$. The outcome is log- or sqrt-transformed if its distribution looks skewed, and standardized; and categorical predictors are converted to binary dummies, one for each stratum.
    % I checked that it's 1 per stratum both in my impl and in Chipman et al (2010), rather than n(strata) - 1.
    Each dataset is randomly split 5:1 in train:test parts, 20 times. The train parts are again randomly split in 5 folds for cross-validation. After partitioning and dummyfication, the training sample size ranges from $n=80$ to $n=5671$, and the number of predictors from $p=3$ to $p=67$. Hugh Chipman kindly provided their original data files, so I have exactly the same transformations, train/test splits, and folds of \textcite{chipman2010}.
    % the transformations are already applied in the data files, as explained in the readme of the datasets directory.
    
    \subsection{Methods}

    I test six methods, with this naming scheme:
    \begin{description}
        \item[\emph{MCMC}] standard BART (\autoref{sec:modelmcmc}).
        \item[\emph{MCMC-CV}] BART cross-validated as in \textcite{chipman2010} (\autoref{sec:modelmcmccv}).
        \item[\emph{MCMC-XCV}] BART cross-validated with the pre-packaged routine \texttt{xbart} provided by \texttt{dbarts} (\autoref{sec:modelmcmcxcv})
        \item[\emph{GP}] the infinite trees limit of BART as GP regression (\autoref{sec:modelgp}).
        \item[\emph{GP-hyp}] like \emph{GP} but with hyperparameter tuning (\autoref{sec:modelgphyp}).
        \item[\emph{GP+MCMC}] standard BART with hyperparameters set to those found by the GP regression (\autoref{sec:modelgpmcmc}).
    \end{description}

    \subsubsection{Method: MCMC}
    \label{sec:modelmcmc}

    I mostly reproduce the setup of \textcite{chipman2010}. See Equations~\ref{eq:bartdeffirst} to~\ref{eq:bartdeflast} for the definition of the model. The hyperparameters are set to
    \begin{itemize}
        
        \item $m = 200$.

        \item $\alpha=0.95$, $\beta=2$.

        \item $\mathbf w = \mathbf 1$.

        \item $\displaystyle \mu_\mu = \frac{\max \mathbf y_\text{train} + \min \mathbf y_\text{train}}{2m}$.

        \item $\displaystyle \sigma_\mu = \frac{\max \mathbf y_\text{train} - \min \mathbf y_\text{train}}{2k\sqrt m}$, $k=2$.
        % I checked in the BART 2.9.6 code this is what it does, loc BART/R/gbart.R::gbart:117.

        \item $\nu = 3$.

        \item $\lambda$ is set such that $P(\sigma<\hat\sigma_\text{OLS}) = q=0.9$, where $\hat\sigma_\text{OLS}^2$ is the error variance estimated by OLS (with an intercept), i.e., the sum of squared residuals over the degrees of freedom.

        \item The splitting grid is set by taking the unique set of values in the training set for each predictor, and putting a cutpoint midway between each one.
        % R BART does this if numcut=<large number> and usequants=T, see BART 2.9.6 loc BART/R/bartModelMatrix.R::bartModelMatrix:92.

    \end{itemize}
    Instead of the original implementation \texttt{BayesTree}, I use the \texttt{BART} package, because it seems to be more adherent to the exact specification of the BART model (see \autoref{sec:crancomparison}). I run 4 MCMC chains in parallel, discarding the first 1000 samples from each, and keeping a total of 1000 samples over all chains, so 250 per chain. I then merge all the chains. The burn-in of 1000 is higher than the default 100; I do this because my personal experience has taught me that 100 is often too low to get reasonable convergence.

    \subsubsection{Method: MCMC-CV}
    \label{sec:modelmcmccv}

    I optimize some of the hyperparameters of the \emph{MCMC} method with cross validation. The folds are described in \autoref{sec:benchmarkdata}. I reproduce the CV procedure of \textcite{chipman2010}, exploring the cartesian product of these sets of parameter values:
    \begin{itemize}
        \item $(\nu,q) \in \{(3, 0.90), (3, 0.99), (10, 0.75)\}$,
        \item $m \in \{50, 200\}$,
        \item $k \in \{1, 2, 3, 5\}$.
    \end{itemize}
    The combination which yields the lowest RMSE over all folds is chosen. The MCMC configuration is the same as \emph{MCMC} for all CV runs. I use \texttt{dbarts} instead of \texttt{BART}, because it is faster, and because it seems to match \texttt{BayesTree} in its details (see \autoref{sec:crancomparison}), such that \emph{MCMC-CV} overall should be an accurate reproduction of the ``BART-cv'' method showcased in \textcite{chipman2010}.

    \subsubsection{Method: MCMC-XCV}
    \label{sec:modelmcmcxcv}

    I use the built-in CV routine \texttt{xbart} of the package \texttt{dbarts}, trying combinations of these hyperparameters:
    \begin{itemize}
        \item $m \in \{50, 200\}$,
        \item $\alpha \in \{0.60, 0.80, 0.95\}$,
        \item $\beta \in \{1, 2, 4\}$,
    \end{itemize}
    with one repetition per combination (\texttt{reps = 1}).\marginpar{Maybe I should use the default \texttt{n.reps = 40}. It would be as slow as \emph{MCMC-CV}, but there could be a reason it is set so high. Also it would adhere better to the goal of using a method as-is.} Instead of imposing the setup of \autoref{sec:modelmcmc}, I leave the configurations to their default values; the only difference this should entail is that the splitting grid is evenly spaced in the range of the data, but I'm not sure. For the full run after the CV, the only non-default configuration I adopt is to set the number of MCMC samples as in \autoref{sec:modelmcmc}. The general idea of this setup is to use a good existing tool as-is, to emulate the applied usage of BART.

    \subsubsection{Method: GP}
    \label{sec:modelgp}

    I implement a GP regression following closely the setup of \emph{MCMC}. The only differences are that the latent function is a GP instead of a sum of trees, and that \texttt{BART} forbids leaves containing less that 5 datapoints; the latter should not make a visible difference as this is unlikely to happen due to the trees being shallow.\marginpar{Unless the predictors are correlated! Could I make a fork of BART and put that limit in a config? It should be easy to do, I can fork bnptools and use install\_github on my fork.} The kernel is $k^{2,5}_{0,1}$ as explained in \autoref{sec:finalkernel} and \autoref{sec:interpolation}, and the variance and mean of the GP are matched to those of the sum of trees, such that the GP represents the infinite trees limit.

    The only free parameter which is not part of the GP is the error variance $\sigma$. The posterior density on $\sigma$ is
    \begin{equation}
        p(\sigma|\mathbf y) \propto \mathcal N(\mathbf y; m\mu_\mu\mathbf 1, m\sigma_\mu^2\Sigma + \sigma^2I) p(\sigma),
    \end{equation}
    where $\Sigma$ is the prior correlation matrix computed with the kernel, and I have left implicit the other hyperparameters. Since $\sigma$ is positive, I do the inference working with $\log\sigma^2$.

    Evaluating the Normal density requires computing the a decomposition of the covariance matrix, which is of the form $C+\sigma^2I$ ($C=m\sigma_\mu^2\Sigma$); this is the bottleneck of the calculation. The default option for this task is the Cholesky decomposition. However, since to draw samples the density is evaluated at multiple $\sigma$ values, it is convenient to use a decomposition which is slower but can be computed only once and recycled: I diagonalize $C$, so $C = ODO^\top$, with $O$ orthogonal and $D$ diagonal. Since $OIO^\top = I$, the same change of basis diagonalizes the whole matrix: $C + \sigma^2I = O(D + \sigma^2I)O^\top$. The rest of the calculation does not require any linear algebra operation as expensive as producing $O$.

    To sample the $\log\sigma^2$ posterior I use the ratio of uniforms method, a rejection sampling method apt for unnormalized densities with unbounded support \autocite{kinderman1977}. To tune the acceptance of the method, I find the MAP and other extrema with some numerical minimizations. I draw 1000 independent samples. To compute the posterior mean I combine the formula for GP regression posterior mean (\autoref{eq:gpreg}) with iterated expectation:
    \begin{align}
        E[\mathbf Y_\text{test}\mid\mathbf y_\text{train}] &=
            E[E[\mathbf Y_\text{test}|\mathbf y_\text{train}, \sigma^2]\mid\mathbf y_\text{train}] \approx \notag \\
        &\approx \frac 1 {n_\text{samples}} \sum_{j=1}^{n_\text{samples}} E[\mathbf Y_\text{test}|\mathbf y_\text{train}, \sigma^2_j].
    \end{align}

    \subsubsection{Method: GP-hyp}
    \label{sec:modelgphyp}

    Starting from the \emph{GP} method, I optimize the hyperparameters $\alpha$, $\beta$, $k$, $\sigma$ by fixing them to their marginal a posteriori mode. I use L-BFGS to maximize the Normal density of the GP times a prior on the free hyperparameters. To improve the quality of the approximation, I follow the heuristic of transforming the hyperparameters such that their prior is standard Normal, and find the mode in the trasformed space. The prior is:
    \begin{align}
        \alpha &\sim \operatorname{Beta}(2, 1), \\
        \beta &\sim \operatorname{InvGamma}(1, 1), \\
        \log k &\sim \mathcal N(\log 2, 2^2), \\
        \sigma &\sim \text{same as BART model.}
    \end{align}
    The inverse gamma prior on $\beta$ keeps it away from~0, to forbid the case $\alpha\to 1$, $\beta\to 0$ where the latent function becomes equivalent to the error term.\marginpar{Tuning $k$ is probably doing a lot here, the R packages leave notes on the ground, as $k$ could be inferred in the MCMC. But then why does \texttt{dbarts} set $k$ to a constant by default, despite implementing its inference?}

    I plug the optimized values of $\alpha$, $\beta$ and $k$ into the same procedure of the \emph{GP} method, thus re-deriving the posterior on $\sigma$ conditional on the other parameters.

    \subsubsection{Method: GP+MCMC}
    \label{sec:modelgpmcmc}

    Starting from \emph{GP-hyp}, instead of running \emph{GP} with the optimized values of the hyperparameters $\alpha$, $\beta$, $k$, I use \emph{MCMC} implemented with \texttt{dbarts}.

    \subsection{Performance metrics}

    I evaluate the models on two standard out-of-sample prediction error metrics: RMSE and log-loss.

    \paragraph{RMSE}

    The RMSE is the root mean square error of an estimate w.r.t.\ the true outcome values from a test set not used as data. As estimate I take the predictive posterior mean, since the mean minimizes the expected squared error.
    \begin{equation}
        \mathrm{RMSE} = \sqrt{
            \frac 1{n_\text{test}} \sum_{i=1}^{n_\text{test}}
            \big(
                y_{i,\text{test}} -
                E[Y_{i,\text{test}}\mid \mathbf Y_\text{train} = \mathbf y_\text{train}]
            \big)^2
        }.
    \end{equation}
    Since the data is standardized, the RMSE is comparable across different datasets, and should yield at most about~1.

    \paragraph{Log-loss}

    The log-loss is a distributional metric that takes into account the whole posterior. It is defined as  the negative log posterior density, evaluated at the actual true value:
    \begin{equation}
        \text{log-loss} =
            -\frac1{n_\text{test}}
            \log p(\mathbf Y_\text{test}=\mathbf y_\text{test}\mid \mathbf Y_\text{train} = \mathbf y_\text{train}).
    \end{equation}
    It is divided by the number of test samples, to make it comparable across cases with different sample sizes.

    It is justified as the only proper scoring procedure which is invariant under changes of variables, i.e., the expected value of the log-loss is minimized by computing it with the same distribution used to take the expectation, and reparametrizations can only shift the log-loss by a constant, since they just add a jacobian factor to the probability density.

    Since the data is standardized, typically the log-loss should fall in $(0, 1)$ or nearby. However, due to using a Normal error distribution in the model, which is light-tailed, the log-loss can explode to very high values if the error variance is estimated too small. For this reason, the log-loss is a quite ``strict'' metric to judge models, and it is often not useful alone: the light-tailedness is an accident of design for computability, while the statisticians using the model understand that in practice prediction errors are not best represented by a Normal distribution.

    \paragraph{Computation of the log-loss in MCMC}

    The MCMC algorithm used in BART does not produce directly a posterior density. However the error term density is not modified by conditioning on the data due to the errors being independent, which allows to estimate the marginal density by averaging the error density over the posterior samples:
    \begin{align}
        p(\mathbf y_\text{test}|\mathbf y_\text{train})
        &= \int
            \mathrm df(X_\text{test})\, \mathrm d\sigma\,
            p(\mathbf y_\text{test}, f(X_\text{test}),\sigma|\mathbf y_\text{train}) = \notag \\
        &= \int
            \mathrm df(X_\text{test})\, \mathrm d\sigma\,
            p(\mathbf y_\text{test}| f(X_\text{test}),\sigma, \mathbf y_\text{train}) p(f(X_\text{test}),\sigma|\mathbf y_\text{train}) = \notag \\
        &= \int
            \mathrm df(X_\text{test})\, \mathrm d\sigma\,
            p(\mathbf y_\text{test}| f(X_\text{test}),\sigma) p(f(X_\text{test}),\sigma|\mathbf y_\text{train}) \approx \notag \\
        &\approx \frac 1 {n_\text{samples}} \sum_{j=1}^{n_\text{samples}}
            p(\mathbf y_\text{test}| f_j(X_\text{test}),\sigma_j) = \notag \\
        &= \frac 1 {n_\text{samples}} \sum_{j=1}^{n_\text{samples}}
            \exp\left(
                -\frac{n_\text{test}}2 \log(2\pi\sigma_j^2)
                -\frac12 \frac{\|\mathbf y_\text{test} - f_j(X_\text{test})\|^2}{\sigma_j^2}
            \right). \label{eq:complogloss}
    \end{align}
    To compute the logarithm of \autoref{eq:complogloss}, I use a single \texttt{logsumexp} function that avoids under/overflow.

    \paragraph{Computation of the log-loss in GP regression}

    With GP regression, the posterior density is computable analytically because it is a multivariate Normal density (\autoref{eq:gpreg}). However I noticed that this yields a systematically lower log-loss than using \autoref{eq:complogloss}. So, to make the comparison fairer, I first draw samples from the posterior and then apply the same procedure I use for \emph{MCMC}.

    There are multiple ways to draw posterior samples in GP regression. I pick a method particularly suited to my specific context, in which the only free non-GP parameter of the model is the error variance $\sigma^2$. Using the notation of \autoref{sec:gprecap}, instead of starting from the posterior covariance matrix $\Sigma/\Sigma_{xx}$, I diagonalize both $\Sigma$ and $\Sigma_{xx}$, generate samples with covariance matrix $\Sigma$, and use Matheron's rule \autocite[72]{terenin2022} to transform them to samples with covariance matrix $\Sigma/\Sigma_{xx}$:
    \begin{align}
        \mathbf u \in \mathbb R^{n_\text{train} + n_\text{test}} &\sim
        \mathcal N(\mathbf 0, I), \\
        \Sigma = \begin{pmatrix}
            \Sigma_{xx} & \Sigma_{xx^*} \\
            \Sigma_{x^*x} & \Sigma_{x^*x^*}
        \end{pmatrix} &= V\diag\mathbf w\, V^\top, \quad\text{(diagonalization)} \\
        \begin{pmatrix}
            f(\mathbf x) \\ f(\mathbf x^*)
        \end{pmatrix}
        &= V\diag\sqrt{\max(\mathbf w, 0)}\, \mathbf u,\quad\text{(joint prior sample)} \\
        \Sigma_{xx} &= U\diag\mathbf s\, U^\top, \quad\text{(diagonalization)} \\
        \boldsymbol\varepsilon &\sim \mathcal N(\mathbf 0, \sigma^2 I), \\
        \tilde f(\mathbf x^*) &=
            f(\mathbf x^*) + \Sigma_{x^*x}(\Sigma_{xx} + \sigma^2I)^{-1}
            (\mathbf y_\text{train} - f(\mathbf x) - \boldsymbol\varepsilon) = \\
        &= f(\mathbf x^*) + \Sigma_{x^*x} U\diag\frac{1}{\max(\mathbf s, 0) + \sigma^2} U^\top
            (\mathbf y_\text{train} - f(\mathbf x) - \boldsymbol\varepsilon),
    \end{align}
    where, by Matheron's rule, $\tilde f(\mathbf x^*)$ is a posterior sample. Diagonalizing $\Sigma_{xx}$ is slower than computing $\Sigma/\Sigma_{xx}$ and its Cholesky decomposition, but it can be done once for all posterior samples of $\sigma$. I could Cholesky-decompose $\Sigma$ instead of diagonalize it, but I prefer diagonalization because it's more numerically accurate, since anyway I am already spending a similar amount of time to diagonalize $\Sigma_{xx}$. Diagonalizing $\Sigma$ also allows me to compute other quantities that I was interested in, but which I won't report on for brevity.

    \subsection{Results}

    To summarize the result of the comparisons across all datasets and random train/test splits, for each case, I take the best performing method on that case, and rescale/shift the metrics relative to its metric. Then, I plot the empirical cumulative function of these relative metrics. The reason for making them relative is that absolute performance changes a lot based on the intrinsic difficulty of the dataset.

    See \autoref{fig:comparison4} and \autoref{sec:benchmarkmain} for results and commentary. The running time of \emph{GP} and \emph{GP-hyp} does not include generating the samples used only to compute the log-loss. The maximum values of the relative metrics are:
    \begin{center}
        \begin{tabular}{lrr}
            \toprule
            & \multicolumn{2}{c}{Metric (relative)} \\
            \cmidrule{2-3}
            Method & RMSE & log-loss \\
            \midrule
            \emph{MCMC} & 4.9 & 115 \\
            \emph{MCMC-CV} & 1.9 & 20 \\
            \emph{MCMC-XCV} & 4.6 & 76 \\
            \emph{GP} & 3.7 & 45 \\
            \emph{GP-hyp} & 4.0 & 190 \\
            \emph{GP+MCMC} & 5.1 & 10 \\
            \bottomrule
        \end{tabular}
    \end{center}
    To look more closely at the difference between \emph{GP} and \emph{MCMC}, here I also show a comparison between those two methods alone in \autoref{fig:comparison2}.

    Note: in previous unpublished expositions of this work, I stated that \emph{GP-hyp} had a strictly better log-loss profile than \emph{MCMC-CV}, unlike the current results in which the opposite holds. That was due to a very poor choice on my part of how to compute the log-loss (I used a KDE of the posterior density), which, although a priori fair, was quite inaccurate.

    \begin{figure}
        
        \widecenter{\includempl{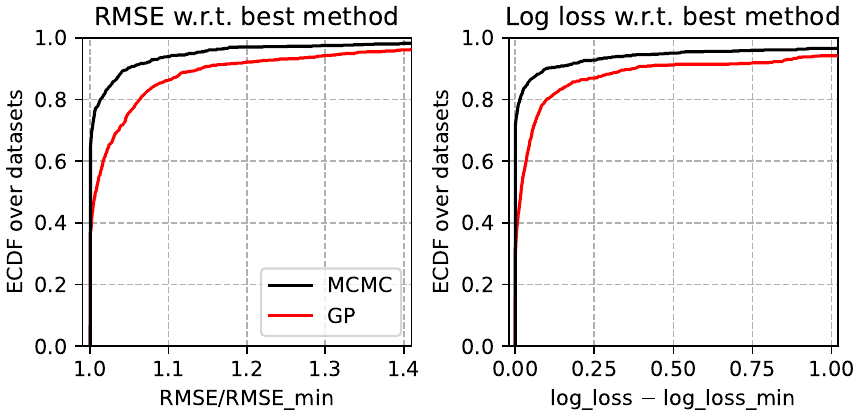}}
        
        \caption{\label{fig:comparison2} Comparison between BART (\emph{MCMC}) and its infinite trees limit implemented as a GP regression (\emph{GP}). The datasets, model definition, and fixed hyperparameters are exactly the same of \textcite{chipman2010}. Methods are compared on their predictions on a held-out test set, and the metric is scaled or shifted relative to the best of them. The plots show the distribution of relative metrics over datasets and 20 random train/test splits. The plots are truncated; the maxima are: relative RMSE: \emph{MCMC} 3.0, \emph{GP} 2.8, log-loss: \emph{MCMC} 74, \emph{GP} 15.}
        
    \end{figure}

    \subsection{Analysis}
    \label{sec:anal}

    Looking at other properties of the data or the inference beyond the performance metrics, I establish the following:
    \begin{itemize}
        
        \item \emph{GP-hyp} improves relative to \emph{MCMC-CV} when deeper trees would be more appropriate (\autoref{sec:treedepth}).
        
        \item \emph{GP-hyp} is mostly using $\alpha$ and $\beta$ together to set the depth, rather than distinguishing their effects (\autoref{sec:diffdepth}).
        
        \item There is no relationship between number of datapoints or number of features and the relative performance of \emph{GP-hyp} over \emph{MCMC-CV} (\autoref{sec:datasize}).
        
        \item \emph{MCMC-CV} does better than \emph{GP-hyp} on low-noise datasets (\autoref{sec:difficulty}).
        
        \item There is no relationship between the hyperparameters chosen by \emph{GP-hyp} and \emph{MCMC-XCV} (\autoref{sec:hyperpick}).
    
    \end{itemize}

    \subsubsection{Effect of tree depth}
    \label{sec:treedepth}

    The raison d'être of BART is using shallow trees because the tree structure Metropolis sampler does not work well on deep trees. So the tree depth choice is not justified on a statistical basis, but on a practical one. This makes me expect that there should be cases where deeper trees would be more appropriate.

    The calculation of the GP version of BART is indifferent to the depth of the trees, since they are marginalized away analytically, and \emph{GP-hyp} tunes the parameters $\alpha$ and $\beta$ that regulate the depth distribution. So maybe \emph{GP-hyp} is improving on \emph{GP} and getting closer to \emph{MCMC-CV} by doing better in particular on those datasets where deeper trees would be useful. From the point of view of the GP kernel, deeper trees means shorter correlation length.

    To investigate this, I look at the optimal values of $\alpha$ and $\beta$ found by \emph{GP-hyp}, and check if there is a relationship between them and the relative RMSE of \emph{GP-hyp} over \emph{MCMC-CV}. This comparison is shown in \autoref{fig:rmseratio}, panels~1 and~2. The slanted dashed lines vaguely indicate the trend expected if, indeed, \emph{GP-hyp} is doing better when it represents deeper trees.

    \begin{figure}
        \widecenter{\includempl{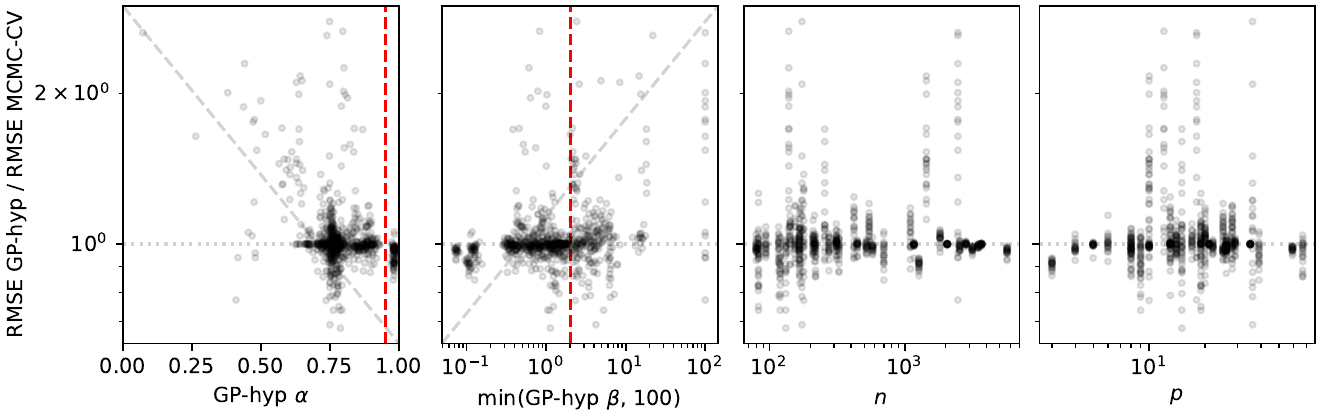}}
        \caption{\label{fig:rmseratio} Relationship between the relative performance of \emph{MCMC-CV} over \emph{GP-hyp}, and various potentially relevant variables. The red dashed lines mark the values of $\alpha$ and $\beta$ used in \emph{MCMC-CV}.}
    \end{figure}

    Although there is not a striking trend, the points towards the extremities do confirm the pattern by a small but unambiguous amount.

    \subsubsection{Effect of dataset size}
    \label{sec:datasize}

    I do not particularly expect in principle that dataset size would influence the relative performance of \emph{GP-hyp} over \emph{MCMC-CV}, but it is a very relevant variable, and so worth exploring nevertheless. The comparison is shown in \autoref{fig:rmseratio}, panels~3 and~4, and no pattern is visible.
    
    \subsubsection{Effect of dataset difficulty}
    \label{sec:difficulty}

    Like for dataset size, I have no expectation either way on what influence intrinsic dataset difficulty should have on \emph{GP-hyp} vs.\ \emph{MCMC-CV}, but it would be very relevant if there was, so I take a look. As a proxy of non-difficulty, I take the minimum RMSE between the two methods. The result is shown in \autoref{fig:explore}, panels~1-2. \emph{MCMC-CV} is evidently doing better on low-RMSE datasets, while \emph{GP-hyp} does better on high-RMSE ones.

    \begin{figure}
        \widecenter{\includempl{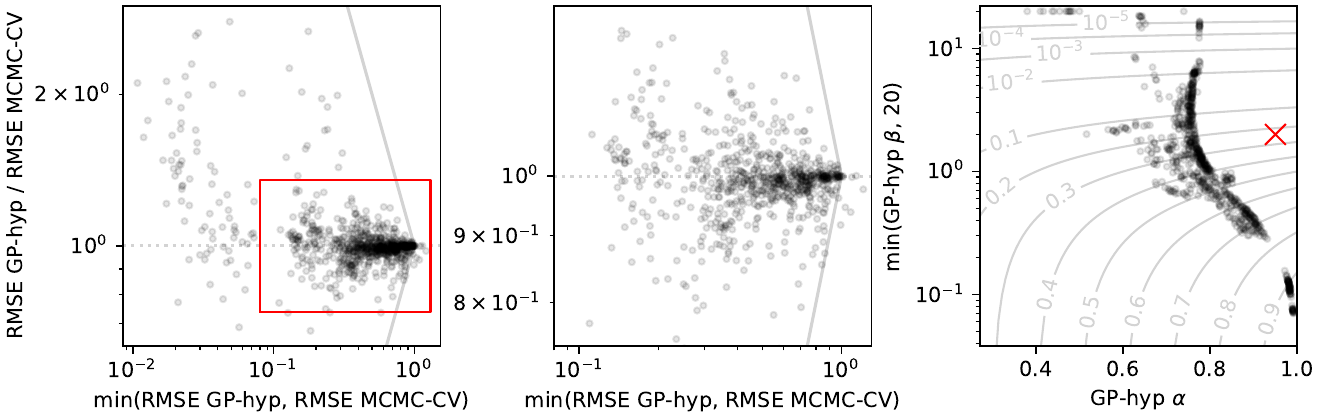}}
        \caption{\label{fig:explore} \textbf{Left and center panel:} relationship between the relative RMSE of \emph{MCMC-CV} over \emph{GP-hyp}, and the best RMSE; this tells if one of the two models tends to be better depending on the difficulty of the task. The funnel marks the boundary at which the worse method crosses RMSE=1. \textbf{Right panel:} relationship between the $\alpha$ and $\beta$ values found by \emph{GP-hyp}; the red cross marks the fixed values used in the MCMC, while the level curves indicate the value of $P_1=\alpha 2^{-\beta}$.}
    \end{figure}

    The data is standardized, so the maximum RMSE is about~1, when the predictive variance is equal to the total data variance. Many datasets have RMSEs above 0.5. Even though on this kind of datasets these sophisticate models may seem wasted, they are still correctly reverting to behaving like a regression with a constant function, rather than making wild extrapolations.

    \subsubsection{Differentiation of depth hyperparameters}
    \label{sec:diffdepth}

    The two depth-regulating parameters have a different effect on the trees: increasing $\beta$ not only limits the depth like lowering $\alpha$, but also makes the trees more balanced. To see if \emph{GP-hyp} is making use of this fine distinction, I plot together the values of $\alpha$ and $\beta$, and the values of the prior probability $P_1=\alpha 2^{-\beta}$ of creating children nodes from a node at depth~1. This comparison is shown in \autoref{fig:explore}, panel~3.

    The pairs of values mostly follow a streak orthogonal to the level curves of $P_1$, or slight variations of the same path. This suggests that the effects of the two parameters are poorly distinguished, since the inference is sticking to a starting point and then moving only along the direction that changes $P_1$. The specific location of the path along the level curves likely depends mostly on the choice of prior over $\alpha$, $\beta$.

    The striking feature of this graph is how wide the range of probabilities is. The extreme choices correspond to depth~1 trees (no additive interactions) and to deep trees (almost white noise). This suggests that \emph{GP-hyp} is reaching its performance primarily by adapting to the differences between datasets, rather than being good on a specific kind of dataset. It also weakly suggests that BART is limited by not being flexible about the depth of the trees, and it could be useful in general to at least try shallower depths since they are computationally cheap.

    Note: since the kernel approximation is poorer for high probability values (see \autoref{sec:interpolation}), high values of $P_1$ in \emph{GP-hyp} correspond to lower values of $P_1$ in actual BART.

    \subsubsection{Relationship between hyperparameter choice in MCMC and GP}
    \label{sec:hyperpick}

    Both \emph{MCMC-XCV} and \emph{GP-hyp} tune $\alpha$ and $\beta$, and I would expect their choices to be correlated. Instead, as shown in \autoref{fig:hyperpick}, there is no relationship at all.

    \begin{figure}
        \widecenter{\includempl{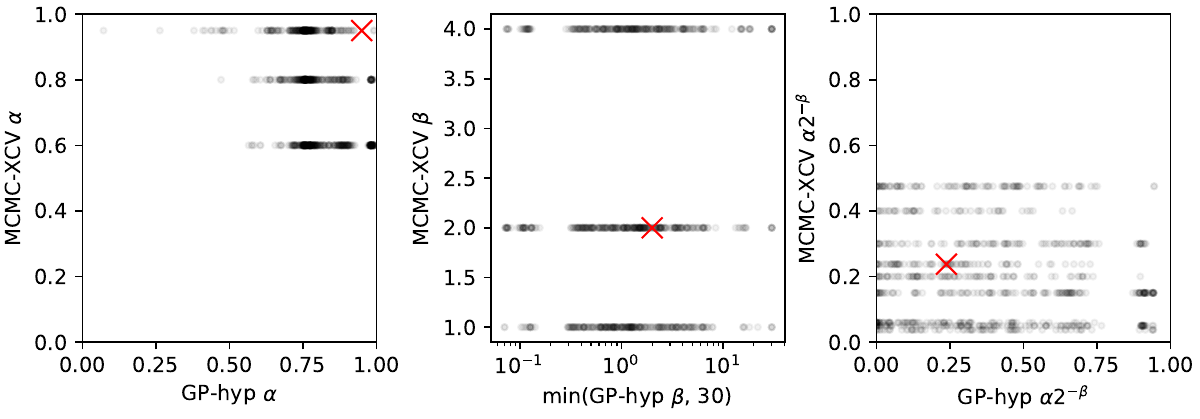}}
        \caption{\label{fig:hyperpick} Relationship between the depth regulating hyperpameters chosen by \emph{MCMC-XCV} and \emph{GP-hyp}. The red crosses mark the fixed values used in methods that do not optimize them.}
    \end{figure}

    \subsection{Cherry-picking/overtesting concerns}

    When showcasing benchmarks, there is always the possibility that the results were selected to make oneself look better. This unfortunately happens even in good faith. As hypothetical example, in choosing the datasets, I could have picked only datasets were I expected my method to work well. As assurance against this kind of problem, I took the following precautions:
    \begin{itemize}
        
        \item I test many methods together. Since the metrics are relative to the best method on each dataset, this way it is more difficult for me to influence the outcome.

        \item There are many datasets which again makes it more difficult to target any particular outcome.

        \item I reproduce as closely as possible the setup of the original BART paper \autocite{chipman2010}, with the method \emph{MCMC-CV}, the datasets, and dataset preprocessing. Since that paper was worked on for a long time \autocite[at least since][]{chipman2006}, I presume an effort was made to obtain good performance, so I'm letting the adversary choose the playing field.

        \item I considered all the popular implementations of BART. In particular the method \emph{MCMC-XCV} should represent the best of what the landscape has to offer out of the box.

        \item When trying out methods and their variations for the first time, I run them only on the first train/test split out of 20. Overtest on the first split hopefully does not completely generalize to overtest on all the splits.

        \item I had some degrees of freedom in choosing the parameters $D$, $\gamma$, $r$ of the kernel $k^{D,r}_{0,\gamma}$. Initially I explored various values and used those that yielded the best performance. I walked that back and used first-principles settings for everything.

    \end{itemize}

    \subsection{Choice of BART implementations}
    \label{sec:crancomparison}
    
    I consider the packages on CRAN that implement BART as originally defined in \textcite{chipman2010}, which are also the most popular by RStudio CRAN mirror download counts:
    \begin{itemize}
        \item \texttt{bartMachine} \autocite{kapelner2016,kapelner2023}
        \item \texttt{dbarts} \autocite{dorie2024}
        \item \texttt{BART} \autocite{sparapani2021,mcculloch2024}
        \item \texttt{BayesTree} \autocite{chipman2010,chipman2024}
    \end{itemize}
    I compare them against method \emph{GP} (\autoref{sec:modelgp}) on the \texttt{Abalone} dataset, $n_0=4177$, $p_0=8$, split in train-test 5:1 and transformed as described in \autoref{sec:benchmarkdata}. I use the hyperparameter values listed in \autoref{sec:modelmcmc}, but for a higher number of trees $m=1000$, to move BART closer to its GP limit. I run the MCMC for 2000 samples, dropping the first 1000.

    \paragraph{Configuration of the BART packages}

    I try to configure the packages to represent exactly the same Bayesian model I use for \emph{GP} (apart from the number of trees being finite). This is not trivial as their documentation is partial. I am confident that I'm following  the specification with \texttt{BART}, a bit less confident for \texttt{dbarts} and \texttt{BayesTree}, and I don't think I am perfectly adherent with \texttt{bartMachine}.

    I configure \texttt{BART} and \texttt{dbarts} to run 4 chains in parallel for 1250 samples, dropping the first 1000 and merging the rest to obtain 1000 posterior samples. I configure \texttt{bartMachine} to run a single chain for 2000 samples because it uses too much memory with multiple chains. \texttt{BayesTree} does not provide parallelization.

    \paragraph{Results}

    The running times and test set prediction errors are:
    \begin{center}
        \begin{tabular}{lrr}
            \toprule
            Method & RMSE & Time \\
            \midrule
            \emph{GP} & 0.589 & \SI{10}s \\
            \texttt{dbarts} & 0.584 & \SI{17}s \\
            \texttt{BART} & 0.585 & \SI{111}s \\
            \texttt{bartMachine} & 0.585 & \SI{101}s \\
            \texttt{BayesTree} & 0.585 & \SI{487}s \\
            \bottomrule
        \end{tabular}
    \end{center}
    The RMSE is very similar. \texttt{BART} is 6x slower than \texttt{dbarts}. Note that the time comparison between the R packages and \emph{GP} is unfair because I'm using an unusually large number of trees.

    To investigate further, in \autoref{fig:abalone} I compare the predictions and the error standard deviation across packages.
    \begin{figure}
        
        \widecenter{\includempl{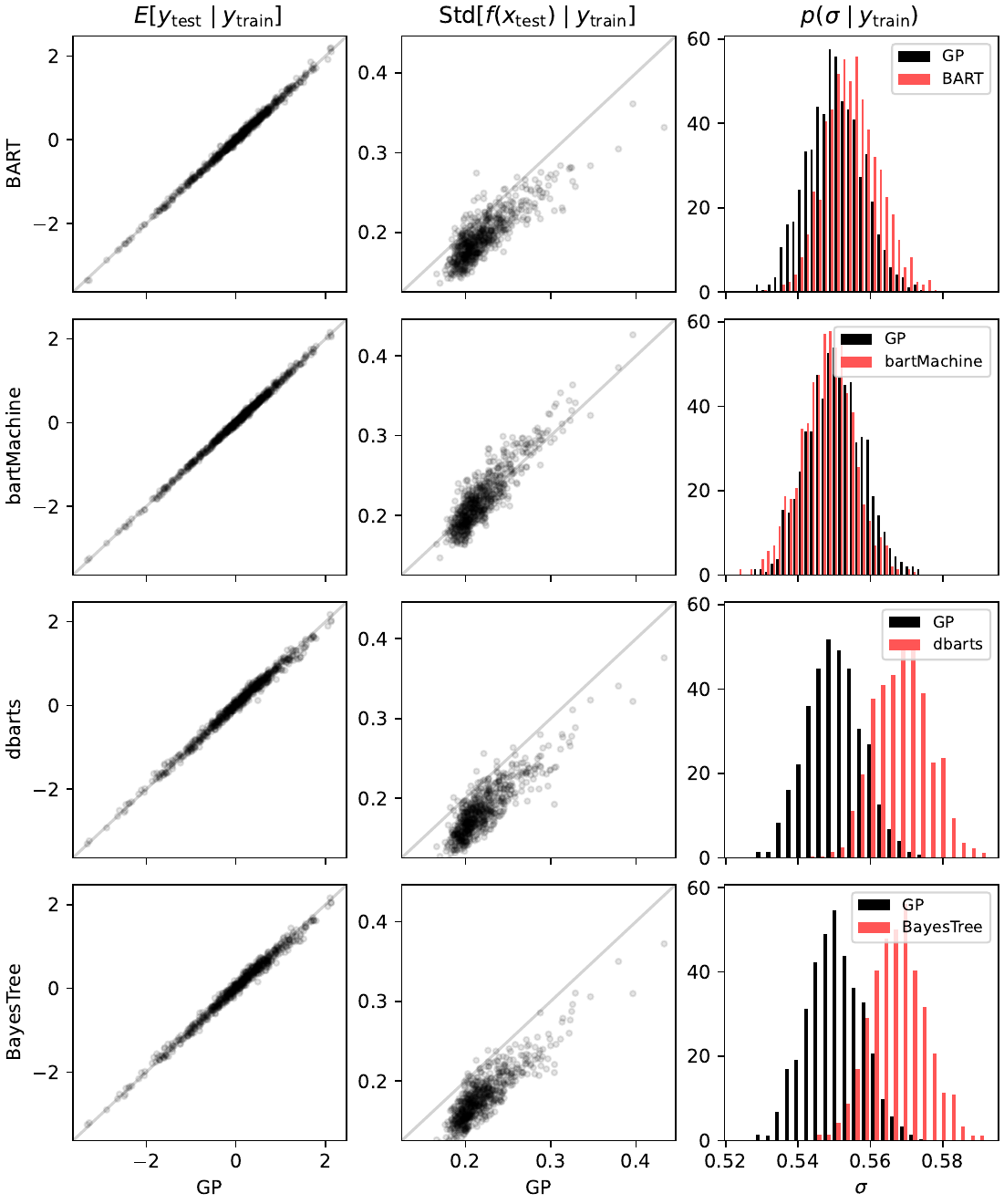}}
        
        \caption{\label{fig:abalone} Comparison of the BART packages \texttt{BART} , \texttt{bartMachine}, \texttt{dbarts}, \texttt{BayesTree}, and my infinite-trees limit \emph{GP}. In each column, the axis scales are the same in all plots. The rightmost column shows the marginal posterior of the error standard deviation, the only free hyperparameter. See \autoref{sec:crancomparison}.}
        
    \end{figure}
    There is some small but visible discrepancy on these. The methods are divided in two groups:
    \begin{itemize}
        \item \texttt{BayesTree}, \texttt{dbarts},
        \item \texttt{BART}, \texttt{bartMachine}, \emph{GP}.
    \end{itemize}
    The methods agree with each other within each group, but disagree with the other group. The first group yields a larger error standard deviation than the second. Repeating the experiment with a different dataset, a longer MCMC chain, a different random seed, or a different number of trees, does not change the general picture.

    \paragraph{Interpretation of the differences}

    The fact that \emph{GP} agrees with \texttt{BART} and \texttt{bartMachine} is unlikely to be a coincidence because \emph{GP} is a wildly different implementation. The fact that \texttt{BayesTree} agrees with \texttt{dbarts} makes sense because \texttt{dbarts} claims to be an improved drop-in replacement for \texttt{BayesTree}. So overall I take the conclusion that \texttt{BART} and \texttt{bartMachine} are more likely to be implementing BART as exactly specified on paper.
    
    A potential explanation for the difference between \texttt{BayesTree}-\texttt{dbarts} and the others is that they are the only ones to implement all the tree moves in the Metropolis proposal listed in \textcite{chipman1998}, while \texttt{BART} and \texttt{bartMachine} only implement the simplest moves \autocites[][\S A.3, p.~38]{kapelner2016}[][\S C, p.~57]{sparapani2021}. So there might be some small error or misspecification in the implementation of \texttt{BayesTree}-\texttt{dbarts} favored by the complexity of the moves.

    \paragraph{Final pick}

    The predictive performance is the same across packages. The fastest is \texttt{dbarts}, \texttt{BART} uses much less memory than \texttt{bartMachine}, and \texttt{BART} is more likely than the others to be implementing the model as on paper. So I choose \texttt{BART} for \emph{MCMC} where I care about making an exact comparison between the finite and infinite trees cases, and \texttt{dbarts} for all other cases.

    \section{Inferential accuracy on synthetic data}
    \label{sec:simdata}

    I analyse the Atlantic Causal Inference Conference (ACIC) 2022 Data Challenge synthetic datasets with a GP version of BCF using the BART kernel derived here; c.f.\ \textcite{horii2023} for another work looking at a GP version of BCF, but with standard kernels.

    \subsection{Data generating process}

    I use the ACIC 2022 Data Challenge datasets \autocite{thal2023}. They are 3400 datasets generated in groups of 200 from a family of 17 data generating process (DGP) variations on the same setup. The simulation is calibrated to imitate real longitudinal data about the causal effect of policy interventions to reduce Medicare spending. I use only the first 850 datasets to speed up execution; the order is randomized. The units are \num{300000} patients, grouped in 500 practices. The binary intervention $Z$ varies at the practice level, so patients within an hospital are either all subject to the intervention, or not. The outcome $Y$ is yearly expenditure, measured over 4 years. The intervention takes effect starting from year 3. There are patient-level covariates $V$, and practice-level covariates $X$, fixed throughout the years. Patients can enter or drop out from practices. The DGP is stated to satisfy a parallel trends  assumption:
    \begin{equation}
        Y_{3,4}(0) - Y_{1,2}(0) \Perp Z \mid X, V, Y_1, Y_2, \label{eq:did}
    \end{equation}
    where the subscript indicates the year, and $Y(z)$ indicates the potential outcome $Y(Z=z)$, representing $Y$ in a simulation where $Z$ was set to $z$ irrespectively of its actual causal antecedents, all else the same \autocite{rubin1974,imbens2015}. This assumption allows to identify the average causal effect on the treated units (SATT). See the results report \autocite{thal2023} and the official website \autocite{lipman2022} for further details.

    \subsection{Estimand}

    The estimand is the Sample Average Treatment effect on the Treated (SATT):
    \begin{equation}
        \operatorname{SATT}(S,T) =
        \frac 1 {\sum_{i:X_i \in S, Z_i = 1} \sum_{t\in T} \sum_j 1}
        \sum_{i:X_i \in S, Z_i = 1} \sum_{t\in T} \sum_j
        \big(Y_{ijt}(1) - Y_{ijt}(0)\big), \label{eq:patientsatt}
    \end{equation}
    where $T\subseteq\{3,4\}$ indicates the post-intervention years, and $S$ the subgroup of practices selected by the values of $X$, to be averaged over; and the indices $ijt$ run over practice, patient, and year respectively.

    The average is computed over the observed patients each year, without correcting for patient movement.

    The competition materials provide the true value of the SATT for various subgroups, computed by simulating both $Y(0)$ and $Y(1)$. The data contains only either one of the two as $Y_{ijt} = Z_i 1_{t>2} Y_{ijt}(1) + (1 - Z_i1_{t>2}) Y_{ijt}(0)$.
    
    \subsection{Structural modeling choices}

    I analyse the data aggregated at the practice level, because I don't have an implementation of Gaussian process regression with the BART kernel that scales to the number of patients. I start from the provided pre-computed averages, and add the within-practice standard deviation, minimum, and maximum of the non-binary patient-level covariates $V_{1,2,4}$. Since the formation of practices can change with time, the summaries of the patient-level covariates are yearly variables.

    Instead of hewing to the identifiability assumption provided by the competition rules, I assume the Structural Causal Model (SCM) \autocite{pearl2009} depicted in \autoref{fig:acicgraph}, left panel.
    \begin{figure}
        \begin{tikzpicture}
            % Nodes
            \node[latent] (Y1) {$\mathbf Y_1$} ;
            \node[latent, right=of Y1] (Y2) {$\mathbf Y_2$} ;
            \node[latent, right=of Y2] (Y3) {$\mathbf Y_3$} ;
            \node[latent, right=of Y3] (Y4) {$\mathbf Y_4$} ;
            \node[latent, above left=of Y1, yshift=1em] (V) {$\mathbf V$} ;
            \node[latent, above right=of Y2, xshift=-3ex, yshift=4em] (Z) {$Z$} ;
            \node[latent, above=of V] (X) {$X$} ;

            % Edges
            \edge {V} {Y1,Y2,Y3,Y4,Z};
            \edge {Y1,Y2} {Z};
            \edge {Z} {Y3,Y4};
            \edge {X} {Z,Y1,Y2,Y3,Y4};
            \draw[->] (Y1) -- (Y2);
            \draw[->] (Y2) -- (Y3);
            \draw[->] (Y3) -- (Y4);
            \path (Y1) edge [->,out=-20,in=-160] (Y3);
            \path (Y1) edge [->,out=-20,in=-160] (Y4);
            \path (Y2) edge [out=-20,in=-160] (Y4);

            % Plates
            \plate {plate} {(Z)(X)(Y1)(Y2)(Y3)(Y4)(V)} {practice $i$} ;
        \end{tikzpicture}
        \hfill
        \begin{tikzpicture}
            % Nodes
            \node[latent] (Y) {$\bar Y_{3,4}$} ;
            \node[latent, above left=of Y, xshift=2ex] (Z) {$Z$} ;
            \node[latent, left=of Y] (X) {$\tilde X$} ;

            % Edges
            \edge {Z} {Y};
            \edge {X} {Z,Y};

            % Plates
            \plate {plate} {(Z)(X)(Y)} {practice $i$} ;
        \end{tikzpicture}
        \caption{\label{fig:acicgraph} Causal models used to analyse the data, using plate notation to imply repetition of the nodes. Left panel: unaggregated model, presumed to be a mild assumption. Boldface variables are vectors over the patients in each practice. Right panel: aggregate model, obtained with approximations from the left one.}
    \end{figure}
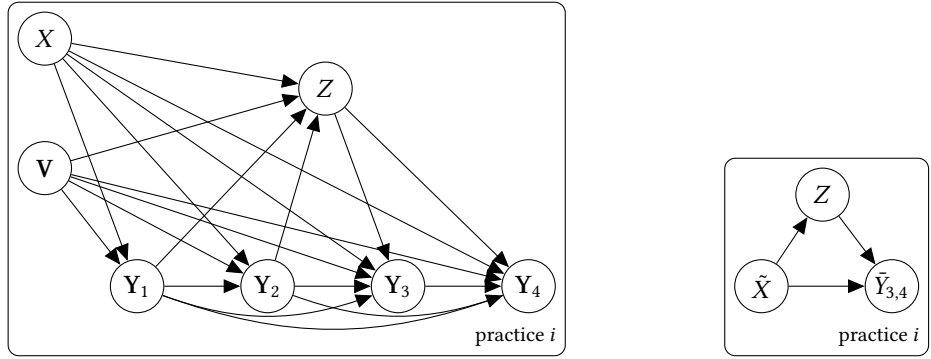
    It is a reasonable assumption because yielding the unconfoundedness property of \autoref{eq:did} without also satisfying that causal model would probably imply deliberate trickery on part of the authors, which does not seem to be in the style of the ACIC Data Challenge. This however is not a good justification because it tries to game the task; to refer to the actual use cases that the simulation is intended to imitate, that causal model implies assumptions practitioners\marginpar{Maybe ask Mealli about this.} would not in general be more or less willing to make respect to no-intervention trend unconfoundedness.

    The relevant conditional independence implied by the graph is
    \begin{equation}
        \{Y_{ijt}(z_i)\}_j \Perp Z_i \mid X_i, \{V_{ij}, Y_{ij1}, Y_{ij2}\}_j, \qquad t \in\{3,4\}.
        \label{eq:exactunconfoundedness}
    \end{equation}
    To make the aggregate analysis feasible, I hypothesize that, within some approximation, this also holds conditioning on the aggregates along patients:
    \begin{equation}
        \{Y_{ijt}(z_i)\}_j \Perp Z_i \mid
        X_i,
        \bar V_{i1}, \bar V_{i2},
        \bar Y_{i1}, \bar Y_{i2}, \bar Y_{i2} - \bar Y_{i1},
        n_{i1}, n_{i2},
        \qquad t \in\{3,4\},
        \label{eq:approxunconfoundedness}
    \end{equation}
    where $\bar V$ includes the additional summaries mentioned above in addition to averages, $n_{it}$ is the number of patients in practice $i$ during year $t$, and I added the pre-intervention outcome trend because it is explicitly hinted at as an important covariate in the rules. Instead, aggregation on the l.h.s.\ is an exact implication:
    \begin{equation}
        \bar Y_{it}(z_i) \Perp Z_i \mid X_i, \bar V_{i1}, \bar V_{i2}, \bar Y_{i1}, \bar Y_{i2}, \bar Y_{i2} - \bar Y_{i1}, n_{i1}, n_{i2}, \qquad t \in\{3,4\}.
    \end{equation}
    Defining $\tilde X_i = (X_i, \bar V_{i1}, \bar V_{i2}, \bar Y_{i1}, \bar Y_{i2}, \bar Y_{i2} - \bar Y_{i1}, n_{i1}, n_{i2})$, the final operative causal model is the basic one represented in \autoref{fig:acicgraph}, right panel, and the practice-level identifiability assumption is
    \begin{equation}
        \bar Y_{3,4}(z) \Perp Z \mid \tilde X.
    \end{equation}

    The target estimand (\autoref{eq:patientsatt}) is an average over patients, so it can be estimated working at the practice level:
    \begin{equation}
        \operatorname{SATT}(S,T) =
        \frac 1 {\sum_{i:X_i \in S, Z_i = 1} \sum_{t\in T} n_{it}}
        \sum_{i:X_i \in S, Z_i = 1} \sum_{t\in T} 
        n_{it} \big(\bar Y_{it}(1) - \bar Y_{it}(0)\big).
        \label{eq:sattaggr}
    \end{equation}
    Since the target of the inference is an average over the observed patients, I don't have to deal with the issue of move-in and drop-out. In a real setting, this would be relevant, because the intervention could have an effect on them: e.g., the policy could be lowering the expenses by ``convincing'' problematic patients to move to another hospital, rather than by lowering the expenses for each individual. However, in the context of the simulated competition, what matters is estimating the target quantity, whatever its meaning.

    \subsection{Specific model for inference}

    In the following, the binary covariates $X_{2,4}$ are one-hot encoded.

    I regress $Z$ on $\tilde X$ using BART with probit link, with the \texttt{bart} function of the R package \texttt{dbarts} \autocite{dorie2024}, with default settings but for \texttt{usequants = True, numcut = \emph{\textless large number\textgreater}, nchain = 4}, i.e., I use the midpoints between consecutive observed covariate values as grid of allowed tree splits, and run 4 Markov chains for 1250 samples each, discarding the first 1000. I extract the posterior mean as estimate of the propensity score: $\hat\pi_i = E[P(Z=1|\tilde X_i) \mid \text{data}]$.\marginpar{It would be nice to estimate $\hat\pi$ with GP-BART as well.}

    I then use a GP version of the Bayesian Causal Forests (BCF) model detailed in \autoref{sec:bcf} to regress $\bar Y_t$ on $(\tilde X, Z, \hat\pi, t)$. I sample 50 times a Laplace approximation of the hyperparameters posterior, then sample the Gaussian process 20 times for each hyperparameter sample. I compute the target estimands with \autoref{eq:sattaggr} on each sample, obtaining 1000 posterior samples of the SATT on all units and within various subsets specified by the competition rules. I take as estimate the posterior mean, and as \SI{90}\% interval the range from the \SI{5}\% to the \SI{95}\% posterior quantile.

    \subsection{GP version of BCF}
    \label{sec:bcf}

    I adapt the BCF model of \textcite{hahn2020} to use a GP with the BART kernel instead of BART. I also add a data transformation step with the Yeo-Johnson formula \autocite{yeo2000}, and an additional GP term that depends only on the propensity score to increase its a priori relevance amongst the covariates.

    Using the notations $\mathbf x_i = (\tilde x_i, t_i)$ and $y_i = \bar y_{t_i}$, the complete probabilistic model is:
    \begin{align}
        y_i &= g^{-1}(\eta_i; \lambda_g), \\
        g &= \text{Standardization, then Yeo-Johnson with parameter $\lambda_g$}, \\
        \eta_i &= m +
                \lambda_\mu \mu(\mathbf x_i, \hat\pi_i) +
                \lambda_\tau \tau(\mathbf x_i) (z_i - z_0) +
                \lambda_\pi f(\hat\pi_i) +
                \varepsilon_i,
    \end{align}
    \begin{align}
        \varepsilon_i &\sim \mathcal N(0, \sigma^2), &
        \log \sigma^2 &\sim \mathcal N(0, 2^2), \label{eq:bcferrorvar} \\
        m &\sim \mathcal N(0, 1), &
        z_0 &\sim U(0, 1), \\
        \lambda_\mu, \lambda_\pi &\sim \mathrm{HalfCauchy}(2), &
        \mu &\sim \mathcal{GP}(0, k_\mathrm{BART}(\cdot,\cdot;\alpha_\mu, \beta_\mu) ), \\
        \lambda_\tau &\sim \mathrm{HalfNormal}(1.48), &
        \tau &\sim \mathcal{GP}(0, k_\mathrm{BART}(\cdot,\cdot;\alpha_\tau, \beta_\tau) ), \\
        \lambda_g/2 &\sim \mathrm{Beta}(2,2), &
        f & \sim \mathrm{GP}(0, k_\text{Laplace}), \\
        \alpha_\mu, \alpha_\tau &\sim \mathrm{Beta}(2, 1), &
        \beta_\mu, \beta_\tau &\sim \mathrm{InvGamma}(1, 1).
    \end{align}
    I left the conditionals implicit for brevity. The parameters of the half-Normal and half-Cauchy distributions are set to yield median~1. The inverse gamma over $\beta_*$ forbids $\beta_* \to 0$ to avoid identification problems as $k_\mathrm{BART}$ tends to white noise. As BART kernel I use $k_\text{BART} = k^{2,5}_{0,1}[P_0 \mapsto 1]$; I set $P_0=1$ to remove the BART implicit random intercept term, as I already have the parameter $m$. The Laplace kernel is $k_\text{Laplace}(x,x') = e^{-|x-x'|}$.

    To make inference, I follow the usual practice of marginalizing the whole Gaussian process ($\mu$, $\tau$, $f$, $m$, $\boldsymbol\varepsilon$), and then maximizing the marginal posterior of the other parameters, considering them hyperparameters. Finally, I approximate the posterior of the hyperparameters as Normal, using as covariance matrix the BFGS\marginpar{BFGS has one article to cite per acronym letter, so I won't do it.} inverse Hessian estimate computed on the last 10 steps of the optimization. To improve the approximation, I reparametrize the hyperparameters such that their prior is a standard Normal.

    The parameter of the Yeo-Johnson transformation is included in the inference. This requires adding the jacobian factor $\prod_i g'(y_i;\lambda_g)$ to the marginal likelihood.

    \subsection{Justification of the weighting choice}

    I use unweighted errors (\autoref{eq:bcferrorvar}). The practices are inhomogeneous in size, so it seems natural that the error variance should instead be weighted with practice size, as $\sigma^2/n_i$. Indeed, formally, one could imagine deriving the model for aggregate data starting from a model over the unaggregated data, flat within each group, and homoscedastic. Then its average is a model with error variances weighted as expected. The aggregation of covariates reduces in general the predictive power, but the within-group variance was already allocated completely to the error term, so there is no inconsistence.

    This reasoning, however, misses that aggregating the covariates interferes with the unconfoundedness assumption. In going from \autoref{eq:exactunconfoundedness} to \autoref{eq:approxunconfoundedness}, it's only an approximation based on intuition that the potential outcomes should stay independent of the treatment given the group average of the covariates, rather than the full set of all individual covariates within each group.

    So I'd like the error term to also capture the bias from what the model estimates to what would be a valid setup to measure the causal effect. I can't simply add a set of free bias parameters, because it is not something that can be inferred from the data alone: it's not about predicting better the regression function given the aggregated covariates, it's about knowing that the prediction conditional on a certain set of covariates can be used to infer a causal relationship rather than a mere correlation. The unconfoundedness assumption is not implied by the data.

    An inflated error term does not automatically make the model consistent again w.r.t.\ the causal effect, because it is random: it does not know in which direction the correct causal effect lies, and the data can't tell. It is a finite-sample hack to increase the estimated uncertainty to a degree that should include the correct result.

    The causal error probably increases with the group size, since more covariate values have been aggregated, straying farther from the original unconfoundedness statement. I do not know exactly how, so to keep it simple, I decided to assume the two effects balance and to use a constant error variance in the aggregate model.

    \subsection{Results}

    Averaging over datasets, I obtain root mean square error (RMSE) 17(1) (top \SI{21}\% amongst listed competitors) and coverage 0.935(8) (top \SI7\% in distance from 0.90). The classification with the SATT over all units is reported in \autoref{fig:aciclist}, while the per-group metrics are shown in \autoref{fig:acicgroup}.

    \begin{figure}
        \widecenter{\includempl{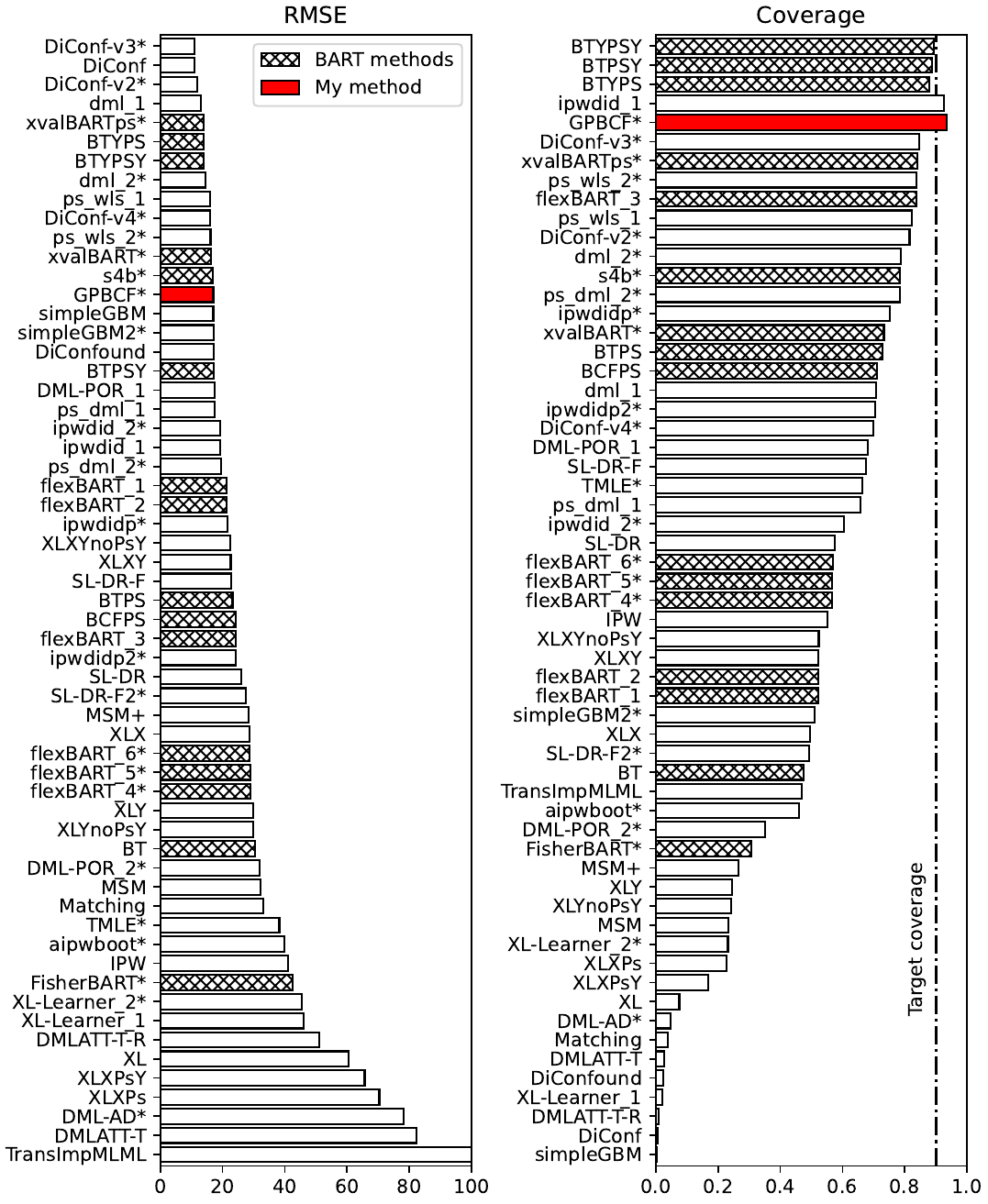}}
        \caption{\label{fig:aciclist} The official results of the ACIC 2022 Data Challenge (an hidden truth estimation competition with simulated data), plus my own method (\autoref{sec:simdata}) marked in solid red. Each label identifies a competitor; an asterisk after the label indicates that the entry is post-competion, submitted after the results were out and the dissemination of the hidden truth.}
    \end{figure}

    \begin{figure}
        \widecenter{\includempl{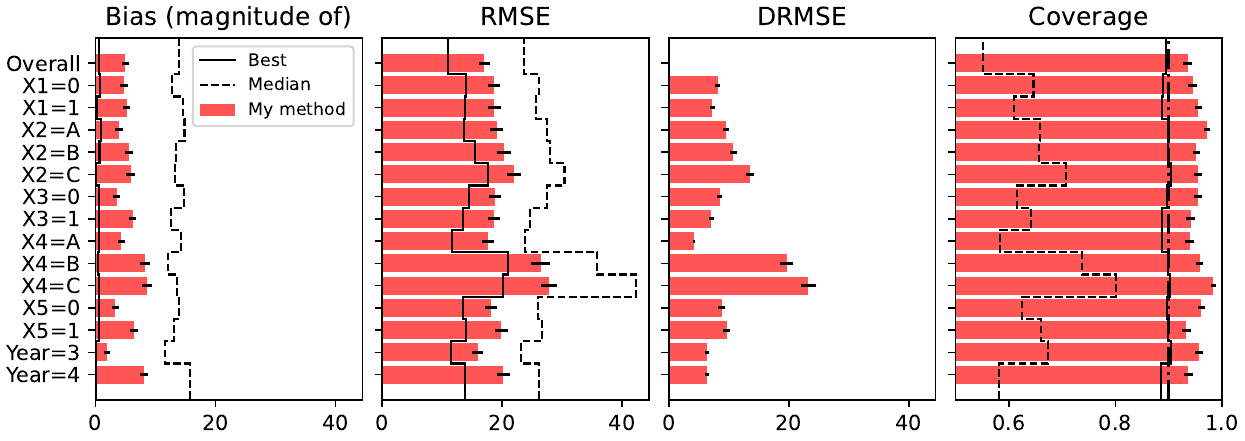}}
        \caption{\label{fig:acicgroup} Subgroup analysis results of the ACIC 2022 Data Challenge (an hidden truth estimation competition with simulated data), with my method (\autoref{sec:simdata}) pitted against the per-subgroup best and median competitor. The vertical axis labels indicate the subgroup. The DRMSE is the RMSE for the estimate of the difference between the effect in one subgroup and the grand effect; unfortunately the DRMSE results were not provided with the competition material. The small dashes set on the tip of each bar are standard errors.}
    \end{figure}

    To benchmark the ability to estimate the heterogeneity of the effect, the competition uses the DRMSE \autocite[15]{thal2023}, defined as the RMSE of the estimate of the difference between the effect in one subgroup and the overall effect:
    \begin{equation}
        \mathrm{DRMSE}(S) = \sqrt{E[((\hat\theta(S) - \hat\theta) - (\theta(S) - \theta))^2]}.
    \end{equation}
    The post-competition material does not include sufficient information to compute the DRMSE of the other competitors, so I visually read the results off \textcite[fig.~2, p.~16]{thal2023}:
    \begin{center}
        \begin{tabular}{llll}
            \toprule
            & \multicolumn{3}{c}{DRMSE} \\
            \cmidrule{2-4}
            Subgroup & Best & My method & BART second-worst \\
            \midrule
            Largest & 3.0 & 4.2(2) & 6.1 \\
            Smallest & 16 & 23(1) & 26 \\
            \bottomrule
        \end{tabular}
    \end{center}
    I consider the worst performer only amongst BART-based methods, which were overall the best at this, but ignore the BART absolute worst because it is an outlier, with the rest of the BART methods nicely bunched together.

    Overall, by RMSE my method is in the bunch of decent methods, by DRMSE it performs in line with the other BART-based methods (the best ones), and by coverage it's amongst the best and one of the only two methods overcovering rather than undercovering. The two reasons I can think of that got me the highest coverage are:
    \begin{itemize}
        
        \item I sampled the hyperparameters: I guess all other competitors with complex models kept hyperparameters fixed to default values, or optimized them with CV without keeping into account their uncertainty.

        \item I used homoscedastic errors instead of weighting by group size, to keep into account the unconfoundedness violation due to aggregation. Here I'm less sure about other competitors' choices, I suspect most of them have done the same.

    \end{itemize}

    \subsection{Competition fairness}

    My entry is post-competition. This makes it unfairly advantaged because I had already seen which methods were successful and knew more about the DGP. For example, maybe I would not have chosen the ACIC 2022 challenge as test bed, if I didn't already know that BART methods had been successful in it.

    Moreover, to check my procedure, I picked the first (and only the first) dataset, and visually compared its hidden truth with the results of the first version of my code. I noticed no mistakes then, but I ended up doing various adjustments afterwards anyway. I chose them based on general reasoning and other numerical experiments, rather than picking what minimized the error on that specific dataset. These changes were 1) sampling the hyperaparameters, 2) using homoscedastic errors, 3) adding a boosting term for the propensity score.

    Even though I ``know'' why I did them, I don't fully trust my brain to know not to have made different choices based on observing the hidden truth. What mostly keeps me unworried is: the randomness of a single dataset, there being 17 variations of the DGP, and the choices being binary (I did not try subversions of them, or properly compared combinations).

    All this is somewhat compensated by there being other post-competition entries in the results.

    These issues limit what can be inferred from my comparison. Overall, I think I can claim that my GP version of BART performs in line with the other BART methods, which is what I am mostly interested about here.

    \section{Comparison with the Laplace kernel}
    \label{sec:laplacekernel}

    \textcite[eq.~5.2]{linero2017} states that with some modifications to the BART prior, chosen to simplify calculations, the BART prior correlation function becomes the exponential (or Laplace) kernel
    \begin{equation}
        k(\mathbf x, \mathbf x') = \exp(-\eta \|\mathbf x-\mathbf x'\|_1). \label{eq:laplacekernel}
    \end{equation}
    In this section I compare the Laplace kernel with my exact derivation (\autoref{th:corr}) and with my computational approximation $k^{2,5}_{0,1}$ (\autoref{sec:compcorr}).

    \subsection{The Laplace kernel cannot approximate the BART kernel}

    In this section I provide a semi-formal argument that the Laplace kernel in \autoref{eq:laplacekernel} can not approximate the BART kernel well.

    Consider that BART, with default hyperparameters, gives low prior weight to interactions. This means that the covariance function can be written as a sum separate along covariate axes plus some small term. To make the exponential of the Laplace kernel approximately separable as a sum, its parameter $\eta$ has to be small. However, a small $\eta$ means that the minimum correlation will be close to~1, while it should stay fixed at $1 - \alpha$ according to property~\ref{it:lower}, \autoref{th:corrprop}.

    To make this argument more precise, I'll try to approximate the BART correlation function with the form
    \begin{equation}
        k(\mathbf x, \mathbf x')
        = \sigma_0^2 + \sigma^2 \exp\left(
            -\eta\alpha \frac{\|\mathbf x - \mathbf x'\|_1}p
        \right),
        \label{eq:altkernel}
    \end{equation}
    where $\sigma_0^2$ and $\sigma^2$ are free parameters, and I added a factor $\alpha/p$ such that the linear term of the exponential at $\eta=1$ matches the $\beta\to\infty$ limit of BART (property~\ref{it:nointer}, \autoref{th:corrprop}), assuming that the BART grid is evenly spaced in $[0, 1]$ along all axes. The parameter $\sigma_0^2$ represents an intercept, while $\sigma^2$ rescales the variance of the non-constant part of the kernel; I added these because they are the most basic extensions commonly used in GP regression, so comparing a ``pure'' Laplace kernel with the BART kernel would not be practically meaningful.

    I now impose that \autoref{eq:altkernel} match the extremal cases of the BART kernel:
    \begin{align}
        \begin{cases}
            k(\mathbf 0, \mathbf 0) = 1, \\
            k(\mathbf 0, \mathbf 1) = 1 - \alpha
        \end{cases}
        &&\implies &&
        \begin{cases}
            \sigma_0^2 + \sigma^2 = 1, \\
            \sigma_0^2 + \sigma^2 e^{-\eta\alpha} = 1 - \alpha.
        \end{cases}
    \end{align}
    Solving for $\sigma^2$ yields
    \begin{equation}
        \sigma^2 = \frac\alpha{1 - e^{-\eta\alpha}}.
    \end{equation}
    The first equation of the system implies $\sigma^2 \le 1$, so
    \begin{align}
        \frac\alpha{1 - e^{-\eta\alpha}} \le 1
        &&\implies &&
        \eta \ge \frac{-\log(1 - \alpha)}\alpha.
    \end{align}
    For values of $\alpha$ close to 1, this lower bound on $\eta$ becomes large. For the default $\alpha=0.95$, it is $\eta \ge 3.2$. As clearly visible in \autoref{fig:plotcov}, with the default $\beta=2$ the BART kernel (approximated to \SI{1}\%, see \autoref{sec:finalkernel}) is pretty close to its straight-shaped $\beta\to\infty$ limit, while the Laplace kernel with $\eta=3.2$ is already too curved to be approximated by its straight-shaped linear term: to approximate $\exp(-\eta\alpha\|\mathbf x-\mathbf x'\|_1/p) \approx 1-\eta\alpha\|\mathbf x-\mathbf x'\|_1/p$, $\eta$ should be at least less than~1.

    \subsection{Some Laplace-like kernel may approximate the BART kernel}

    Graphically, the problem with the Laplace kernel explained in the previous section is that if $\eta\to 0$, the shape of the kernel becomes straight, but it also tends to a constant function close to~1 over the domain $[0,1]^p\times [0,1]^p$. Instead, if $\eta$ is large, the kernel varies over the domain, but it's curved, instead of almost straight like the BART kernel. See \autoref{fig:bartvslaplace}.

    \begin{figure}
        \widecenter{\includegraphics[width=80ex]{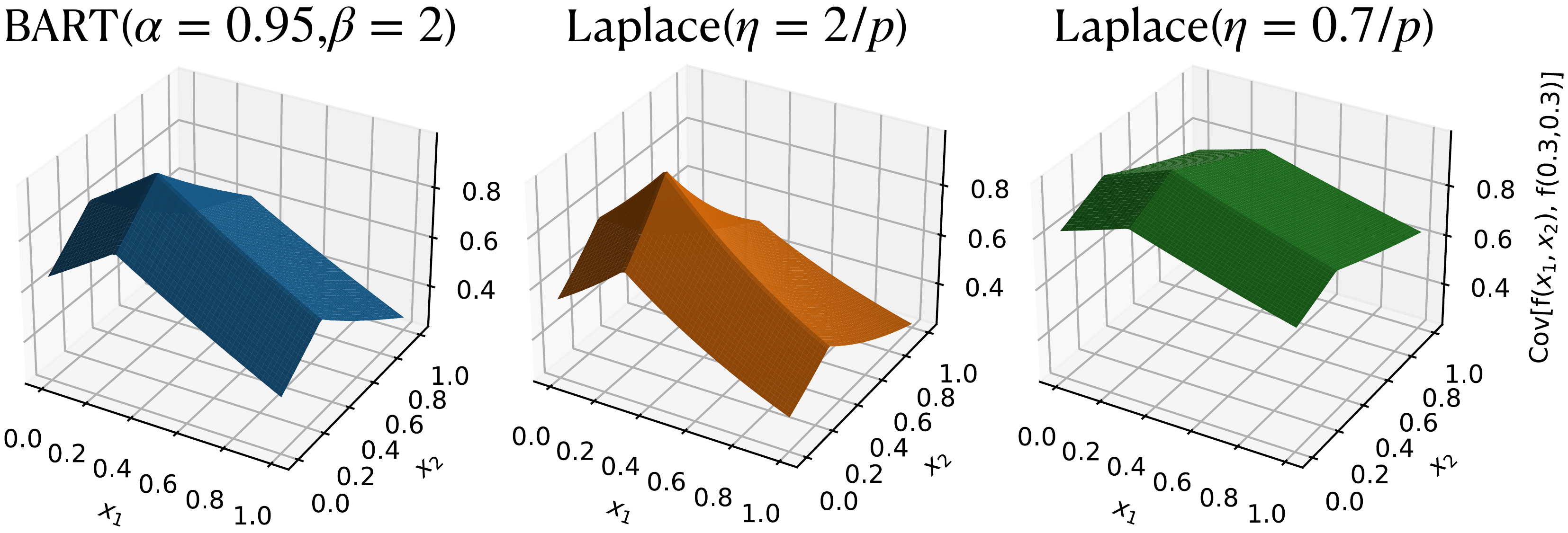}}
        \caption{\label{fig:bartvslaplace} Graphic comparison of the BART kernel (computed with $k^{2,5}_{0,1}$, \autoref{sec:compcorr}) with the Laplace kernel (\autoref{eq:laplacekernel}). Increasing $\eta$ makes the Laplace kernel reach the correct extreme values, but too curved; while decreasing $\eta$ makes the Laplace kernel straight, but almost constant.}
    \end{figure}

    \paragraph{Subtraction of the intercept from the Laplace kernel}

    Intuitively, it is was possible to keep $\eta$ small, subtract a constant from the Laplace kernel, and then rescale it, it would have the correct global shape. Subtracting positive constants from kernels is not a generally positivity preserving operation, but it turns out it is possible in limited form with the Laplace kernel:

    \begin{theorem}
        If the Laplace kernel $k_L(\mathbf x,\mathbf x')=\exp(-\eta\|\mathbf x-\mathbf x'\|_1/p)$ is restricted to the domain $[0, 1]^p\times [0, 1]^p$, the function $\tilde k_L(\mathbf x,\mathbf x') = k_L(\mathbf x,\mathbf x') - \min k_L$ is a valid kernel. \label{th:laplacesub}
    \end{theorem}

    \begin{proof}
        
        The function $k_0(\mathbf x, \mathbf x') = 1 - \|\mathbf x - \mathbf x'\|_1/p$ is pos.\ def.\ for $\mathbf x \in [0, 1]^p$ because it is the sum of the triangular covariance function \autocite[eq.~4.21, first line, p.~88]{rasmussen2006}\marginpar{I'd like a more direct reference than this one.} over each axis. The function $T(x) = e^{\eta x} - 1$ has a Taylor series with positive coefficients, which means that it preserves positivity if applied to the output of a covariance function \autocite[eq.~18.30, p.~682]{murphy2023}.\marginpar{Actually, that reference states it for polynomials. I haven't encountered a place with a proper proof for the series case.} So the function $e^{-\eta} T(k_0(\mathbf x, \mathbf x'))$ is pos.\ def.\, and equal to
        \begin{align}
            e^{-\eta} T(k_0(\mathbf x, \mathbf x')) &=
            e^{-\eta}\left(\exp\left(
                \eta \left(1 - \frac{\|\mathbf x - \mathbf x'\|_1} p\right)
            \right) - 1\right) &= \notag \\
            &= \exp\left(
                -\eta\frac{\|\mathbf x - \mathbf x'\|_1} p
            \right) - e^{-\eta} = \notag \\
            &= k_L(\mathbf x,\mathbf x') - \min_{[0, 1]^p\times [0, 1]^p} k_L.
        \end{align}

    \end{proof}

    I try again to match BART to the Laplace kernel like in the previous section, but this time shifting the Laplace kernel with \autoref{th:laplacesub}. I rescale and shift the kernel, and multiply $\eta$ by $\alpha$:
    \begin{equation}
        k(\mathbf x, \mathbf x')
        = \sigma_0^2 + \sigma^2 \left(\exp\left(
            -\eta\alpha \frac{\|\mathbf x - \mathbf x'\|_1}p
        \right) - e^{-\eta\alpha}\right),
        \label{eq:altkernelsub}
    \end{equation}
    and then impose that \autoref{eq:altkernelsub} match the extremal cases of the BART kernel:
    \begin{align}
        \begin{cases}
            k(\mathbf 0, \mathbf 0) = 1, \\
            k(\mathbf 0, \mathbf 1) = 1 - \alpha
        \end{cases}
        && \implies &&
        \begin{cases}
            \sigma_0^2 + \sigma^2(1 - e^{-\eta\alpha}) = 1, \\
            \sigma_0^2 = 1 - \alpha
        \end{cases}
        && \implies &&
        \begin{cases}
            \sigma^2 = \alpha/(1 - e^{-\eta\alpha}), \\
            \sigma_0^2 = 1 - \alpha.
        \end{cases}
    \end{align}
    Substituting, the approximant kernel turns out
    \begin{equation}
        k(\mathbf x, \mathbf x') =
        1 - \alpha +
        \frac{\alpha}{1 - e^{-\eta\alpha}}
        \left(
            \exp\left( -\eta\alpha\frac{\|\mathbf x-\mathbf x'\|_1}p \right)
            - e^{-\eta\alpha}
        \right). \label{eq:altkernelfinal}
    \end{equation}
    Expanding to first order in $\eta$, this becomes
    \begin{equation}
        k(\mathbf x, \mathbf x') = 1 - \alpha\frac{\|\mathbf x-\mathbf x'\|_1}{p} + o(\eta),
    \end{equation}
    which matches the $\beta\to\infty$ limit of BART. Since BART at default hyperparameters is close to that limit, this kernel might be a good approximation in practice to the BART kernel. I haven't conducted empirical experiments to study this further.

    \paragraph{Other potential approximations to the BART kernel}

    It is probably possible find many other simple kernels that can be adapted to approximate the BART kernel. The first alternative I came up with is
    \begin{align}
        k(\mathbf x, \mathbf x') &= 1 - \alpha + \alpha
        \left( 1 - \frac{\|\mathbf x - \mathbf x'\|_1}{p} \right)^q.
        \label{eq:extq}
    \end{align}
    I was not able to formally prove it is p.s.d.\ for non-integer $q$, but I checked numerically that it seems to be p.s.d.\ for any real $q \ge 1$. Setting $q$ to something slightly above~1 makes the kernel slightly curved like the BART kernel, so it might work as an approximation.

    More in general, it might be fruitful to search for a kernel in the form
    \begin{equation}
        k(\mathbf x, \mathbf x') = 1 - \alpha + \alpha
        f_\theta\left( 1 - \frac{\|\mathbf x - \mathbf x'\|_1}{p} \right),
        \label{eq:generalizedlaplace}
    \end{equation}
    with $f_\theta:[0,1]\to[0,1]$ a family of functions parametrized by $\theta$ such that $f_\theta(0)=0$, $f_\theta(1)=1$, $f_\theta$ yields the identity for some $\theta$, is convex, and preserves positive definiteness. The latter two conditions could be fulfilled by searching amongst functions with nonnegative coefficients in their power series around~0; the simplest example would be polynomials with nonnegative coefficients, scaled to the desired range.

    \section{Extrapolation of the results to GP regression in general}
    \label{sec:extrapolation}

    In this section I speculate about what all the results presented in this work say about how GP regression with arbitrary kernels compares to BART and indirectly to other sophisticate methods.

    \subsection{Bounding the complexity of GP regression}

    Since in practice GP regression almost always includes adapting some free hyperparameters of the kernel, a fair comparison between GP regression and BART requires doing so. However, if arbitrary kernel families are allowed, BART itself counts as a GP regression \autocite[footnote~5, \S5.2, p.~1052]{hahn2020}: since the leaves are a priori Normal, and the sum of Normals is Normal, the BART prior is a GP at fixed tree structures, and so a mixture of GPs with the full prior specification, like any in-practice ``GP'' prior.

    Then, to make a meaningful GPs vs.\ BART comparison, I have to limit what I allow to count as ``GP''. I don't know how to write down a precise mathematical definition that captures the mixtures of GPs allowed, and I don't think it would make sense either; the choice of allowed GP mixtures should reflect the practice and computational methods commonly called ``GP regression'': so there should be a focus on solving linear systems defined by a prior covariance matrix, and optimization or sampling of the hyperparameters of this matrix with explicit (marginal) likelihood evaluations.

    For the purpose of this discussion, I lay down some concrete rules on the kernel that try to capture the above requirement:
    \begin{itemize}
        \item the kernel formula is simple enough to be written in computable form on a few lines at most
        \item the evaluation of the kernel takes at most $O(p)$ time and memory, where $p$ is the number of predictors
        \item at most $O(10)$ ``general'' hyperparameters with a loose prior
        \item at most $O(p)$ ``feature-learning'' hyperparameters with an informative prior
        \item all the hyperparameters have a real interval as domain, and the kernel depends smoothly on them
        \item computing the gradient of the kernel w.r.t.\ the hyparameters is convenient
    \end{itemize}
    This list tries to represent the most typical GP regression application; more special cases would require some adaptation, but to keep things simple I avoid handling them.

    By these rules, BART does not count as GP regression, while the infinite trees limit of BART as I present it does.

    \subsection{Generalization of the BART vs.\ GP-BART case}

    In this work I compared BART with its own infinite trees GP limit, which in this section I'll refer to as ``GP-BART'' \autocite[I avoid doing so elsewhere because the name has already been used with different meaning in][]{maia2022}.

    I showed that in practice GP-BART is overall almost as good as BART. I start from this fact to make an extrapolation about what would happen if one tried to systematically search for the ``best'' GP methodology using BART as reference comparison: by this I mean that the GP method would be scored in terms of how better it does than BART on each dataset, and the goal would be trying to improve some average score over some distribution of datasets representative of the applied usage of BART.

    The argument is as follows. GP-BART is a GP regression method produced by the precise specification of taking the infinite trees limit of BART, without adding pieces to the model or hyperparameters. When trying to maximize a score over some set of options, enlarging the set of options can only increase the maximum score achievable. So considering arbitrary GP methods instead of GP-BART alone should allow to increase the quality of the method. It would be surprising if it was possible to almost match BART with a very specific choice of GP, but then it was \emph{impossible} to close the gap by allowing arbitrary adjustments to the GP method.

    A key step in the above argument is going from ``the score \emph{cannot decrease} if I enlarge the set of methods considered'' to ``the score \emph{probably increases} enough if I enlarge the set of methods''. I think this holds because GP-BART was not chosen specifically to maximize the score, but rather to study the infinite trees limit of BART, which was already suspected not to be the best configuration of BART. So GP-BART is unlikely to already be close to the best achievable, on the grounds that devising good methods is found empirically difficult by researchers and not something that comes for free.

    I realize this argument would be more compelling with an empirical test showing such a GP method, but I decided to cut this work short to dedicate my time to other pursuits. I think searching for a GP kernel and mixture specification yielding a regression method uniformly superior-or-equal to BART is the most promising direction of research indicated by this work. A starting point could be playing with the variants of the Laplace kernel in \autoref{sec:laplacekernel}, Equations~\ref{eq:altkernelfinal}, \ref{eq:extq} and~\ref{eq:generalizedlaplace}.

    \paragraph{Empirical evidence that GP-BART could be improved upon}

    In the analysis of the benchmarks on real data, I show that GP-BART is slightly better on difficult (noisy) datasets, while quite worse on easy (high SNR, predictable) ones (\autoref{sec:difficulty}, \autoref{fig:explore}). Intuitively, improving the performance on the easy datasets should be an easier task. The higher RMSE on low-noise data may be due to the GP prior being too rigid to fit the specific, precisely visible shape of the data; if that's the case, my intuition is that this would be addressed by adding some free hyperparameters.

    A potential counterargument is that the features of low-noise datasets that may be giving a hard time to GP regression are local, like occasional jumps. This kind of feature is more difficult to handle with a simple kernel, even if there are many hyperparameters. I speculate that if this is an issue, it could be addressed by adding what I called ``feature-learning'' hyperparameters that control how predictors are transformed and used in the kernel.

    \subsection{About the intrinsic theoretical limits of GP regression}
    \label{sec:gpbad}

    There are general theorems stating that GP regression has intrinsic limits compared to the class of nonparametric regression methods in general \autocite[see, e.g.,][]{donoho1998,giordano2022,agapiou2024,abraham2023}. These results support the expectation that the infinite trees limit of BART would be, ceteris paribus, a worse method, as confirmed empirically in \autoref{sec:benchmarkmain}. However, to the extent of my knowledge, all these results concern GPs with no free hyperparameters, or possibly with very limited and specific degrees of freedom. Although I recall encountering once a concrete research work with a GP without free hyperparameters, the norm is to have free hyperparameters because it greatly improves the method. Thus, as detailed above, the practical question I care about is how \emph{mixtures} of GPs, with a parametrization of the kernel amenable to the most usual and straightforward computational techniques for GP regression, compare to other methods.

    \section{Code and data}

    The results of this article are reproducible with \textcite{petrillo2024d}. The GP versions of BART and BCF are ready to use in the Python package \texttt{lsqfitgp} \autocite{petrillo2024c}.
        
\end{document}